\documentclass[preprint]{imsart}

\RequirePackage{amsthm,amsmath,amsfonts,amssymb}
\RequirePackage[authoryear]{natbib}
\RequirePackage[colorlinks,citecolor=blue,urlcolor=blue]{hyperref}
\RequirePackage{graphicx}
\graphicspath{{figures/}{./}}

\usepackage[utf8]{inputenc}
\usepackage[margin=3.5cm]{geometry}

\usepackage{multirow}%
\usepackage{mathrsfs}%
\usepackage[title]{appendix}%
\usepackage{xcolor}%
\usepackage{textcomp}%
\usepackage{manyfoot}%
\usepackage{booktabs}%
\usepackage{algorithm}%
\usepackage{algorithmicx}%
\usepackage{algpseudocode}%
\usepackage{listings}%
\usepackage{nicefrac}
\usepackage{circuitikz}
\usepackage{subcaption}

\DeclareMathOperator{\prob}{P}
\DeclareMathOperator{\E}{E} % uncond'l expectation operator

\DeclareMathOperator{\diag}{diag}
\DeclareMathOperator{\argmin}{argmin}
\DeclareMathOperator{\tr}{tr}

\usepackage{mleftright} % for '\mleft' and '\mright' macros
\newcommand\pr[1]{\prob\mleft(\,#1\,\mright)}
\newcommand\ex[2][]{\E_{#1}\mleft[\,#2\,\mright]} % has one optional argument (first one) whose default is empty if not provided. Call via \ex[opt]{req} or \ex{req}

\newcommand{\vertiii}[1]{{\left\vert\kern-0.15ex\left\vert\kern-0.15ex\left\vert #1 \right\vert\kern-0.15ex\right\vert\kern-0.15ex\right\vert}}

\newcommand{\bs}[1]{\boldsymbol{#1}}
\newcommand{\dd}[1]{\mathrm{d}#1}
\newcommand{\pa}[1]{{\mathrm{pa}(#1)}}
\newcommand{\Pa}[1]{{\mathrm{Pa}(#1)}}

\newcommand{\anc}[1]{{\mathrm{an}(#1)}}

\DeclareMathOperator{\PaTrEr}{PTE}

\usepackage{thmtools}
\usepackage{thm-restate}

\declaretheorem[name=Theorem,numberwithin=section]{theorem}
\declaretheorem[name=Definition,sibling=theorem]{definition}

\begin{document}

\begin{frontmatter}
\title{Non-parametric recovery of causal diffusion mechanisms from steady-state observations}
\runtitle{Causal Diffusions}

\begin{aug}

\author{\fnms{Richard}~\snm{Schwank}$^{1}$\ead[label=e1]{richard.schwank@tum.de}} \and
\author{\fnms{Mathias}~\snm{Drton}$^{1,2}$\ead[label=e2]{mathias.drton@tum.de}}
%\orcid{0000-0000-0000-0000}}
\address{$^{1}$School of Computation, Information and Technology,
Technical University of Munich\printead[presep={,\ }]{e1,e2}}
\address{$^2$Munich Center for Machine Learning}
\end{aug}
\runauthor{R. Schwank and M. Drton}

\begin{abstract}
We consider sparse multivariate stochastic systems that evolve in continuous time according to a causal mechanism and present methodology to recover the system's time-infinitesimal transition mechanism from mere cross-sectional data. This observational paradigm is motivated by applications such as gene expression analysis, where destructive experimental techniques may only allow recording data once over a cell's lifetime.  Precisely, we assume the system follows a time-homogeneous diffusion process that has reached an equilibrium distribution at observation time. Further, we assume the causal mechanism is fully described by the diffusion drift, is acyclic, and its causal structure graph is known. In this setting, we prove that the full causal mechanism, i.e., the drift function, can be non-parametrically identified under a weak non-explosion criterion. We derive a non-parametric kernel estimator for this challenging inverse problem and prove its consistency. Moreover, we propose a cross-validation scheme for hyperparameter tuning, illustrate the behavior of our estimator in simulations, and we discuss connections with irreversible generative diffusion models and low-frequency sampled data. 
%The code is available as a Python package.
\end{abstract}

% \begin{keyword}[class=MSC]
% \kwdgroup[type=primary]{\kwd{00X00}
% \kwd{00X00}}
% \kwdgroup[type=secondary]{\kwd{00X00}}
% \end{keyword}

\begin{keyword}
\kwd{Stationary diffusion process}
\kwd{Causal dynamical system}
\kwd{Fokker-Planck inverse problem}
\kwd{Reproducing kernel Hilbert space}
\kwd{Score matching}
\kwd{Logistic population growth in a random environment}
\kwd{Graphical continuous Lyapunov model}
\end{keyword}

\end{frontmatter}

\section{Introduction}\label{sec:introduction}
Let $\mathbf{x}_{1},\ldots,\mathbf{x}_{n}\in\mathbb{R}^d$ be multivariate data collected by observing each one of $n$ independent individuals at a single time point. Suppose that probing individuals a second time is difficult/impossible, as is the case, for instance, when single-cell gene sequencing requires sacrificing the considered cells. We are interested in inference on the mechanism governing the time evolution of the random vector $\mathbf{x}$ from which the $n$ observations were generated. For example, we may wish to use expression data to infer an underlying gene regulation or protein signaling mechanism \citep{Lorch2026,Zhao2025, pmlr-v124-varando20a}. % Zhao2025 considers identifiability from distributions concentrated on a one-dimensional sphere.

Inspired by related literature \citep{Peters2022CausalModelsDynamicalSystems} and models for gene regulation \citep{Pratapa2020b}, we assume the unobserved evolution of $\mathbf{x}$ follows a time-homogeneous diffusion process $\mathbf{x}(t)$ given by the stochastic differential equation
\begin{equation*}
    \dd{\mathbf{x}(t)} = \mathbf{b}\big(\mathbf{x}(t)\big)\,\dd{t} + \bs{\Sigma}\big(\mathbf{x}(t)\big)\,\dd{\mathbf{w}(t)}
\end{equation*} 
with $\mathbf{b}:\mathbb{R}^d\to\mathbb{R}^d$ an unknown drift function  and $\mathbf{w}(t)$ a $d$-dimensional standard Brownian motion. The scaling function $\bs{\Sigma}$ is assumed to be known; often the case with $\bs{\Sigma}$ a constant multiple of the identity has been considered. We assume that our observations are representative, in the sense of being collected after the system has reached its steady state (and is in equilibrium). In other words, we assume the data %$\mathbf{x}_{1,T},\ldots,\mathbf{x}_{n,T}$ 
$\mathbf{x}_{1},\ldots,\mathbf{x}_{n}$ is generated as an i.i.d.~sample from a distribution with a density $p$ that is a stationary probability density of the diffusion process \citep{Lorch2024StationaryDiffusions,pmlr-v124-varando20a}. 

Even under the equilibrium assumption, there are multiple drifts $\mathbf{b}$ inducing the same stationary density $p$, see Section \ref{subsec:accelerating_difusion_theory} and \citet{Nickl2020}. % end of second page
One setting that allows unique recovery of the mechanism $\mathbf{b}$ are suitably sparse causal systems, i.e., the $i$-th drift component $b_i$ depends only on coordinates $x_j$ with $j$ belonging to a (typically small) index set. We depict this sparsity graphically via the parent set of $i$ in a directed graph whose nodes are the vector indices. 
In this paper, we consider a non-parametric model for drifts $\mathbf{b}$ that are causally structured according to a known directed acyclic graph (DAG). This setup arises for example when modeling the effect of a covariate process on an outcome process, see Section \ref{subsec:two_dim_population_rand_env} and \citet{Gregorio2025DiffusionProcessRegression}. % Also ergodic!
We first show that the drift $\mathbf{b}$ is non-parametrically identified by the stationary density $p$ under mild assumptions, i.e., that we can recover $\mathbf{b}$ uniquely given infinite data. On a technical level, we link $\mathbf{b}$ and $p$ via the Fokker-Planck equation and show uniqueness if all drift components $b_i$ are $p$-integrable, a weak condition to ensure the diffusion process does not explode \citep{Haesung2021FullSDEPackageUnderLocalIntegrability}. This result generalizes existing work making strong parametric assumptions \citep{Dettling2022}.

Our second contribution is a non-parametric kernel estimator for the drift $\mathbf{b}$ given a user-specified DAG.
%, \textcolor{red}{to be} available as a Python package. 
The estimator builds on a result from the identifiability section, which we interpret as an identification formula for the image of $\mathbf{b}$ under an integral operator. This formula is reminiscent of an identity from nonparametric score matching, and following \citet{Zhou2020ScoreEstimators}, we approximate both sides by plug-in estimators. All that then remains is to invert the integral operator, for which we study different regularization schemes. Building on  \citet{Vito2005}, we theoretically analyze the Tikhonov-regularized estimator and derive a high-probability concentration inequality. We deduce that any drift with components $b_i\in L^2(p)$ can be consistently estimated when an $L^2$-universal kernel is used. Additionally, if the Hilbert space corresponding to the chosen kernel approximates the drift sufficiently well, we derive a quantitative bound for the components' $L^2$ risks depending on their depth in the structure graph. For practical use, we also provide a cross-validation procedure to tune hyper-parameters. 

We illustrate our method on simulated data from population growth in a random environment and from a sigmoidal drift in seven variables. Further, we note that our plug-in approach aligns well with low frequency samples from a single ergodic process and performs well 
%observe favorable numerical performance 
in regimes with low to moderate autocorrelation. Finally, we switch perspective from recovering to creating $\mathbf{b}$, asking when a given density $p$ is the stationary density of a diffusion with DAG-structured drift. One motivation for this question is that DAG-structured drifts induce irreversible processes that converge faster to their stationary distribution than the score-based drifts used in generative diffusion models. We prove that any sufficiently smooth density $p$ can be generated by the complete DAG and illustrate accelerated convergence in a simple setting. 

\paragraph{Notation}
We boldface vectors and matrices in lower and upper case, respectively.
%both deterministic and random ones. 
For $\mathbf{a}\in\mathbb{R}^d$, $\mathbf{A}\in\mathbb{R}^{d\times d}$ and index sets $S, S'\subset\{1,\ldots,d\}$, we write $\mathbf{a}_S := (a_i)_{i\in S}$, $\mathbf{A}_{S,S'}:=(A_{ij})_{i\in S, j\in S'}$, and $\mathbf{a}_{-S}:= (a_i)_{i\notin S}$. For Lebesgue densities $p\colon\mathbb{R}^d\to\mathbb{R}_{+}$ we define the marginal $p_S(\mathbf{x}_S):=\int_{\mathbb{R}^{d-|S|}}p(\mathbf{x})\,\dd{\mathbf{x}_{-S}}$. For $a_1,a_2$ reals, $a_1 \vee a_2:=\max\{a_1,a_2\}$ and $a_1\wedge a_2:=\min\{a_1,a_2\}$. The vector $q$-norm, $q\in[1,\infty]$, is written $\|\mathbf{a}\|_q$, with  $\|\mathbf{a}\|:=\|\mathbf{a}\|_2$. 

We use the Lebesgue spaces $L^q(\mathbb{R}^k, \lambda)$ and abbreviate to $L^q(\mathbb{R}^k)$ for Lebesgue's measure or to $L^q(\lambda)$ if the ambient space is clear from $\lambda$. We also use the local versions $L^q_\mathrm{loc}$. Let $L^q(\mathbb{R}^{k_1}, \mathbb{R}^{k_2}, \lambda)$ denote functions $\mathbf{a}\colon\mathbb{R}^{k_1}\to\mathbb{R}^{k_2}$ with component functions $a_i\in L^{q}(\mathbb{R}^{k_1},\lambda)$ for all $i\leq k_2$. We use the Sobolev spaces $W^{n,q}(\mathbb{R}^k, \lambda)$ and write $\nabla$ for the (weak) gradient, $\nabla\cdot\mathbf{a} := \sum_{i=1}^{k} \partial_i a_i $ for the divergence, $\nabla^2$ for the Hessian matrix, and $\Delta$ for the Laplace operator. All derivative operators may only be relative to a subset of variables, e.g., $\nabla^2_S\,a := (\partial_{ij}\,a)_{i,j\in S}$. When $a$ is a function of one variable, we write $a':=\nabla a$. We write $C^n(\mathbb{R}^k)$ (and $C^n_0(\mathbb{R}^k)$) for $n$-times continuously differentiable (and compactly supported) functions where $n\in \mathbb{N}\cup \{\infty\}$. Let $C^{n,0}(\mathbb{R}^d\times\mathbb{R}^d)$ denote the space of continuous functions $a(\mathbf{x},\mathbf{y})$ such that $\partial_{\mathbf{x}_{\bs{\alpha}}}a$ exists and is continuous for all multi-indices $|\bs{\alpha}|\leq n$. We denote function composition by $a\circ b$. Finally, for Hilbert spaces $\mathcal{A},\mathcal{B}$,
$\mathcal{L}(\mathcal{A},\mathcal{B})$ denotes bounded linear operators $F\colon \mathcal{A}\to \mathcal{B}$, whose adjoints we denote by $F^*$.

\section{Related literature}
The idea to explain observations by repeated application of a causal mechanism over a long unobserved time window has parallels in the causality literature. For instance, it has been considered to give meaning to cyclic structural equation models (SCMs) as surveyed by \citet[Sect. 2.3.3]{Peters2017CausalityBook}. % \citep[Sect. 2.3.3]{Lauritzen2002, Lacerda2008, Hyttinen2012, Peters2017CausalityBook}, 
If convergence can be established, \citet{bongers2018random} %Mooij2013
prove that equilibria of random differential equations also follow SCMs. Although conceptually similar, the quantitative situation in this paper is different since our diffusion process does not converge to a constant almost surely. Therefore, a sample from the stationary distribution reveals information on common time history between coordinates. As a result, conditional independence relations are significantly reducted in the stationary distribution \citep{Boege2025LyapunovCI}, which generally is \emph{not} Markov with respect to the structure graph underlying the drift function.

A key theme of this paper is identifiability: characterizing what information on the causal mechanism is uniquely determined by the stationary distribution when no time history is available.
When the structure graph is known, it was shown by \citet{Dettling2022} that linear diffusion drifts can be fully recovered when the graph is free of 2-cycles; examples show generic identifiability for additional graphs. When the graph is an unknown DAG, the stationary distribution determines a graph equivalence class, which for linear drifts are finer than the Markov equivalence classes encountered for SCMs \citep{Amendola2025}. Additional interventional data can narrow down the drift even further \citep{Inglese2002, Rohbeck2024Bicycle}. Identifiability of linear drifts is often studied under the term graphical Lyapunov model in the literature, as the Lyapunov equation links drift matrix and covariance matrix of the Gaussian stationary distribution. We note that sparsely interacting diffusion processes also feature in literature from probability theory \citep{Lacker2023StationaryDiffusionOnTree}.

For drift estimation, maximum likelihood is an option as the Fokker-Planck equation yields the stationary density for a proposal drift. This approach is considered for parametric models by \citet{Pedretscher2019FokkerPlanckMaximumLikelihood}. %\citet{Kaltenbacher2018FokkerPlanckMaximumLikelihood}
Similarly, one can optimize distributional metrics between the observed distribution and the solution to the Fokker-Planck equation \citep{Botvinick2023}.
Setting computational demands in high dimensional parameter spaces aside, the main challenge is the curse of dimensionality when solving the stationary Fokker-Planck equation in dimensions greater than three or four \citep[Fig. 10]{Sun2014StatFP}. The alternative, to sample from the SDE until reaching steady state, also suffers from the curse of dimensionality in general \citep{Amorino2023InvariantDensityEstimation}. We are only aware of the work by \citet{Lorch2024StationaryDiffusions} that scales to higher dimensions by optimizing a kernelized surrogate loss. Common to the aforementioned work is the need to optimize a complicated nonlinear function in the proposal drift, usually via gradient descent. In contrast, the method we develop here targets the mean squared error between proposal and ground truth drift directly; see Section \ref{subsec:estimator_derivation}. This is very convenient for the theoretical analysis controlling the $L^2(p)$ distance between estimated and ground truth drift.

Our work also has connections with score matching. First, we find drift learning to be equivalent to score learning in one dimension when the diffusivity is constant. Second, our estimator naturally extends nonparametric score learning methods \citep{Zhou2020ScoreEstimators}. Third, the score solves a conceptually related problem in generative diffusion modeling: learning a drift inducing a given stationary distribution without observing any time history \citep{Song2019}.

\section{Model identifiability}
\subsection{Diffusion processes with structured drifts and the Fokker-Planck equation}
Consider an individual with attribute vector $\mathbf{x}\in\mathbb{R}^d$, which evolves over time according to the time-homogeneous stochastic differential equation (SDE) \begin{equation}\label{eq:sde}
    \dd{\mathbf{x}(t)} = \mathbf{b}\big(\mathbf{x}(t)\big)\,\dd{t} + \bs{\Sigma}\big(\mathbf{x}(t)\big)\,\dd{\mathbf{w}(t)}
\end{equation} understood in Ito-sense, where $\mathbf{b}\colon\mathbb{R}^d\to\mathbb{R}^d$ is the drift function, $\mathbf{w}(t)$ is a standard $d$-dimensional Wiener process, and  $\bs{\Sigma}\colon\mathbb{R}^d\to\mathbb{R}^{d\times d}$ is a full rank matrix-valued function. Equation \eqref{eq:sde} can be interpreted as $\mathbf{x}(t+\Delta t) = \mathbf{x}(t) + \mathbf{b}(\mathbf{x}(t))\Delta t + \sqrt{\Delta t} \cdot\bs{\Sigma}(\mathbf{x}(t))\, \mathbf{Z}(t)$ for arbitrarily small $\Delta t > 0$, where $\mathbf{Z}(t)$ is standard Gaussian. 

We are interested in drifts adhering to a causal structure; more precisely to a directed acyclic graph (DAG) with self-loops \citep{Dettling2022}. That is, we define a DAG to be a collection of directed edges $i\to j$ between nodes $i,j\in\{1,\dots,d\}$ 
with exactly one path from node $i$ back to itself, namely the self-loop $i\to i$. Let $\Pa{i}:=\{j : j\to i\}$ be the parents of node $i$ and let $\pa{i}:= \{j\neq i : j\to i\}$ denote the proper parents. In particular, $\Pa{i}=\pa{i}\cup\{i\}$. Let $\anc{i}$ denote the proper ancestors, i.e., nodes with a path to $i$ but excluding $i$ itself.

\begin{restatable}{definition}{defDriftStructured}\label{defDriftStructured} 
    A function $f\colon\mathbb{R}^d\to\mathbb{R}$ \emph{only depends on $\mathbf{x}_S$} for some index set $S\subset \{1,\ldots,d\}$ if $f = f\circ \pi_S$, where $\pi_S$ denotes the projection $\pi_S(\mathbf{x})=\mathbf{x}_S$.

    A function $\mathbf{b}\colon\mathbb{R}^d\to\mathbb{R}^d$ is \emph{structured according to a DAG} $\mathcal{D}$ on nodes $\{1,\ldots,d\}$, if every component function $b_i$ only depends on $\mathbf{x}_{\Pa{i}}$.
\end{restatable}
In particular, if the drift $\mathbf{b}$ is DAG-structured, each component function $b_i$ can depend on $x_i$ due to self-loops. The reasoning is that linear drifts $\mathbf{b}(\mathbf{x})=\mathbf{B}\mathbf{x}$ given by matrix $\mathbf{B}\in\mathbb{R}^{d\times d}$ can only be guaranteed to admit a stationary distribution if all eigenvalues have a negative real part. However, DAG-structured matrices can only satisfy this condition with a non-zero diagonal, since they are lower triangular after a suitable permutation of the indices. 

For a causal interpretation, it would be natural to assume that $(\Sigma_{i, k})_{k=1}^d$, and thus the entire SDE equation for $\dd{x_i}(t)$, only depends on the $x_j$ with $j\to i\in\mathcal{D}$. However, as subsequent results are not affected, we keep the structure of the matrix function $\bs{\Sigma}(\mathbf{x})$ unrestricted.

We consider processes $\mathbf{x}(t)$ that approach an equilibrium distribution as $t\to\infty$. Under suitable assumptions on $\mathbf{b}$ and $\mathbf{D}$ \citep[Sect. 4]{KhasminskiiBookStochasticProcesses}, % Lemma 4.16 or better the whole section
this equilibrium distribution admits a probability density $p$ with respect to Lebesgue measure that satisfies the \emph{stationary Fokker-Planck equation} \begin{equation}\label{eq:stat_FP_classic}
    \frac{1}{2} \sum_{i,j=1}^d \partial_i \partial_j (p\,D_{ij}) =\nabla\cdot( \mathbf{b}\,p),
\end{equation} where the \emph{diffusivity parameter} $\mathbf{D}\colon\mathbb{R}^d\to\mathbb{R}^{d\times d}$ is defined by $\mathbf{D}(\mathbf{x}) := \bs{\Sigma}(\mathbf{x}) \bs{\Sigma}(\mathbf{x})^T$. Hereafter, we omit ``stationary" and refer to \eqref{eq:stat_FP_classic} as the Fokker-Planck equation. Before discussing how to interpret this partial differential equation (PDE) in general, we consider a simple setting used as illustration throughout the paper: 

\begin{restatable}[Running example: 2-path]{example}{runExample}\label{runExample} Consider an object with $d=2$ attributes $x_1, x_2$, e.g., concentrations of two mRNAs in a cell, and assume regulation according to \rotatebox[origin=c]{270}{$\circlearrowright$}\hspace{0.1cm}$1 \to 2$ \rotatebox[origin=c]{90}{$\circlearrowleft$}, meaning the drift is of the form $\mathbf{b}(x_1,x_2) = (b_1(x_1), b_2(x_1, x_2))\in\mathbb{R}^2$.

Let $p\in C^2(\mathbb{R}^2)$ be a bounded % Need for Khasminskii Lemma 4.16
probability density. Consider $\bs{\Sigma}$ to be $\sqrt{2}$ times the constant identity, meaning the left hand side of \eqref{eq:stat_FP_classic} becomes the Laplacian $\Delta p$. Assume $\mathbf{b}\in C^1(\mathbb{R}^2, \mathbb{R}^2)$ and that the Fokker-Planck relation $\Delta p = \nabla\cdot (\mathbf{b}\,p)$ holds, i.e., \begin{equation}\label{eq:run_ex_statFP}
    \partial_{11}p(\mathbf{x}) + \partial_{22}p(\mathbf{x}) = \partial_1 (b_1 \, p)(\mathbf{x}) + \partial_2 (b_2\,p)(\mathbf{x}) \qquad\forall\mathbf{x}\in\mathbb{R}^2.
\end{equation}
If $\mathbf{b}$ has bounded first derivatives, the SDE $\dd{\mathbf{x}(t)} = \mathbf{b}(\mathbf{x}(t))\dd{t} + \sqrt{2}\dd{\mathbf{w}(t)}$ has a solution $(\mathbf{x}(t))_{t\geq 0}$ for arbitrary initial $\mathbf{x}(0)\in\mathbb{R}^2$. % e.g. Oskendal Theorem 5.2.1
If this solution is positively recurrent, the density of $\mathbf{x}(t)$ approaches $p$ at every point as $t\to\infty$ \citep[Lemma 4.16/4.17]{KhasminskiiBookStochasticProcesses}.
\end{restatable}

To minimize assumptions on the drift, we consider the following weak formulation.

\begin{definition}\label{def:stat_FP}
    We say $\mathbf{b}\in L^1_\mathrm{loc}(\mathbb{R}^d, \mathbb{R}^d, p)$ solves the Fokker-Planck equation with (Lebesgue) density $p$ and diffusivity $\mathbf{D}\in L^1_\mathrm{loc}(\mathbb{R}^d, \mathbb{R}^{d\times d}, p)$, when \begin{equation}\label{eq:stat_FP}
    \int_{\mathbb{R}^d} \mathbf{b}^T \nabla \varphi \, \dd{p} = \int_{\mathbb{R}^d}  \frac{1}{2}\langle \mathbf{D}, \nabla^2\varphi\rangle_F\,\dd{p} \qquad\forall \varphi\in C^\infty_0(\mathbb{R}^d),
\end{equation} where $\langle \mathbf{A}, \mathbf{B}\rangle_F = \tr(\mathbf{A}^T\mathbf{B}) = \sum_{ij} A_{ij}B_{ij}$ is the Frobenius product of matrices $\mathbf{A}, \mathbf{B}$.
\end{definition}

If $\mathbf{b}$, $p$, and $\mathbf{D}$ are twice differentiable in the classical sense, integration by parts and the  fundamental lemma of the calculus of variations recovers the classical equation \eqref{eq:stat_FP_classic}.

%That the drift function satisfies a Fokker-Planck equation \eqref{eq:stat_FP} is the central assumption of our theoretical analysis and estimation strategy. One could weaken the Fokker-Planck equation further and allow a broader class of admissible drifts \citep{Bogachev2015BookFP}, however proving existence of a corresponding diffusion process converging to an equilibrium distribution usually requires stronger assumptions on the drift than integrability anyway \citep{Haesung2021FullSDEPackageUnderLocalIntegrability}, and this is the main data generating process we have in mind.
\subsection{The one dimensional case}\label{subsec:identif_univ}
We first consider recovering a univariate drift function $b\colon\mathbb{R}\to\mathbb{R}$ from a univariate stationary density $p$. This part provides intuition and serves as a base case for general DAGs.

Let $p\in C^2(\mathbb{R})$ be a positive density that is  stationary for the univariate SDE \begin{equation}\label{eq:univ_sde}
    \dd{x}(t) = b(x(t))\dd{t} + \sigma\dd{w}(t)
\end{equation} with drift $b\in C^1(\mathbb{R})$ and a constant $\sigma>0$. Then, the Fokker-Planck equation \eqref{eq:stat_FP_classic} reads \begin{equation*}
    \frac{1}{2}\sigma^2 p'' = (bp)' \quad\iff\quad b = \frac{\sigma^2}{2}\frac{p'}{p} + \frac{c}{p} \quad\text{for} \; c\in\mathbb{R},
\end{equation*} using the fundamental theorem of calculus. Note that $\tfrac{p'}{p}=(\log p)'$ is the score function of $p$. Hence, the Fokker-Planck equation determines the drift up to a one-parameter family $(b_c)_{c\in\mathbb{R}}$. 

It is intuitive that $b_{c=0}$ is the only sound drift in the family: $\nicefrac{c}{p}$ grows strongly for $c\neq 0$, especially when $p$ is light-tailed. In fact, the growth is so strong that $x(t)$ is forced to leave the state space $\mathbb{R}$ and attain one of the symbolic values $\pm\infty$ at a finite time, called an \emph{explosion}. For univariate SDEs, Feller's explosion test (Lemma \ref{lemmaFellerExplosionTest}) yields:

\begin{restatable}{lemma}{lemmaUnivExplodes}
    The solution of the SDE \eqref{eq:univ_sde} with drift $b_c(x) = \frac{\sigma^2}{2} \frac{p'}{p} + \frac{c}{p}$ explodes with positive probability if $c\neq 0$.
\end{restatable}

An explosion in the process $x(t)$ strongly contradicts our goal to model objects in a stable equilibrium. Especially, since one can show that $x(t)$ following \eqref{eq:univ_sde} with $b_c$ for $c\neq 0$ has a positive probability of having exploded before \emph{any} time $T>0$ \citep{Karatzas2016ExplosionTimeDistribution}. Consequently, for our purposes, $b = \tfrac{\sigma^2}{2}\log(p)'$ is the only reasonable drift that yields $p$ as a stationary density.

When considering an $x$-dependent diffusivity $\sigma$ or multivariate diffusions, exploding drifts are much harder to characterize sharply. % \citep[Ch. 12]{Bhattacharya2023DiffusionBook}. 
In this paper, we consider drifts from the class of functions absolutely integrable with respect to the density $p$, i.e., $b\in L^1(\mathbb{R},p)$. This is a mild restriction, yet strong enough to rule out explosions under suitable assumptions on the diffusivity $\mathbf{D}$ \citep{Haesung2021FullSDEPackageUnderLocalIntegrability, Hwang2005AcceleratingDiffusions}. %Stannat1999DirichletOperatorsL1
The assumption $b\in L^1(\mathbb{R},p)$ consistently also singles out the score from the class $(b_c)_{c\in\mathbb{R}}$: when $\int_{\mathbb{R}}|p'(x)|\dd{x}<\infty$, then $\int_\mathbb{R}|b_c|\dd{p} < \infty$ if and only if $c=0$. % If it were integrable for $c\neq 0$, the difference with the score would also be integrable, but it is not.
The following Lemma demonstrates that the assumption $b\in L^1(\mathbb{R},p)$ is also sufficient to recover the drift for $x$-dependent diffusivity $\sigma$.

\begin{restatable}{lemma}{lemmaUniqueOneDim}\label{lemmaUniqueOneDim}
    Let $b^*, b\in L^1(\mathbb{R}, p)$ both solve a Fokker-Planck equation \eqref{eq:stat_FP} with the same density $p>0$ and diffusivity $D\in L^1_\mathrm{loc}(\mathbb{R}, p)$. Then, $b = b^*$ almost everywhere.
\end{restatable}

\subsection{General DAGs}
Let the drift $\mathbf{b}\colon\mathbb{R}^d\to\mathbb{R}^d$ be structured according to a DAG $\mathcal{D}$. Building on the framework developed for the univariate case in the previous section, we show that requiring $\mathbf{b}\in L^1(\mathbb{R}^d, \mathbb{R}^d, p)$ is enough to recover the drift from the stationary density $p$. This integrability assumption was motivated by its ability to effectively rule out explosions in the associated diffusion process. 

The main result at the end of this section is derived via induction over a topological ordering of $\mathcal{D}$. By shuffling the nodes, we may assume the ordering to be $1,\ldots, d$ without loss of generality. This means any edge $i_1\to i_2$ in the DAG satisfies $i_1\leq i_2$.

It is important for the $k$-th induction step to have an equation involving only the components $b_i$ with $i\leq k+1$. We state such an equation below and afterwards discuss how it simplifies in the running Example \ref{runExample}.

\begin{restatable}{lemma}{lemmaMargRecursion}\label{lemmaMargRecursion}
    Let the $\mathcal{D}$-structured drift $\mathbf{b}\in L^1(\mathbb{R}^d,\mathbb{R}^d, p)$ solve the Fokker-Planck equation \eqref{eq:stat_FP} with density $p$ and diffusivity $\mathbf{D}\in L^1(\mathbb{R}^d, \mathbb{R}^{d\times d}, p)$. For any index $i\leq d$ and any $\varphi=\tilde{\varphi}\circ\pi_{\Pa{i}}$ with $\tilde{\varphi}\in C^\infty_0(\mathbb{R}^{|\Pa{i}|})$ it holds that \begin{equation}\label{eq:lemmaMargRecursion}
        \int_{\mathbb{R}^{|\Pa{i}|}} b_i \partial_i \varphi\, \dd{p_{\Pa{i}}} = - \int_{\mathbb{R}^d} \mathbf{b}_{\pa{i}}^T \nabla_{\pa{i}} \varphi \, \dd{p} + \int_{\mathbb{R}^d}\frac{1}{2} \langle\mathbf{D}_{\Pa{i},\Pa{i}}, \nabla_{\Pa{i}}^2 \varphi\rangle_F\,\dd{p}.
    \end{equation}
\end{restatable}
In the $i$-th induction step, $\mathbf{b}_{\pa{i}}$ is identified by assumption. To see that the right hand side of \eqref{eq:lemmaMargRecursion} suffices to determine $b_i$, consider $i=2$ in the running Example \ref{runExample}, for which equation \eqref{eq:lemmaMargRecursion} simplifies to the Fokker-Planck equation $\partial_2(b_2p) = \Delta p - \partial_1(b_1 p) $. Assume this equation had two solutions $b_2, b_2^*\in C^1(\mathbb{R}^2)\cap L^1(\mathbb{R}^2, p)$ for the same right hand side. Taking the difference of both equations implies $\partial_2 (\delta\, p) = 0$, where we defined $\delta:= b_2 - b_2^*$. By the fundamental theorem of calculus, $(\delta p)(\mathbf{x})=f(x_1)$ for some function $f$. The assumptions on $b_2,b_2^*$ imply $\delta\in L^1(\mathbb{R}^2, p)$, however \begin{equation*}
    \int_{\mathbb{R}^2}|\delta|\,\dd{p} = \int_{\mathbb{R}^2} |f(x_1)| \dd{x_1}\,\dd{x_2} = \int_\mathbb{R} \|f\|_{L^1(\mathbb{R})} \dd{x_2}
\end{equation*} is finite only if $\|f\|_{L^1(\mathbb{R})}=0$, implying $\delta = 0$ if $p>0$ and therefore $b_2 = b_2^*$. The following main result generalizes this proof idea to the equations \eqref{eq:lemmaMargRecursion}.

\begin{restatable}{theorem}{theoremIdentifiability}\label{theoremIdentifiability}
    Assume $\mathbf{b}^*,\mathbf{b}\in L^1(\mathbb{R}^d, \mathbb{R}^d,p)$ both solve a Fokker-Planck equation \eqref{eq:stat_FP} with the same density $p>0$ and same diffusivity $\mathbf{D}\in L^1(\mathbb{R}^d, \mathbb{R}^{d\times d}, p)$. 
    When $\mathbf{b},\mathbf{b}^*$ are structured according to a common DAG $\mathcal{D}$, then $\mathbf{b}=\mathbf{b}^*$ Lebesgue-almost everywhere.
\end{restatable}

To summarize: the stationary distribution, the structure DAG $\mathcal{D}$, and the diffusivity function $\mathbf{D}$ together are sufficient to recover the drift function. Next, we construct a practical estimator for this task when a sample of the stationary distribution is available.

\section{Estimation}\label{sec:estimation}
\subsection{Reproducing kernel Hilbert spaces (RKHS)}\label{subsec:RKHS_theory}
Kernel methods, popularized by support vector machines, are applied under different philosophies in the literature. Here, we view kernels as generating nonparametric function spaces, from which we select elements to approximate the drift function. Their advantages are closed-form optima and the ability of the kernel to absorb partial derivatives from the drift in the Fokker-Planck equation.

We call $k\colon\mathbb{R}^k\times\mathbb{R}^k\to\mathbb{R}$ a \emph{kernel} if $k$ is symmetric, i.e., $k(\mathbf{x},\mathbf{y})=k(\mathbf{y}, \mathbf{x})$, and if $k$ is positive definite, meaning $\sum_{i,j=1}^m \alpha_i\alpha_j k(\mathbf{x}_i, \mathbf{x}_j)\geq 0$ for all $\bs{\alpha}\in\mathbb{R}^m $ and $\mathbf{x}_1,\ldots,\mathbf{x}_m\in\mathbb{R}^k$. An example is the \emph{Gaussian} kernel $\exp(-\tfrac{1}{2\gamma^2}\|\mathbf{x} - \mathbf{y}\|_2^2)$ with bandwidth $\gamma>0$. 

The function $k(\mathbf{x},\cdot)\colon\mathbb{R}^k\to\mathbb{R}, \mathbf{y}\mapsto k(\mathbf{x},\mathbf{y})$ is called the \emph{feature map} corresponding to $\mathbf{x}\in\mathbb{R}^k$. On the linear span of the feature maps, define an inner product via $\langle k(\mathbf{x},\cdot), k(\mathbf{y},\cdot)\rangle := k(\mathbf{x}, \mathbf{y})$. The \emph{RKHS} given by the kernel $k$ is the completion  of the linear span with respect to $\langle\cdot,\cdot\rangle$ and denoted by $\mathcal{H}$.  % See DeVito or Pereverzyev.
The functions $f\colon\mathbb{R}^k\to\mathbb{R}$ it contains have the \emph{reproducing}  property $\langle f,k(\mathbf{x},\cdot)\rangle_\mathcal{H} = f(\mathbf{x})$. The kernel choice determines smoothness and many other properties of functions in $\mathcal{H}$.

For $k$ continuous and bounded, the inclusion $J\colon\mathcal{H}\to L^2(p)$ is continuous for any probability measure $p$ on $\mathbb{R}^k$, and $\|J\|_{\mathcal{L}(\mathcal{H}, L^2(p))} \leq \|k\|_\infty$. % \citep{Steinwart2008SVM} Lemma 4.33, Theorem 4.26
The inclusion into $L^2(p)$ in particular allows us to use mean squared error. The adjoint $J^*\colon L^2(p) \to \mathcal{H}$, often called \emph{integral operator}, is given by: \begin{equation*}
    (J^*f) (\mathbf{y}) = \int_{\mathbb{R}^k} k(\mathbf{x},\mathbf{y})f(\mathbf{x})\dd{p}(\mathbf{x}).
\end{equation*} We frequently use $L:=J^*J$. Under the above assumptions, $L$ is compact and self-adjoint, with nuclear norm bounded by $\|k\|_\infty$ \citep{Steinwart2008SVM}. % \citep{Steinwart2008SVM} 4.27, nuclear implies Hilbert Schmidt

During our statistical analysis, we approximate $J, J^*$ and $L$ by sample versions. Let $\mathbf{x}_1,\ldots,\mathbf{x}_n\in\mathbb{R}^k$, and let $E^n$ be $\mathbb{R}^n$ with the inner product $\langle\mathbf{v},\mathbf{w}\rangle_{E^n}:= \tfrac{1}{n}\sum_{i=1}^n v_i w_i$. Following \citet{Vito2005}, we define $J_\mathbf{x}\colon\mathcal{H}\to E^n, f\mapsto (f(\mathbf{x}_1),\ldots, f(\mathbf{x}_n))$. Then, \begin{align*}
    J_\mathbf{x}^* \mathbf{v} &= \frac{1}{n}\sum_{i=1}^n v_ik(\mathbf{x}_i, \cdot), & \hat{L}f := J_\mathbf{x}^* J_\mathbf{x} f &=\frac{1}{n}\sum_{i=1}^n f(\mathbf{x}_i) k(\mathbf{x}_i,\cdot).
\end{align*}

\subsection{Idea and setup}\label{subsec:estimation_setup}
Let the probability density $p$ on $\mathbb{R}^d$ solve the Fokker-Planck equation with unknown drift $\mathbf{b}^*\in L^2(\mathbb{R}^d, \mathbb{R}^d,p)$ but known diffusivity $\mathbf{D}\in L^2(\mathbb{R}^d, \mathbb{R}^{d\times d}, p)$. Assume $\mathbf{b}^*$ is structured according to a known DAG $\mathcal{D}$. As we show below, we can actually regress on the unknown $\mathbf{b}^*$ using only suitable expectations under the data distribution $p$. This is analogous to score matching, where one implicitly regresses on the true score. 

We approximate each drift component $b_i^*$ by functions from an RKHS $\mathcal{H}_i\subset L^2(\mathbb{R}^d,p)$. Assuming each node in $\mathcal{D}$ has a self-loop to avoid case distinctions, we take the corresponding kernel $k_i(\mathbf{x},\mathbf{y})$ only depending on $(\mathbf{x}_{\Pa{i}}, \mathbf{y}_{\Pa{i}})$. Consider the infinite sample Tikhonov regression \begin{equation}\label{eq:estimator_pop_approx}
    b_{i,\lambda} := \argmin_{b\in\mathcal{H}_i} \|b - b_i^*\|_{L^2(p)}^2 + \lambda\|b\|_{\mathcal{H}_i}^2
\end{equation} for some regularization strength $\lambda>0$. We have that \begin{align}\label{eq:estimation_setup1}
    b_{i,\lambda} & = -(L_i + \lambda I)^{-1} \zeta_i, & \zeta_i(\mathbf{y}) := -\int_{\mathbb{R}^d} k_i(\mathbf{x},\mathbf{y}) b_i^*(\mathbf{x})\dd{p(\mathbf{x})}\in\mathcal{H}_i,
\end{align} where $I$ denotes the identity on $\mathcal{H}_i$. We now free $\zeta_i$ from the unknown $b_i^*$, meaning we can compute the regression function $b_{i,\lambda}$ without requiring $b_i^*$.

Noticing the similarity between $\zeta_i$ and the left hand side of Lemma \ref{lemmaMargRecursion}, we consider functions $\mathcal{K}_{i,\mathbf{y}}\colon\mathbb{R}^d\to\mathbb{R}$ indexed by $\mathbf{y}\in\mathbb{R}^d$ with the property that $\partial_i \mathcal{K}_{i,\mathbf{y}} = k_i(\cdot,\mathbf{y})$:

\begin{restatable}{lemma}{lemmaZetaUnbiased}\label{lemmaZetaUnbiased}
    Assume $p$ is bounded, let $i\leq d$, and let $\mathbf{y}\in\mathbb{R}^d$. If  $\mathcal{K}_{i,\mathbf{y}}\in W^{2,\infty}(\mathbb{R}^d)$ only depends on $\mathbf{x}_{\Pa{i}}$ and $\partial_i \mathcal{K}_{i,\mathbf{y}}=k_i(\cdot,\mathbf{y})$, % Understood almost everywhere, since K is in a Sobolev space
    then $\zeta_i(\mathbf{y})=\ex[\mathbf{x}\sim p]{z_i(\mathbf{x}, \mathbf{y})}$, where \begin{equation*}
        z_i(\mathbf{x}, \mathbf{y}) :=  \mathbf{b}_{\pa{i}}^*(\mathbf{x})^T \nabla_{\pa{i}} \mathcal{K}_{i,\mathbf{y}}(\mathbf{x}) - \frac{1}{2}\langle\mathbf{D}_{\Pa{i},\Pa{i}}(\mathbf{x}), \nabla_{\Pa{i}}^2 \mathcal{K}_{i,\mathbf{y}}(\mathbf{x})\rangle_F.
    \end{equation*}
\end{restatable}This means we can express $\zeta_i$ as an expectation under $p$ only involving quantities strictly before $i$ in the topological order induced by the DAG $\mathcal{D}$. For $\pa{i}=\emptyset$ and $\mathbf{D}$ the constant identity matrix, $z_i$ reduces to \citet{Zhou2020ScoreEstimators}. Following this work, we replace expectations with sample averages and consider an iterative plug-in estimator, which we formally introduce in Section   \ref{subsec:estimator_derivation}. We consider alternative regularization strategies to Tikhonov in Section \ref{subsec:alternative_reg}. The core idea also works without any regularization, see Section \ref{subsec:cv}.

\subsection{Estimator derivation}\label{subsec:estimator_derivation}
Continuing in the setup from Section \ref{subsec:estimation_setup}, the first step is to choose kernels $k_i(\mathbf{x}, \mathbf{y})$ depending on $(\mathbf{x}_{\Pa{i}},\mathbf{y}_{\Pa{i}})$ to approximate $b_i^*$ for all $i\leq d$. Our method requires the anti-derivative of $k_i$ with respect to $x_i$, i.e., $\mathcal{K}_{i,\mathbf{y}}$ such that $\partial_i \mathcal{K}_{i,\mathbf{y}}= k_i(\cdot,\mathbf{y})$ for almost all $\mathbf{y}\in\mathbb{R}^d$. We can compute $\mathcal{K}_{i,\mathbf{y}}$ analytically for some popular kernels, see Table \ref{tab:kernels_with_antideriv}. Here, $\Phi$ denotes the standard Gaussian cdf. For derivative kernels, see \citet[Def. 4.35]{Steinwart2008SVM}.

\begin{table}[]
\renewcommand{\arraystretch}{1.5}
\caption{Example kernels with corresponding $\mathcal{K}_{i,\mathbf{y}}$.}
    \centering
    \begin{tabular}{c c c}
        Name & $k(\mathbf{x}, \mathbf{y})$  & $\mathcal{K}_{i,\mathbf{y}}(\mathbf{x})$\\\hline
         Gaussian & $\exp(-\tfrac{\|\mathbf{x} - \mathbf{y}\|_2^2}{2\gamma^2})$ & $\gamma\sqrt{2\pi}\,\Phi(\tfrac{x_i - y_i}{\gamma})\cdot k(\mathbf{x}_{-i}, \mathbf{y}_{-i})$ \\
         Linear & $\mathbf{x}^T\mathbf{y} + c$ & $\tfrac{1}{2}x_i^2y_i  + x_i(\mathbf{x}_{-i}^T\mathbf{y}_{-i} + c)$\\
         Sigmoid & $\tanh(\alpha\, \mathbf{x}^T\mathbf{y} + c)$ & $\tfrac{1}{\alpha y_i}\log \cosh(\alpha\,\mathbf{x}^T\mathbf{y} + c)$ \\
         Derivative & $\partial_{x_i}\partial_{y_i}\tilde{k}(\mathbf{x}, \mathbf{y})$ & $\partial_{y_i}\tilde{k}(\mathbf{x}, \mathbf{y})$
    \end{tabular}
    
    \label{tab:kernels_with_antideriv}
\end{table}

Let $\mathbf{x}_1,\ldots,\mathbf{x}_n\in\mathbb{R}^d$ be samples from the density $p$. We compute the drift estimate $\hat{\mathbf{b}}:=(\hat{b}_1,\ldots,\hat{b}_d)$ component-wise along the causal order of $\mathcal{D}$. The central object in each step is $\hat{\zeta}_i\colon\mathbb{R}^d\to\mathbb{R}$, which approximates $\zeta_i$. Inspired by Lemma \ref{lemmaZetaUnbiased}, we use already learned drift components $(\hat{b}_j)_{j\in\pa{i}}$ to approximate $z_i$ via \begin{equation}\label{eq:z_hat}
    \hat{z}_i(\mathbf{x},\mathbf{y}) :=  \hat{\mathbf{b}}_{\pa{i}}(\mathbf{x})^T \nabla_{\pa{i}} \mathcal{K}_{i,\mathbf{y}}(\mathbf{x}) - \frac{1}{2}\langle\mathbf{D}_{\Pa{i},\Pa{i}}(\mathbf{x}), \nabla_{\Pa{i}}^2 \mathcal{K}_{i,\mathbf{y}}(\mathbf{x})\rangle_F,
\end{equation} and set $\hat{\zeta}_i(\mathbf{y}) := \tfrac{1}{n}\sum_{j=1}^n \hat{z}_i(\mathbf{x}_j, \mathbf{y})$. All that remains is to transform $\hat{\zeta}_i$ to the drift estimate $\hat{b}_i$. We present a Tikhonov regularized transformation here and point to Section \ref{subsec:alternative_reg} for alternative regularizers. 

The idea is to plug $\hat{\zeta}_i$ and $\hat{L}_i$ from Section \ref{subsec:RKHS_theory} into equation \eqref{eq:estimation_setup1}, i.e., $\hat{b}_i\approx -(\hat{L}_i + \lambda I)^{-1}\hat{\zeta}_i$. For reasons discussed below, we use the following explicit formula instead, which is based on an extended representer theorem \citep[C.4.1]{Zhou2020ScoreEstimators} % Althouhg there seems to be a sign issue in the second line: to get the inverse of (L_k + lambda I)^{-1}\hat{zeta}, they would need to subtract 2 <s, \hat{\zeta}> and not add it in the optimization problem. So they actually compute - (L_k + lambda I)^{-1}\hat{zeta} already as desired.
\begin{restatable}{lemma}{lemmaSampleTikhonovInverse}\label{lemmaSampleTikhonovInverse}
    Let $f\in\mathcal{H}_i, \mathbf{y}\in\mathbb{R}^d$, and let $\lambda>0$. Define $\mathbf{K}_i\in\mathbb{R}^{n\times n}$ via $K_{i,jl}:=k_i(\mathbf{x}_j, \mathbf{x}_l)$, set $\mathbf{k}_{i,\mathbf{y}}\in\mathbb{R}^n$ to $(\mathbf{k}_{i,\mathbf{y}})_j := k_i(\mathbf{y}, \mathbf{x}_j)$, and $\mathbf{f}\in\mathbb{R}^n$ to $f_j:=f(\mathbf{x}_j)$. Then, \begin{equation*}
        ((\hat{L}_i + \lambda I)^{-1} f)(\mathbf{y}) = \frac{1}{\lambda} f(\mathbf{y}) - \frac{1}{\lambda}\mathbf{k}_{i,\mathbf{y}} ^T (\mathbf{K}_i + n\lambda \mathbf{I})^{-1}\mathbf{f},
    \end{equation*} where $\mathbf{I}\in\mathbb{R}^{n\times n}$ is the identity matrix.
\end{restatable}

We can now state the entire Tikhonov regularized estimation procedure. \begin{restatable}{definition}{defTikhEst}\label{defTikhEst}
    For each $i$ ranging through the topological order of $\mathcal{D}$, set $\hat{\zeta}_i(\mathbf{y}):=\tfrac{1}{n}\sum_{j=1}^n \hat{z}_i(\mathbf{x}_j,\mathbf{y})$ with $\hat{z}_i$ from \eqref{eq:z_hat}, and define $\hat{\bs{\zeta}}_i\in\mathbb{R}^n$ via $(\hat{\bs{\zeta}}_i)_j := \hat{\zeta}_i(\mathbf{x}_j)$. We define \begin{equation*}
        \hat{b}_i(\mathbf{y}) := \frac{1}{\lambda_i}\mathbf{k}_{i,\mathbf{y}} ^T (\mathbf{K}_i + n\lambda_i \mathbf{I})^{-1}\hat{\bs{\zeta}}_i - \frac{1}{\lambda_i} \hat{\zeta}_i(\mathbf{y}),
    \end{equation*} where $\mathbf{k}_{i,\mathbf{y}},\mathbf{K}_i,\mathbf{I}$ are as in Lemma \ref{lemmaSampleTikhonovInverse} and $\lambda_i>0$ are the regularization strengths.
\end{restatable} The remainder of this chapter considers tuning $\lambda$, alternative regularization strategies, and theoretically analyses $\hat{b}_i$. Finally, we did \emph{not} formally apply $(\hat{L}_i + \lambda I)^{-1}$ to $\hat{\zeta}_i$, since it is difficult to guarantee $\hat{\zeta}_i\in\mathcal{H}_i$. As a start, $\mathbf{y}\mapsto\mathcal{K}_{i,\mathbf{y}}(\mathbf{x})$ from Table \ref{tab:kernels_with_antideriv} usually lies \emph{not} in $\mathcal{H}_i$ due to boundary value mismatch, and the partial derivatives of $\mathcal{K}_{i,\mathbf{y}}$ do not simplify the matter. For the Gaussian kernel, defining $\mathcal{K}_{i,\mathbf{y}}$ via integration from $0$ to $x_i$ solves the problem and $\hat{\zeta}_i\in\mathcal{H}_i$ is guaranteed \citep{Viktor2026}. We implement this rule in the package, however a proof for general kernels seems difficult.

\subsection{Alternative regularization methods}\label{subsec:alternative_reg}
The Tikhonov estimator approximates $b_{i,\lambda} := -(L_i + \lambda I)^{-1}\zeta_i$, see Section \ref{subsec:estimation_setup}. The purpose of adding $\lambda I$, namely to stabilize the otherwise ill-conditioned inversion of $L_i$, can be achieved with many alternative methods: the \emph{Spectral cut-off} regularization discards small eigenvalues of $L_i$ and the \emph{Landweber} method performs early-stopped gradient descent on $L_i f = -\zeta_i$. The $\nu$-\emph{Method} is an accelerated version of the Landweber method. \citet{Zhou2020ScoreEstimators} provide explicit formulas for each method when $L_i$ is replaced by $\hat{L}_i$, similar to Lemma \ref{lemmaSampleTikhonovInverse}. We implement all methods in the Python package by plugging our $\hat{\zeta}_i$ into these formulas.

We compare the regularizers empirically in Section \ref{sec:sim_and_ext}. Theoretically, an advantage of spectral cut-off regularization is that it provably adapts to regression functions admitting a simple expansion in the eigenbasis of $J_iJ_i^*$ \citep{Dicker2017}, but note our discussion in Section \ref{subsec:theory}.

\subsection{Theoretical analysis}\label{subsec:theory}
Let the Lebesgue density $p$ on $\mathbb{R}^d$ be bounded and solve the Fokker-Planck equation with unknown drift $\mathbf{b}^*\in L^\infty(\mathbb{R}^d, \mathbb{R}^d)$ but known diffusivity $\mathbf{D}\in L^\infty(\mathbb{R}^d, \mathbb{R}^{d\times d})$. Denote \begin{equation*}
    B := \|\mathbf{b}^*\|_{L^\infty(\mathbb{R}^d, \mathbb{R}^d)} \vee \frac{1}{2} \|\mathbf{D}\|_{L^\infty(\mathbb{R}^d, \mathbb{R}^{d\times d})}.
\end{equation*} Assume $\mathbf{b}^*$ is structured according to a known DAG $\mathcal{D}$. For $i\leq d$, let $k_i\in C^0(\mathbb{R}^d\times\mathbb{R}^d)$ be bounded and only depend on $(\mathbf{x}_{\Pa{i}}, \mathbf{y}_{\Pa{i}})$. Let $\mathcal{K}_{i,\mathbf{y}}(\mathbf{x})\in C^{2,0}(\mathbb{R}^d\times\mathbb{R}^d)$ % Require slices in W^{2,\infty} AND that the evaluations K_{i,x}(x) are well defined. 
only depend on $(\mathbf{x}_{\Pa{i}}, \mathbf{y}_{\Pa{i}})$ and satisfy $\partial_{\mathbf{x}_i}\mathcal{K}_{i,\mathbf{y}}(\mathbf{x})=k_i(\mathbf{x}, \mathbf{y})$. Further, assume \begin{equation*}
    \kappa := \max_{i\leq d}\,\max_{|\bs{\alpha}|\leq 2}\, \|\partial_{\mathbf{x}^{\bs{\alpha}}}\mathcal{K}_{i,\cdot}(\cdot)\|_{L^\infty(\mathbb{R}^d\times\mathbb{R}^d)} < \infty.
\end{equation*}

Let $\mathbf{x}_1,\ldots,\mathbf{x}_n$ be iid random variables following $p$. Let $\hat{\mathbf{b}}:=(\hat{b}_1,\ldots,\hat{b}_d)$ be the drift estimator from Definition \ref{defTikhEst} with Tikhonov regularization strengths $\lambda_1,\ldots,\lambda_d >0$. Define the training evaluation vectors $\mathbf{b}_i^*\in\mathbb{R}^n$ via $(\mathbf{b}_i^*)_j:=b_i^*(\mathbf{x}_j)$, and analogously $\hat{\mathbf{b}}_i,\bs{\zeta}_i, \hat{\bs{\zeta}}$, and $\mathbf{b}_{i,\lambda_i}\in\mathbb{R}^n$, where $b_{i,\lambda_i}$ is the population Tikhonov approximation \eqref{eq:estimation_setup1}. % Law is well-defined even if, e.g., b_i^* is only defined almost everywhere.

We first analyze how $\hat{\zeta}_i, \hat{\bs{\zeta}_i}$, and $\hat{L}_i$, the main estimated quantities in the definition of $\hat{b}_i$, concentrate around their ground truth counterparts. Here, we use $\hat{L}_i$ since it later allows us a unified treatment of $\mathbf{k}_{i,\mathbf{y}}$ and $\mathbf{K}$. Let $a_i := |\Pa{i}|+1$ be the number of arguments of $b_i^*$ plus one. A Hilbert-space valued Hoeffding inequality yields:
\begin{restatable}{lemma}{lemmaZetaGeneralizationBound}\label{lemmaZetaGeneralizationBound}
    For any $i\leq d$ it holds with probability at least $1-\delta$ that \begin{equation}
        \|\hat{\zeta}_i - \zeta_i\|_{L^2(p)} \leq \frac{a_i^2\kappa B \sqrt{2}\log^{1/2}\tfrac{2}{\delta}}{\sqrt{n}}  + \kappa \sum_{j\in\pa{i}} \|\hat{\mathbf{b}}_j - \mathbf{b}_j^*\|_{E^n}
    \end{equation}
\end{restatable} Note that $\|\hat{\zeta}_i - \zeta_i\|_{L^2(p)}$ still depends on the sample $\mathbf{x}_1,\ldots,\mathbf{x}_n$; it is the expected squared distance from $\zeta_i$ with respect to a fresh sample from $p$. We also bound the training error with a U-statistics approach: 
\begin{restatable}{lemma}{lemmaZetaTrainingBound}\label{lemmaZetaTrainingBound}
    Let $n\geq 3$. For any $i\leq d$ it holds with probability at least $1-\delta$ that \begin{equation*}
        \|\hat{\bs{\zeta}}_i- \bs{\zeta}_i\|_{E^n} \leq \frac{a_i^2\kappa B\sqrt{3}}{\sqrt{n}} + \frac{ a_i^2 \kappa B\cdot(\log^{1/2} n + \log^{1/2}\tfrac{2 \vee c_1}{\delta})}{\sqrt{c_2}\sqrt{n}} + \kappa\sum_{j\in\pa{i}} \|\hat{\mathbf{b}}_j - \mathbf{b}_j^*\|_{E^n}
    \end{equation*} for some global positive constants $c_1, c_2$.
\end{restatable}
Finally, another Hoeffding bound allows us to bound $\|\hat{L}_i - L_i\|_{\mathcal{L}(\mathcal{H}_i, \mathcal{H}_i)}$.
\begin{restatable}{lemma}{lemmaMercerKernelApprox}\label{lemmaMercerKernelApprox}
    For any $i\leq d$ it holds with probability at least $1-\delta$ that \begin{equation*}
        \|\hat{L}_i - L_i\|_{\mathcal{L}(\mathcal{H}_i, \mathcal{H}_i)} \leq \frac{2\kappa\sqrt{2}\log^{1/2} \tfrac{2}{\delta}}{\sqrt{n}}.
    \end{equation*}
\end{restatable}
Next, we bound the effect of deterministic deviations on the generalization error of $\hat{b}_i$. Since Lemmas \ref{lemmaZetaGeneralizationBound} and \ref{lemmaZetaTrainingBound} involve the training error, we also bound it here.  
\begin{restatable}{lemma}{lemmaDeterministicTrainingBound}\label{lemmaDeterministicTrainingBound}
    If $\|\hat{\bs{\zeta}}_i - \bs{\zeta}_i\|_{E^n} \leq \varepsilon_1$, $\|\hat{\zeta}_i - \zeta\|_{L^2(p)} \leq \varepsilon_2$ and $\|\hat{L}_i - L_i\|_{\mathcal{L}(\mathcal{H}_i, \mathcal{H}_i)} \leq \varepsilon_3$, then \begin{align}
        \|\hat{\mathbf{b}}_i - \mathbf{b}_i^*\|_{E^n} &\leq \|\mathbf{b}_{i,\lambda_i} - \mathbf{b}_i^*\|_{E^n} + \frac{\varepsilon_3 \|b_i^*\|_{L^2(p)}}{\lambda_i} + \frac{2\varepsilon_1}{\lambda_i},\\
        \|\hat{b}_i - b_i^*\|_{L^2(p)} & \leq \|b_{i,\lambda_i}\hspace{-0.1cm} - b_i^*\|_{L^2(p)} + \left(\frac{\varepsilon_3}{\lambda_i} + \frac{\varepsilon_3^{3/2}}{\lambda_i^{3/2}}\right)\|b_i^*\|_{L^2(p)} + \frac{\varepsilon_1 + \varepsilon_2}{\lambda_i} + \frac{\varepsilon_1 \sqrt{\varepsilon_3}}{\lambda_i^{3/2}}.
    \end{align}
\end{restatable} Inserting Lemmas \ref{lemmaZetaGeneralizationBound}, \ref{lemmaZetaTrainingBound}, and \ref{lemmaMercerKernelApprox} into Lemma \ref{lemmaDeterministicTrainingBound} yields the main result. We are interested in choosing $\lambda_i$ as $n$ grows, hence we hide constants that do not depend on $n$.

\begin{restatable}{theorem}{theoremDriftConcentration}\label{theoremDriftConcentration}
    Let $i\leq d$ and $\lambda_i(n)\geq \tfrac{c}{\sqrt{n}}$ for some constant $c>0$. Then, \begin{align}\label{eq:thm_driftConc1}
    \begin{aligned}
        \|\hat{b}_i - b_i^*\|_{L^2(p)}   \leq & \;\|b_{i,\lambda_i} - b_i^*\|_{L^2(p)} + \frac{\mathcal{O}(\log n)}{\lambda_i \sqrt{n}} \log\left( \tfrac{2}{\delta}\right) \\
        &  +\frac{\mathcal{O}(1)}{\lambda_i} \log \left(\tfrac{2}{\delta}\right) \sum_{j\in\pa{i}} \|\hat{\mathbf{b}}_j - \mathbf{b}_j^*\|_{E^n},
    \end{aligned}\\\label{eq:thm_driftConc2}
    \begin{aligned}
        \|\hat{\mathbf{b}}_i - \mathbf{b}_i^*\|_{E^n}\;\;  \leq& \;  \|\mathbf{b}_{i,\lambda_i} - \mathbf{b}_i^*\|_{E^n} + \frac{\mathcal{O}(\log n)}{\lambda_i \sqrt{n}}\cdot\log(\tfrac{2}{\delta}) \\&+ \frac{2\kappa}{\lambda_i} \sum_{j\in\pa{i}} \|\hat{\mathbf{b}}_j - \mathbf{b}_j^*\|_{E^n}
    \end{aligned}
\end{align} with probability at least $1-\delta$ for any $\delta\in (0,1)$. Here, $\mathcal{O}(\cdot)$ is as $n\to\infty$.
\end{restatable} Theorem \ref{theoremDriftConcentration} bounds the generalization and training error of $\hat{b}_i$ with high probability. We find the same rate $(\lambda_i\sqrt{n})^{-1}$ as \citet{Vito2005} although we do not require $\hat{\zeta}_i\in\mathcal{H}_i$. This rate would diverge if $\lambda_i < 1/\sqrt{n}$, hence the restriction in the theorem statement. % Unique to our analysis are the parent training error contributions due to the drift plug-in along the causal order. 
For a more concise result, we consider the expected generalization error, for which we combine \eqref{eq:thm_driftConc1} and \eqref{eq:thm_driftConc2} and finally integrate out $\delta$.
\begin{restatable}{corollary}{corollayExpectGenError}\label{corollaryExpectGenError}Let $d_{ij}$ be the length of longest path from $i$ to $j$ in $\mathcal{D}$, excluding self-loops and therefore possibly zero. Let $d_i:=\max_j d_{ji}$ and set $r_i := 6\cdot 3^{d_i}$. Use $T_j(\lambda) := \|b_{j,\lambda} - b_j^*\|_{L^2(p)}$ to abbreviate the Tikhonov population approximation error. If \begin{equation}\label{eq:corollaryExpectGenError_claim1}
    \lambda_i(n) \geq n^{-2/r_i} \;\vee_{j\in\anc{i}}\; T_j^{2/3^{d_{ji}}}(\lambda_j(n))
\end{equation} for all $i\leq d$, then the expected generalization error is bounded as follows: \begin{equation}\label{eq:corollaryExpectGenError_claim2}
    \ex{\|\hat{b}_i - b_i^*\|_{L^2(p)}} \leq\frac{\mathcal{O}(\log n)}{n^{1/r_i}} + \mathcal{O}(1)\sum_{j\in\anc{i}\cup \{i\}}T_j^{1/3^{d_{ji}}}(\lambda_j) \vee T_j(\lambda_j).
\end{equation} In particular, if $\mathcal{H}_i$ is dense in $L^2(p_{\Pa{i}})$ for all $i\leq d$, the condition \eqref{eq:corollaryExpectGenError_claim1} allows choices $\lambda_i(n)$ with  $\lim_{n\to\infty}\lambda_i(n) = 0$, for which then \begin{equation*}
    \lim_{n\to\infty} \ex{\|\hat{b}_i - b_i^*\|_{L^2(p)}} = 0.
\end{equation*}
Further, consider $T_j(\lambda) = \mathcal{O}(\sqrt{\lambda}\cdot \log^{\alpha_j}\tfrac{1}{\lambda})$ as $\lambda\to0$ for all $j\leq d$ and constants $\alpha_j\geq 0$. For any $\varepsilon\in(0,\tfrac{1}{2d})$, the choice $\lambda_i := n^{-2(1-\varepsilon_i)/r_i}$ with $\varepsilon_i:=d_i\cdot\varepsilon$ guarantees that\begin{equation*}
    \ex{\|\hat{b}_i - b_i^*\|_{L^2(p)}} =\mathcal{O}(n^{-(1-2\varepsilon_i)/r_i}).
\end{equation*}
\end{restatable}

This Corollary guarantees consistency of our method if $\mathcal{H}_i$ is dense in $L^2(p_{\Pa{i}})$, i.e., if $k_i$ is $L^2$ universal. An important example is the Gaussian kernel \citep[4.36]{Steinwart2008SVM}, which we implement in the software package. The condition \eqref{eq:corollaryExpectGenError_claim1} requires that the regularization strength $\lambda_i(n)$ may only decrease as fast as the the ancestor drifts $b_j^*$ can be approximated in the space $\mathcal{H}_j$ with regularization strength $\lambda_j(n)$. In practice, we recommend to choose $\lambda_i$ via cross-validation (Section \ref{subsec:cv}). The final part of Corollary \ref{corollaryExpectGenError} is inspired by $T_j(\lambda)\leq \mathcal{O}(\sqrt{\lambda})$ for $b_j^*\in\mathcal{H}_j$. We allow for additional log factors since drifts of practical interest, e.g.,  $\mathbf{b}(\mathbf{x}) = -\mathbf{x}$, may possess RKHS smoothness locally, but $b_i^*\notin\mathcal{H}_i$ due to tail growth mismatch. If $p$ is sub-Gaussian, the tail can sometimes be controlled by such log-factors, e.g., $b_1^*(x) = ax + f(x)$ satisfies $T_j(\lambda) = \mathcal{O}( \sqrt{\lambda} \log^2 \tfrac{1}{\lambda})$ for any $f$ in the Gaussian RKHS and $a\in\mathbb{R}$ \citep{Viktor2026}. The growth $T_j(\lambda)\approx \sqrt{\lambda}$ influenced our choice of $r_i$ and the exponent $3^{d_{ji}}$, since $\sqrt{\lambda} + (n^{\alpha}\lambda)^{-1}$ is balanced by $\lambda(n)= n^{-2\alpha/3}$ for $\alpha>0$. A different assumption on $T_j(\lambda)$ would imply a different balancing choice $r_i$, however note our discussion on source conditions below and note that the consistency result is unaffected.

In the remaining section, we discuss how to improve the doubly-exponential rate $r_i$ and other open problems. First, we are not aware of any lower bounds for learning structured drifts from the stationary distribution. For the related problem of learning the score function $\nabla\log p=\nabla p / p$, see Section \ref{subsec:accelerating_difusion_theory}, one can achieve minimax optimal $L^2(p)$ generalization error by estimating $\hat{p}$ and essentially returning $\nabla\hat{p}/\hat{p}$, see \citet{Wibisono2024}. The closest parallel to this viewpoint is equation \eqref{eq:complete_dag_identity} in Theorem \ref{theoremDriftExistenceCompleteDAG}, expressing $b_i^*$ in terms of conditionals $p_{i|1,\ldots, i-1}$, their derivatives and partial integrals. This suggests learning $b_i^*$ may be harder than learning the score, however equation \eqref{eq:complete_dag_identity} could potentially simplify. 

A well-known method to guarantee faster rates are \emph{source conditions} of the form $b_i^*= (J_i J_i^*)^r f_i$ for some $r >0$ and $f_i\in L^2(p)$ \citep{Dicker2017, Zhou2020ScoreEstimators}. The assumption $b_i^*\in\mathcal{H}_i$ corresponds to $r=\tfrac{1}{2}$. Regularization methods with sufficient qualification can guarantee $\|b_{i,\lambda} - b_i^*\|_{L^2(p)}=\mathcal{O}(\lambda^r)$, however a source condition for $r >\tfrac{1}{2}$ depends on $p$ through the integral operator $J_i^*$. %Bauer2007 after equation 9
Since $\mathbf{b}^*$ is completely determined by $p$, the source condition becomes a condition on $p$, which we leave to future work. 

The term driving the rate $r_i$ is $\tfrac{1}{\lambda_i}\sum_{j\in\pa{i}}\|\hat{\mathbf{b}}_j - \mathbf{b}_j^*\|_{E^n}$. Could $\hat{b}_i$ be de-coupled from $(b_j^*)_{j\in\pa{i}}$, similar to nonparametric score learning? Unfortunately not, since the stationary density $p$ and the parent set $\pa{i}$ alone do not determine $b_i^*$, see Appendix \ref{subsec:app_theory_discuss}. However, we believe the factor $1/\lambda_i$ is too conservative in general. It arises from the mismatch that replacing $b_i^*$ by some other $b_i\in L^2(p_{\Pa{i}})$ in $z_i$ structurally affects $\ex[\mathbf{x}\sim p]{z_i(\mathbf{x},\cdot)}$, e.g., it may not lie in the image of the integral operator $J_i^*$. As the proof of Theorem \ref{theoremDriftConcentration} illustrated, the integral operator $J_i^*$ controls the inverse power of $\lambda$. In Appendix \ref{subsec:app_theory_discuss}, we show how to potentially link $\ex[\mathbf{x}\sim p]{z_i(\mathbf{x},\cdot)}$ with $J_i^*$ for certain kernels, although tighter theoretical guarantees may require replacing $\hat{\zeta}_i$ by an analytically better behaved estimator. In particular, a sufficiently strong bound on $\|\hat{\zeta}_i - \zeta_i\|_{\mathcal{H}_i}$ would allow us to follow the argument by \citet{Zhou2020ScoreEstimators}, leading to $1/\sqrt{n\lambda_i}$. % Note that Zhou obtains 1/lambda sqrt(n) in the H-norm, not L^2; note the final line in the proof of their Theorem B.1.
That being said, the current choice $\hat{\zeta}_i$ being a sample average has its own advantages, e.g., see Section \ref{subsec:dependent_samples}. 

% \begin{restatable}[Partial knowledge of parent drifts sometimes suffices]{remark}{remarkPartialParentKnowledge}
% Consider a three-variate distribution $(X_1, X_2, X_3)$ with drift graph $1\to 2 \to 3$ including self-loops. A closer look Lemma \ref{lemmaParentMarg} shows $\zeta_3=L_3 b_3$ only depends on $f(x_2, x_3):=\coex{b_2(X_1, x_2)}{X_2=x_2, X_3=x_3}$. It is unclear whether one can learn $f$ without learning $b_2$ first, especially since $X_1 \not\perp X_3\, |\, X_2$ in general \citep{Boege2025LyapunovCI}.  
% \end{restatable}

\subsection{Hyperparameter selection}\label{subsec:cv}
Let the estimate $\hat{b}_{i,\lambda}$ depend on a hyperparameter $\lambda$, e.g., the regularization strength in Tikhonov regularization. We derive a cross-validation scheme to adaptively select $\lambda$, assuming all estimates $\hat{b}_1,\ldots,\hat{b}_{i-1}$ earlier in the causal order have already been tuned.

Our procedure estimates the risk up to an additive constant, similar to cross-validation in nonparametric density estimation. % e.g. see Wassermann's Nonparametric book ch. 6.1
If $\hat{b}_{i,\lambda}\in\mathcal{H}_i$, which for example can be guaranteed for the Gaussian kernel (Section \ref{subsec:estimator_derivation}), we use that
\begin{equation}\label{eq:cv_err_dec}
    \|\hat{b}_{i,\lambda} - b_i^*\|_{L^2(p)}^2 \propto \int_{\mathbb{R}^d} \hat{b}_{i,\lambda}^2 \dd{p} -2 \langle \hat{b}_{i,\lambda}, J^* b_i^*\rangle_{\mathcal{H}_i}
\end{equation} up to the additive constant $ \int_{\mathbb{R}^d} (b_i^*)^2\dd{p}$, where we applied the adjoint relation $\langle J\hat{b}_{i,\lambda}, b_i^*\rangle_{L^2(p)} = \langle \hat{b}_{i,\lambda}, J^*b_i^*\rangle_{\mathcal{H}_i}$ for the inclusion $J\colon\mathcal{H}_i\to L^2(p)$, see Section \ref{subsec:RKHS_theory}.

Let $\hat{b}_{i,\lambda;\mathcal{I}}$ be learned on the training data $(\mathbf{x}_i)_{i\in \mathcal{I}}$. We estimate the right hand side of equation \eqref{eq:cv_err_dec} using test data $(\mathbf{x}_j)_{j\in \mathcal{J}}$, where $\mathcal{I}\cap\mathcal{J}=\emptyset$. Then, $\tfrac{1}{|\mathcal{J}|}\sum_{j\in\mathcal{J}} \hat{b}_{i,\lambda;\mathcal{I}}(\mathbf{x}_j)^2$ is an unbiased estimate of $\int_{\mathbb{R}^d} \hat{b}_{i,\lambda;\mathcal{I}}^2 \dd{p}$. For the second term, recall that $-J^*b_i^*=\zeta_i$ by \eqref{eq:estimation_setup1}. We estimate $\hat{\zeta}_{i;\mathcal{J}}:= \tfrac{1}{|\mathcal{J}|}\sum_{j\in\mathcal{J}}\hat{z}_i(\mathbf{x}_j,\cdot)$. To summarize, we estimate \eqref{eq:cv_err_dec} by\begin{equation*}
    2\langle\hat{b}_{i,\lambda;\mathcal{I}},\hat{\zeta}_{i;\mathcal{J}}\rangle_{\mathcal{H}_i} + \frac{1}{|\mathcal{J}|}\sum_{j\in\mathcal{J}} \hat{b}_{i,\lambda;\mathcal{I}}(\mathbf{x}_j)^2.
\end{equation*}For spectral regularization, $\hat{b}_{i,\lambda:\mathcal{I}}$ is a linear combination of kernel basis functions and we compute $\langle\hat{b}_{i,\lambda;\mathcal{I}},\hat{\zeta}_{i;\mathcal{J}}\rangle_{\mathcal{H}_i}$ using the kernel trick. Otherwise, we project $\hat{\zeta}_{i;\mathcal{J}}$ onto the eigenspaces of $\hat{L}_{i;\mathcal{J}}$ corresponding to the largest eigenvalues $\{\gamma : \gamma\geq\gamma_\mathrm{min}\}$ for some threshold $\gamma_\mathrm{min}>0$, where $\hat{L}_{i;\mathcal{J}} f := \tfrac{1}{|\mathcal{J}|}\sum_{j\in\mathcal{J}}k_i(\mathbf{x}_j,\cdot) f(\mathbf{x}_j)$. This is formally justified if $\zeta_i$ lies in the image of $L_i$, e.g., when $b_i^*\in\mathcal{H}_i$. We can then compute the inner product by the kernel trick, see Appendix \ref{subsec:appendix_cv}.

We use standard $k$-fold crossvalidation, meaning $\{1,\ldots,n\}$ is split into $k$ bins $B_1,\ldots, B_k$ of roughly equal size. Set $\mathcal{J}_l:= B_l$ and $\mathcal{I}_l:= \{1,\ldots,n\}\setminus \mathcal{J}_l$ for $l\leq k$. Apply the method described above on $(\mathcal{I}_i,\mathcal{J}_i)$ and average the resulting $k$ estimates for a final risk estimate. Repeat over a $\lambda$-grid and choose $\lambda$ with lowest risk.

Some practical considerations are as follows. We mostly use kernels for which the eigenvalues of $L_i$ decay very rapidly %\citep{Belkin18}
and set $\gamma_\mathrm{min}=0.01$ by default. How computationally expensive the cross-validation is depends on the regularizer. For example, the Landweber regularizer and its accelerated versions actually obtain $b_{i,\lambda}$ for all $\lambda >\lambda_0$ when computing $b_{i,\lambda_0}$. Due to double descent phenomena, we recommend testing very small $\lambda_0$. Finally, keep in mind that estimation errors compound along the causal order, which affects $\hat{\zeta}_{i;\mathcal{J}}$ and therefore also the accuracy of cross-validation.

\section{Simulations and extensions}\label{sec:sim_and_ext}
\subsection{Synthetic population dynamics in a random environment}\label{subsec:two_dim_population_rand_env}
\begin{figure}
    \centering
    \includegraphics[width=0.45\linewidth]{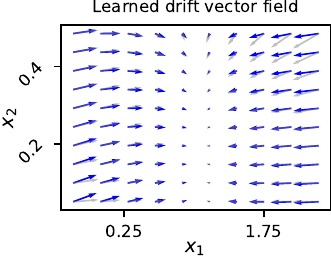}\hfill
    \raisebox{0.15cm}{
    \includegraphics[width=0.45\linewidth]{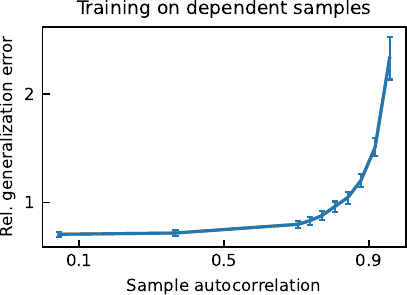}
    }
    \caption{Learning the drift of SDE \eqref{eq:sim_pop_model} from $1000$ equilibrium samples. \textbf{Left:} good agreement between learned (blue) and ground truth drift (grey) up to low data density corners; Section \ref{subsec:two_dim_population_rand_env}. \textbf{Right:} Sampling from a single diffusion path with varying observation gap $\Delta t$; Section \ref{subsec:dependent_samples}. \\ Alt text: Left panel: Arrow plot of learned versus ground truth drift. Right panel: line plot of empirical generalization error versus sample autocorrelation. The error remains roughly stable below a correlation of 0.75 and increases strongly afterwards.}
    \label{fig:population_model_drift}
\end{figure}
Let the population size $x_{2}(t)$ of some species be influenced by an observable exogenous process $x_{1}(t)$, e.g., precipitation. Consider logistic growth with capacity decaying to zero when the exogenous factor $x_{1,t}$ is large: 
\begin{equation}\label{eq:sim_pop_model}
    \dd{\mathbf{x}}(t) = \begin{pmatrix}
        0.5\cdot(1 - x_{1}(t)) \\
        2 x_{2}(t)\left(\frac{1}{1 + \exp(x_{1}(t))} - x_{2}(t)\right)
    \end{pmatrix} \, \dd{t} + \begin{pmatrix}
        1 & 0 \\
        0 & \tfrac{1}{2} x_{2}(t)
    \end{pmatrix} \dd{\mathbf{w}}(t).
\end{equation} The numerical constants were chosen to guarantee stable persistence of the population \citep{Li2011LogisticPopulationErgodicity}. From this stationary distribution, we non-parametrically recover the drift vector field. In particular, we discover that there is a logistic mechanism and that the carrying capacity is $1/(1 + \exp(x_{1}))$.

The synthetic training data are $n=1000$ vectors in $\mathbb{R}^2$, which approximate an independent sample from the stationary distribution of \eqref{eq:sim_pop_model}.  % by generously thinning an Euler-Maruyama simulated path.
Practically, this could mean $n$ isolated populations of the species, and we only require \emph{one} measurement of the population number and the exogenous factor per location. See Section \ref{subsec:dependent_samples} for multiple samples.

We apply our Tikhonov regularized estimator, requiring as input the diffusivity function $(\begin{smallmatrix}
    1 & 0 \\ 0 & .5 x_2
\end{smallmatrix})$ and the structure graph \rotatebox[origin=c]{270}{$\circlearrowright$}\hspace{0.1cm}$1 \to 2$ \rotatebox[origin=c]{90}{$\circlearrowleft$}; see Example \ref{runExample}. We use the Gaussian kernel for both drift components. The regularization parameters $\lambda_1, \lambda_2$ are selected via cross-validation; see Section \ref{subsec:cv}.

The left panel of Figure \ref{fig:population_model_drift} compares the estimated and ground truth drift vector fields on a rectangle encompassing the bulk of the training data. Both fields agree very well. Discrepancies increase towards the top right corner, a low-density region of the stationary distribution, which contribute negligibly to the $L^2(p)$ generalization error, our metric of choice inspired by the score learning literature \citep{Wibisono2024}.

Extending this regression model to multiple covariate processes requires information on the interplay between the covariates, see Appendix \ref{subsec:app_theory_discuss}. A simple model is to assume independent covariate processes \citep[eq. 33]{Gregorio2025DiffusionProcessRegression}. More generally, if a causal hierarchy between the processes can be specified, our estimator can be applied directly. If the covariate processes are coupled, e.g., \citet{Varughese2008} study the effect of a bi-variate coupled exogenous diffusion on a birth-death process, the covariate dynamics may not be identified from the equilibrium alone. If the covariate drifts can be learned from a different data source, our method can still be applied to the outcome drift component, i.e., to learn how the covariate processes affect the outcome process from equilibrium observations only.

\subsection{Synthetic data}
\begin{figure}
    \centering
    \begin{tikzpicture}
        \node[anchor=south west, inner sep=0] (img)
        {\includegraphics{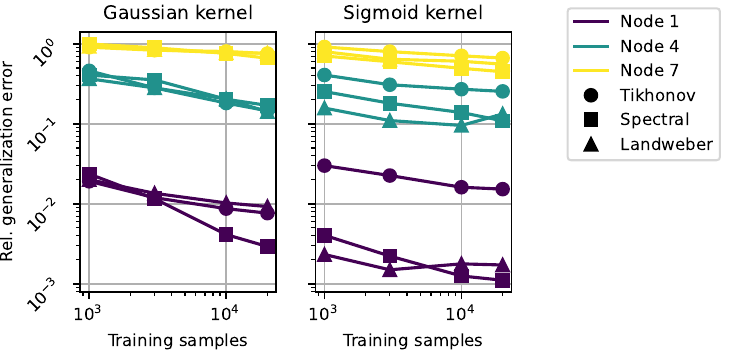}};
    
        % Define anchor
        \coordinate (dagAnchor) at (9.7,-0.5);
        
        \begin{scope}[shift={(dagAnchor)},>=stealth,
        line width=1pt,
        vertex/.style={circle, draw, minimum size=7mm, inner sep=0pt}]
        % Nodes
        \node[vertex] (0) at (0.5,3) {\small 1};
        \node[vertex] (4) at (1.9,3) {\small 5};
        
        \node[vertex] (1) at (0,2) {\small 2};
        \node[vertex] (3) at (1.2,2) {\small 4};
        \node[vertex] (5) at (2.1,2) {\small 6};
        
        \node[vertex] (2) at (0.9,1) {\small 3};
        \node[vertex] (6) at (2.1,1) {\small 7};
        
        % Self-loops
        \draw[->] (0) edge[loop left] ();
        \draw[->] (1) edge[loop below] ();
        \draw[->] (2) edge[loop left] ();
        \draw[->] (3) edge[loop above] ();
        \draw[->] (4) edge[loop right] ();
        \draw[->] (5) edge[loop right] ();
        \draw[->] (6) edge[loop right] ();
        
        \draw[->] (0) -- (1);
        \draw[->] (0) -- (3);
        \draw[->] (1) -- (2);
        \draw[->] (1) -- (3);
        \draw[->] (2) -- (6);
        \draw[->] (3) -- (6);
        \draw[->] (5) -- (6);
        \draw[->] (4) -- (5);

        \end{scope}
    \end{tikzpicture}
    \hspace{0.3cm}
    \caption{Relative generalization error $\|\hat{b}_i - b_i^*\|_{L^2(p)}^2 / \|b_i^*\|_{L^2(p)}^2$ versus training sample size for selected $i=1,4,7$. Bottom right: structure graph of the $7$-variate distribution.\\ Alt text: Two plots compare the generalization error of the Gaussian kernel and the Sigmoid kernel, with multiple lines representing various drift components and regularization methods.}
    \label{fig:sim_results}
\end{figure}
Consider a seven-dimensional drift structured by the graph in Figure \ref{fig:sim_results}, defined via \begin{equation*}
    b_i(\mathbf{x}) := \sum_{j\in \Pa{i}} w_j \tanh(-\mathbf{q}_j^T \mathbf{x}_{\Pa{i}} ),
\end{equation*} where $(\mathbf{q}_j)_{j\in \Pa{i}}$ is a randomly generated orthogonal basis of $\mathbb{R}^{|\Pa{i}|}$ such that $q_{j,i} > 0$, and where $w_j$ are random positive weights. This construction ensures the drift is reverting towards the origin. We take the diffusivity $\mathbf{D}$ to be the constant identity and generate between $n=10^3$ and $n=2\cdot 10^4$ samples from the stationary distribution by MCMC. 

We train with Tikhonov, Spectral, and accelerated Landweber regularization using the Gaussian and the Sigmoid kernel. Regularization strengths are tuned by cross-validation, see Section \ref{subsec:cv}. For each training set size, we record the relative squared generalization errors $\|\hat{b}_i - b_i^*\|_{L^2(p)}^2 / \|b_i^*\|_{L^2(p)}^2$ for all nodes $i=1,\ldots,7$.  To reduce the $\mathcal{O}(n^3)$ computational complexity scaling of Tikhonov and Spectral regularization, we apply Nyström's method, see \citet{Viktor2026}, and reduce $\mathbf{K}$ to  $\mathcal{O}(\sqrt{n})$ rows and columns, covering all cases in \citet{Sutherland2018}. This usually works well, however these two methods produced a handful of  outliers. % Even when scaling m = 0.05*n
In practice, we recommend to keep as many rows and columns as possible, or apply the accelerated Landweber method with naturally only roughly $\mathcal{O}(n^2)$ scaling, which had no performance outliers.

The median relative generalization error over $N=100$ experiments is shown in Figure \ref{fig:sim_results}. The first takeaway is that our method works as intended on a problem of this scale. Predictably, the sigmoid kernel performs better than the Gaussian kernel as it matches the ground truth drift structure. Spectral and Landweber regularization significantly outperform Tikhonov regularization for the sigmoid kernel, consistent with theory that higher qualification methods adapt to drifts which are low complexity in the RKHS. Spectral regularization outperforms for the Gaussian kernel on the root node, possibly because its cross-validation procedure requires one approximation step less. Finally, note that the accelerated Landweber regularization has a rougher error curve due to its discrete number of steps; for example the median step number transitions from two to three after $10^4$ samples for the sigmoid kernel on node four. 

We did not include the method by \citet{Lorch2024StationaryDiffusions} in this experiment since their implementation assumes known linear self-regulation and does not tune hyper-parameters. Instead, we apply it to a linear drift in Appendix \ref{app:kds} and demonstrate its ability to approximate DAG-structured drifts under oracle stopping.

\subsection{Dependent samples}\label{subsec:dependent_samples}
Our main motivation was to require only a single equilibrium sample per individual, however our estimator can also be meaningful when $(\mathbf{x}_j)_{j=1,\ldots,n}$ are collected from one individual at times $t_j:= j\cdot \Delta t$ for some $\Delta t >0$. Namely, if $\Delta t>0$ is independent of $n$, also termed the \emph{low-frequency} regime \citep{Gobet2004LowFrequencySeminal}, and if the evolution is ergodic, the average $\tfrac{1}{n}\sum_{j=1}^n z_i(\mathbf{x}_j, \mathbf{y})$ converges to $\zeta_i(\mathbf{y})$ in probability. This suggests our estimator may also be consistent for ergodic processes under low frequency sampling. The finite sample performance depends on how close to equilibrium the individual was at time $t=0$.

We repeat the simulation from Section \ref{subsec:two_dim_population_rand_env} with low frequency samples, meaning we simulate the SDE \eqref{eq:sim_pop_model} by the Euler-Maruyama method, and take $n=1000$ samples with time distance $\Delta t$. A generous burnin period ensures $\mathbf{x}_1$ approximately follows the stationary distribution. We consider $10$ different choices for $\Delta t$, such that the Lag $1$ autocorrelation of $(\mathbf{x}_j)_{j=1,\ldots,n}$ varies between $0.04$ and $0.96$, and fit the Tikhonov regularized estimator with cross-validation. We estimate the mean squared generalization error relative to $\|\mathbf{b}\|_{L^2(p)}^2$ on a fresh sample and repeat this experiment $N=100$ times. 

The right part of Figure \ref{fig:population_model_drift} shows the average relative generalization error versus the average autocorrelation for each $\Delta t$. As expected, high autocorrelation ($\Delta t \to 0$) decreases performance. Still, performance is reasonable for low and medium autocorrelation. Although our estimator does not make full use of the available information, we include this example here since multivariate nonparametric implementations for low frequency sampled data are rare.

\subsection{Irreversible diffusion models}\label{subsec:accelerating_difusion_theory}
Given observations $\mathbf{x}_1,\ldots,\mathbf{x}_n$ from an unknown density $p$, generative modeling aims to produce a fresh sample $\tilde{\mathbf{x}}_1,\ldots\tilde{\mathbf{x}}_m$ from $p$ as accurately as possible \citep{Song2019}. Diffusion models learn the score function $\mathbf{s}(\mathbf{x}):= \nabla\log p (\mathbf{x})$ and sample via the Langevin SDE \begin{equation}\label{eq:score_sde}
    \dd{\tilde{\mathbf{x}}}(t) = \mathbf{s}(\tilde{\mathbf{x}}(t))\dd{t} + \sqrt{2} \dd{\mathbf{w}(t)},
\end{equation} which has $p$ as its stationary distribution as desired. If a causally structured drift inducing $p$ exists, it could augment the score drift: assume the function $\mathbf{b}$ also induces $p$ as a stationary density when used as a drift in \eqref{eq:score_sde}. Since both then solve the Fokker-Planck equation, we have $\nabla\cdot(\mathbf{b}p) = \Delta p = \nabla\cdot(\mathbf{s}p)$, therefore $\nabla\cdot((\mathbf{b - s})\,p) = 0$. Applying the results by \citet{Hwang2005AcceleratingDiffusions} to the drift functions $\mathbf{s} + \omega \cdot(\mathbf{b} - \mathbf{s})$ for $\omega\in\mathbb{R}$, we find that any $\omega\neq 0$ usually leads to strictly faster convergence against $p$ compared to the score function $\mathbf{s}$ alone, i.e., $\omega=0$. We illustrate this experimentally for a Gaussian distribution $p$ in Figure \ref{fig:Acc_diffusion_theory}; for simulation details see Appendix \ref{subsec:diffusion_sim_details}.

A causally structured drift $\mathbf{b}$ is guaranteed to exist for any given density $p$ when the DAG is complete as we show below. This generalizes a result for Gaussian distributions \citep{Dettling2022}. Note that, even for the complete DAG, $\mathbf{b}\neq \mathbf{s}$; for example by Schwarz's theorem $\partial_i s_j = \partial_j s_i$, which does not hold for $\mathbf{b}$.

\begin{restatable}{theorem}{theoremDriftExistenceCompleteDAG}\label{theoremDriftExistenceCompleteDAG}
    Let $0 <p\in C^2(\mathbb{R}^d)$ be a probability density and $\sigma$ be constant. For any $j\leq d$, define $1:j \, := 1,\ldots, j$, let $f_{j+1}$ and $F_{j+1}$ denote the density and cumulative distribution function of $x_{j+1}$ given $\mathbf{x}_{1:j}$, respectively, and assume $F_{j+1}\in C^3(\mathbb{R}^j)$. % Need \Delta F_{j+1} to have continuous partial_{j+1}. This follows from p in C2 under sufficient regularity.
    
    Define $\mathbf{b}\colon\mathbb{R}^d\to\mathbb{R}^d$ component-wise as $b_1(\mathbf{x}):= \tfrac{\sigma^2}{2}\tfrac{p_1'(x_1)}{p_1(x_1)}$ and $b_{j+1}$ recursively as
    \begin{equation}\label{eq:complete_dag_identity}
        b_{j+1}:=\frac{\sigma^2}{2} \frac{\partial_{j+1} f_{j+1} + \Delta_{1:j} F_{j+1} + 2 (\nabla_{1:j}\log p_{1:j})^T\, \nabla_{1:j} F_{j+1}}{f_{j+1}} + \frac{(\nabla_{1:j}F_{j+1})^T \,\mathbf{b}_{1:j}}{f_{j+1}}.
    \end{equation}
    Then, $\mathbf{b}$ is structured according to the complete DAG with topological order $1,\ldots,d$ (including self loops), and satisfies the Fokker-Planck equation $\tfrac{\sigma^2}{2}\Delta p = \nabla\cdot (\mathbf{b} p)$.
\end{restatable}  

\begin{figure}
    \centering
    \includegraphics{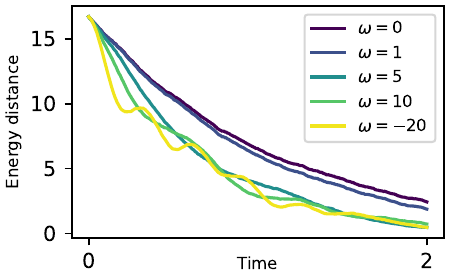}
    \caption{Consider score $\mathbf{s}$ and causal drift $\mathbf{b}$ from Theorem \ref{theoremDriftExistenceCompleteDAG} for some $5$-variate Gaussian density $p$; see Appendix \ref{subsec:diffusion_sim_details}.
    We simulate $\dd{\mathbf{x}}(t) = (\mathbf{s} - \omega\cdot(\mathbf{s} - \mathbf{b}))(\mathbf{x}(t))\dd{t} + \sqrt{2}\dd{\mathbf{w}(t)}$ and plot the energy distance between $\mathbf{x}(t)$ and the target $p$ for various $\omega,t$. As discussed in Section \ref{subsec:accelerating_difusion_theory}, $\omega\neq 0$ speeds up convergence.\\ Alt text: Multiple lines show the decrease in energy distance over time. The larger the absolute value of omega, the faster the decay.}
    \label{fig:Acc_diffusion_theory}
\end{figure}

Note that a different topological ordering than $1,\ldots,d$ could further affect the convergence speed. 
On a high level, perturbing the score by the causal drift makes the diffusion process \emph{irreversible}, which has long been known to speed up convergence \citep{Hwang2005AcceleratingDiffusions, ReyBellet2015IrreversibleLargeDeviation}. For linear drifts, the non-reversible perturbation leading to the fastest convergence has been characterized by \citet{Lelivere2013OptimalLinearIrreversible}. %There also exist perturbations that preserve reversibility but still improve convergence \citep{Abdulle2019AcceleratingDiffusionWithReversible}. 
While the aforementioned literature takes a sampling perspective, i.e., assuming the target density $p$ is (partially) known, generative modeling assumes only observations from $p$ are available.

\section{Discussion and limitations}\label{sec:discussion}
We established nonparametric identifiability under very weak conditions from which future estimators, also parametric ones, can profit. A canonical extension would be to consider cyclic models and models with unknown structure graphs, in parallel with recent advances for linear drifts. A further important research direction concerns the case of unknown diffusivity, which remains open even for linear systems.

Our proposed estimator is computationally feasible, even when the Fokker-Planck forward problem becomes intractable due to the curse of dimensionality. We theoretically guarantee the estimator's consistency in $L^2(p)$, which practically can be achieved by the Gaussian kernel in our package. Although the consistency result assumes bounded drifts for technical convenience, the estimator shows strong empirical performance even when the drift is unbounded. Extending our method to new kernels requires access to its univariate partial integrals. We have so far only used kernels from Table \ref{tab:kernels_with_antideriv} allowing analytical expressions, however numerical approximations to the univariate integral should also work. A notable theoretical limitation of our estimator is the recursive plug-in strategy, which potentially could accumulate estimation errors exponentially along the topological order. Experimentally, the more significant contributor to the estimation error seems to be the number of parent variables, also see \citet{Viktor2026}. Our approach, similar to the referenced score literature, %when using a diagonal kernel, the optimization problems in Zhou et al. de-couple
trades global optimality in a joint loss for computational efficiency and interpretability of the component functions, as these are optimized individually. Practically, our method could potentially be improved by taking biases of cross-validation into account \citep{Sylvain2010BiasCV}. Finally, note that many popular kernels do not produce drift estimates which can be sampled from with constant diffusivity since they decay to zero for large input. One can instead fit kernels weighted with a linear component at the cost of an extra hyperparameter \citep{Sriperumbudur2017KernelExponentialFamily}.

\subsection*{Acknowledgements}
We thank Viktor Vasilev for his contributions to this project through his master’s thesis, conducted under the supervision of R.S. and M.D., as referenced throughout the paper. We acknowledge support from the DAAD programme Konrad Zuse Schools of Excellence in Artificial Intelligence, sponsored by the Federal Ministry of Research, Technology and Space, from the European Research
Council (ERC) under the European Union’s Horizon 2020 research and innovation programme
(grant agreement No 883818), and from the German Federal Ministry of Education and Research and the Bavarian State Ministry for Science and the Arts. The authors take full responsibility for the contents of this article.

\bibliographystyle{imsart-nameyear} % Style BST file (imsart-number.bst or imsart-nameyear.bst)
\bibliography{bibliography}% common bib file

\begin{thebibliography}{44}
% BibTex style file: imsart-nameyear.bst, 2017-11-03
% Default style options (sort=1,type=nameyear).
% Used options (sort=1,type=nameyear).

\bibitem[\protect\citeauthoryear{Amorino and
  Gloter}{2023}]{Amorino2023InvariantDensityEstimation}
\begin{barticle}[author]
\bauthor{\bsnm{Amorino},~\bfnm{Chiara}\binits{C.}} \AND
  \bauthor{\bsnm{Gloter},~\bfnm{Arnaud}\binits{A.}}
(\byear{2023}).
\btitle{Estimation of the invariant density for discretely observed diffusion
  processes}.
\bjournal{Statistics}
\bvolume{57}
\bpages{213--259}.
\end{barticle}
\endbibitem

\bibitem[\protect\citeauthoryear{Améndola et~al.}{2025}]{Amendola2025}
\begin{barticle}[author]
\bauthor{\bsnm{Améndola},~\bfnm{Carlos}\binits{C.}},
  \bauthor{\bsnm{Boege},~\bfnm{Tobias}\binits{T.}},
  \bauthor{\bsnm{Hollering},~\bfnm{Benjamin}\binits{B.}} \AND
  \bauthor{\bsnm{Misra},~\bfnm{Pratik}\binits{P.}}
(\byear{2025}).
\btitle{Structural Identifiability of Graphical Continuous Lyapunov Models}.
\end{barticle}
\endbibitem

\bibitem[\protect\citeauthoryear{Arcones and Gin\'e}{1993}]{Arcones1993}
\begin{barticle}[author]
\bauthor{\bsnm{Arcones},~\bfnm{Miguel~A.}\binits{M.~A.}} \AND
  \bauthor{\bsnm{Gin\'e},~\bfnm{Evarist}\binits{E.}}
(\byear{1993}).
\btitle{Limit theorems for {$U$}-processes}.
\bjournal{Ann. Probab.}
\bvolume{21}
\bpages{1494--1542}.
\end{barticle}
\endbibitem

\bibitem[\protect\citeauthoryear{Arlot and Celisse}{2010}]{Sylvain2010BiasCV}
\begin{barticle}[author]
\bauthor{\bsnm{Arlot},~\bfnm{Sylvain}\binits{S.}} \AND
  \bauthor{\bsnm{Celisse},~\bfnm{Alain}\binits{A.}}
(\byear{2010}).
\btitle{A survey of cross-validation procedures for model selection}.
\bjournal{Stat. Surv.}
\bvolume{4}
\bpages{40--79}.
\end{barticle}
\endbibitem

\bibitem[\protect\citeauthoryear{Bhattacharya and
  Waymire}{2023}]{Bhattacharya2023DiffusionBook}
\begin{bbook}[author]
\bauthor{\bsnm{Bhattacharya},~\bfnm{Rabi}\binits{R.}} \AND
  \bauthor{\bsnm{Waymire},~\bfnm{Edward~C.}\binits{E.~C.}}
(\byear{2023}).
\btitle{Continuous parameter {M}arkov processes and stochastic differential
  equations}.
\bseries{Graduate Texts in Mathematics}
\bvolume{299}.
\bpublisher{Springer}, \baddress{Cham}.
\end{bbook}
\endbibitem

\bibitem[\protect\citeauthoryear{Boege et~al.}{2025}]{Boege2025LyapunovCI}
\begin{barticle}[author]
\bauthor{\bsnm{Boege},~\bfnm{Tobias}\binits{T.}},
  \bauthor{\bsnm{Drton},~\bfnm{Mathias}\binits{M.}},
  \bauthor{\bsnm{Hollering},~\bfnm{Benjamin}\binits{B.}},
  \bauthor{\bsnm{Lumpp},~\bfnm{Sarah}\binits{S.}},
  \bauthor{\bsnm{Misra},~\bfnm{Pratik}\binits{P.}} \AND
  \bauthor{\bsnm{Schkoda},~\bfnm{Daniela}\binits{D.}}
(\byear{2025}).
\btitle{Conditional independence in stationary distributions of diffusions}.
\bjournal{Stochastic Process. Appl.}
\bvolume{184}
\bpages{Paper No. 104604, 16}.
\end{barticle}
\endbibitem

\bibitem[\protect\citeauthoryear{Bongers and Mooij}{2018}]{bongers2018random}
\begin{barticle}[author]
\bauthor{\bsnm{Bongers},~\bfnm{Stephan}\binits{S.}} \AND
  \bauthor{\bsnm{Mooij},~\bfnm{Joris~M}\binits{J.~M.}}
(\byear{2018}).
\btitle{From random differential equations to structural causal models: The
  stochastic case}.
\bjournal{arXiv preprint arXiv:1803.08784}
\bvolume{3}.
\end{barticle}
\endbibitem

\bibitem[\protect\citeauthoryear{Botvinick-Greenhouse
  et~al.}{2023}]{Botvinick2023}
\begin{barticle}[author]
\bauthor{\bsnm{Botvinick-Greenhouse},~\bfnm{Jonah}\binits{J.}},
  \bauthor{\bsnm{Martin},~\bfnm{Robert}\binits{R.}} \AND
  \bauthor{\bsnm{Yang},~\bfnm{Yunan}\binits{Y.}}
(\byear{2023}).
\btitle{Learning dynamics on invariant measures using {PDE}-constrained
  optimization}.
\bjournal{Chaos}
\bvolume{33}
\bpages{Paper No. 063152, 22}.
\bdoi{10.1063/5.0149673}
\end{barticle}
\endbibitem

\bibitem[\protect\citeauthoryear{De~Gregorio
  et~al.}{2025}]{Gregorio2025DiffusionProcessRegression}
\begin{barticle}[author]
\bauthor{\bsnm{De~Gregorio},~\bfnm{Alessandro}\binits{A.}},
  \bauthor{\bsnm{Frisardi},~\bfnm{Dario}\binits{D.}},
  \bauthor{\bsnm{Iacus},~\bfnm{Stefano}\binits{S.}} \AND
  \bauthor{\bsnm{Iafrate},~\bfnm{Francesco}\binits{F.}}
(\byear{2025}).
\btitle{Adaptive elastic-net estimation for sparse diffusion processes}.
\bjournal{Stat. Inference Stoch. Process.}
\bvolume{28}
\bpages{Paper No. 22, 35}.
\end{barticle}
\endbibitem

\bibitem[\protect\citeauthoryear{De~Vito et~al.}{2005}]{Vito2005}
\begin{barticle}[author]
\bauthor{\bsnm{De~Vito},~\bfnm{Ernesto}\binits{E.}},
  \bauthor{\bsnm{Rosasco},~\bfnm{Lorenzo}\binits{L.}},
  \bauthor{\bsnm{Caponnetto},~\bfnm{Andrea}\binits{A.}},
  \bauthor{\bsnm{De~Giovannini},~\bfnm{Umberto}\binits{U.}} \AND
  \bauthor{\bsnm{Odone},~\bfnm{Francesca}\binits{F.}}
(\byear{2005}).
\btitle{Learning from examples as an inverse problem}.
\bjournal{J. Mach. Learn. Res.}
\bvolume{6}
\bpages{883--904}.
\end{barticle}
\endbibitem

\bibitem[\protect\citeauthoryear{Dettling et~al.}{2023}]{Dettling2022}
\begin{barticle}[author]
\bauthor{\bsnm{Dettling},~\bfnm{Philipp}\binits{P.}},
  \bauthor{\bsnm{Homs},~\bfnm{Roser}\binits{R.}},
  \bauthor{\bsnm{Am\'endola},~\bfnm{Carlos}\binits{C.}},
  \bauthor{\bsnm{Drton},~\bfnm{Mathias}\binits{M.}} \AND
  \bauthor{\bsnm{Hansen},~\bfnm{Niels~Richard}\binits{N.~R.}}
(\byear{2023}).
\btitle{Identifiability in continuous {L}yapunov models}.
\bjournal{SIAM J. Matrix Anal. Appl.}
\bvolume{44}
\bpages{1799--1821}.
\end{barticle}
\endbibitem

\bibitem[\protect\citeauthoryear{Dicker et~al.}{2017}]{Dicker2017}
\begin{barticle}[author]
\bauthor{\bsnm{Dicker},~\bfnm{Lee~H.}\binits{L.~H.}},
  \bauthor{\bsnm{Foster},~\bfnm{Dean~P.}\binits{D.~P.}} \AND
  \bauthor{\bsnm{Hsu},~\bfnm{Daniel}\binits{D.}}
(\byear{2017}).
\btitle{Kernel ridge vs. principal component regression: minimax bounds and the
  qualification of regularization operators}.
\bjournal{Electron. J. Stat.}
\bvolume{11}
\bpages{1022--1047}.
\end{barticle}
\endbibitem

\bibitem[\protect\citeauthoryear{Evans}{2010}]{EvansPDEbook}
\begin{bbook}[author]
\bauthor{\bsnm{Evans},~\bfnm{Lawrence~C.}\binits{L.~C.}}
(\byear{2010}).
\btitle{Partial differential equations},
\bedition{second} ed.
\bseries{Graduate Studies in Mathematics}
\bvolume{19}.
\bpublisher{American Mathematical Society}, \baddress{Providence, RI}.
\end{bbook}
\endbibitem

\bibitem[\protect\citeauthoryear{Gobet
  et~al.}{2004}]{Gobet2004LowFrequencySeminal}
\begin{barticle}[author]
\bauthor{\bsnm{Gobet},~\bfnm{Emmanuel}\binits{E.}},
  \bauthor{\bsnm{Hoffmann},~\bfnm{Marc}\binits{M.}} \AND
  \bauthor{\bsnm{Rei\ss},~\bfnm{Markus}\binits{M.}}
(\byear{2004}).
\btitle{Nonparametric estimation of scalar diffusions based on low frequency
  data}.
\bjournal{Ann. Statist.}
\bvolume{32}
\bpages{2223--2253}.
\end{barticle}
\endbibitem

\bibitem[\protect\citeauthoryear{Hwang
  et~al.}{2005}]{Hwang2005AcceleratingDiffusions}
\begin{barticle}[author]
\bauthor{\bsnm{Hwang},~\bfnm{Chii-Ruey}\binits{C.-R.}},
  \bauthor{\bsnm{Hwang-Ma},~\bfnm{Shu-Yin}\binits{S.-Y.}} \AND
  \bauthor{\bsnm{Sheu},~\bfnm{Shuenn-Jyi}\binits{S.-J.}}
(\byear{2005}).
\btitle{Accelerating diffusions}.
\bjournal{Ann. Appl. Probab.}
\bvolume{15}
\bpages{1433--1444}.
\end{barticle}
\endbibitem

\bibitem[\protect\citeauthoryear{Inglese}{2002}]{Inglese2002}
\begin{barticle}[author]
\bauthor{\bsnm{Inglese},~\bfnm{G.}\binits{G.}}
(\byear{2002}).
\btitle{Recovering a vector field with the aid of controlled noise}.
\bjournal{J. Inverse Ill-Posed Probl.}
\bvolume{10}
\bpages{187--193}.
\end{barticle}
\endbibitem

\bibitem[\protect\citeauthoryear{Karatzas and
  Ruf}{2016}]{Karatzas2016ExplosionTimeDistribution}
\begin{barticle}[author]
\bauthor{\bsnm{Karatzas},~\bfnm{Ioannis}\binits{I.}} \AND
  \bauthor{\bsnm{Ruf},~\bfnm{Johannes}\binits{J.}}
(\byear{2016}).
\btitle{Distribution of the time to explosion for one-dimensional diffusions}.
\bjournal{Probab. Theory Related Fields}
\bvolume{164}
\bpages{1027--1069}.
\end{barticle}
\endbibitem

\bibitem[\protect\citeauthoryear{Khasminskii}{2012}]{KhasminskiiBookStochasticProcesses}
\begin{bbook}[author]
\bauthor{\bsnm{Khasminskii},~\bfnm{Rafail}\binits{R.}}
(\byear{2012}).
\btitle{Stochastic stability of differential equations},
\bedition{second} ed.
\bseries{Stochastic Modelling and Applied Probability}
\bvolume{66}.
\bpublisher{Springer}, \baddress{Heidelberg}.
\end{bbook}
\endbibitem

\bibitem[\protect\citeauthoryear{Lacker and
  Zhang}{2023}]{Lacker2023StationaryDiffusionOnTree}
\begin{barticle}[author]
\bauthor{\bsnm{Lacker},~\bfnm{Daniel}\binits{D.}} \AND
  \bauthor{\bsnm{Zhang},~\bfnm{Jiacheng}\binits{J.}}
(\byear{2023}).
\btitle{Stationary solutions and local equations for interacting diffusions on
  regular trees}.
\bjournal{Electron. J. Probab.}
\bvolume{28}
\bpages{Paper No. 4, 37}.
\end{barticle}
\endbibitem

\bibitem[\protect\citeauthoryear{Lee and
  Trutnau}{2021}]{Haesung2021FullSDEPackageUnderLocalIntegrability}
\begin{barticle}[author]
\bauthor{\bsnm{Lee},~\bfnm{Haesung}\binits{H.}} \AND
  \bauthor{\bsnm{Trutnau},~\bfnm{Gerald}\binits{G.}}
(\byear{2021}).
\btitle{Existence, uniqueness and ergodic properties for time-homogeneous
  {I}t\^o-{SDE}s with locally integrable drifts and {S}obolev diffusion
  coefficients}.
\bjournal{Tohoku Math. J. (2)}
\bvolume{73}
\bpages{159--198}.
\end{barticle}
\endbibitem

\bibitem[\protect\citeauthoryear{Leli\`evre
  et~al.}{2013}]{Lelivere2013OptimalLinearIrreversible}
\begin{barticle}[author]
\bauthor{\bsnm{Leli\`evre},~\bfnm{T.}\binits{T.}},
  \bauthor{\bsnm{Nier},~\bfnm{F.}\binits{F.}} \AND
  \bauthor{\bsnm{Pavliotis},~\bfnm{G.~A.}\binits{G.~A.}}
(\byear{2013}).
\btitle{Optimal non-reversible linear drift for the convergence to equilibrium
  of a diffusion}.
\bjournal{J. Stat. Phys.}
\bvolume{152}
\bpages{237--274}.
\end{barticle}
\endbibitem

\bibitem[\protect\citeauthoryear{Li
  et~al.}{2011}]{Li2011LogisticPopulationErgodicity}
\begin{barticle}[author]
\bauthor{\bsnm{Li},~\bfnm{Xiaoyue}\binits{X.}},
  \bauthor{\bsnm{Gray},~\bfnm{Alison}\binits{A.}},
  \bauthor{\bsnm{Jiang},~\bfnm{Daqing}\binits{D.}} \AND
  \bauthor{\bsnm{Mao},~\bfnm{Xuerong}\binits{X.}}
(\byear{2011}).
\btitle{Sufficient and necessary conditions of stochastic permanence and
  extinction for stochastic logistic populations under regime switching}.
\bjournal{J. Math. Anal. Appl.}
\bvolume{376}
\bpages{11--28}.
\end{barticle}
\endbibitem

\bibitem[\protect\citeauthoryear{Lorch
  et~al.}{2024}]{Lorch2024StationaryDiffusions}
\begin{binproceedings}[author]
\bauthor{\bsnm{Lorch},~\bfnm{Lars}\binits{L.}},
  \bauthor{\bsnm{Krause},~\bfnm{Andreas}\binits{A.}} \AND
  \bauthor{\bsnm{Sch{\"{o}}lkopf},~\bfnm{Bernhard}\binits{B.}}
(\byear{2024}).
\btitle{Causal Modeling with Stationary Diffusions}.
In \bbooktitle{AISTATS 2024}.
\bseries{Proceedings of Machine Learning Research}
\bvolume{238}
\bpages{1927--1935}.
\bpublisher{{PMLR}}.
\end{binproceedings}
\endbibitem

\bibitem[\protect\citeauthoryear{Lorch et~al.}{2026}]{Lorch2026}
\begin{barticle}[author]
\bauthor{\bsnm{Lorch},~\bfnm{Lars}\binits{L.}},
  \bauthor{\bsnm{Zhang},~\bfnm{Jiaqi}\binits{J.}},
  \bauthor{\bsnm{Bunne},~\bfnm{Charlotte}\binits{C.}},
  \bauthor{\bsnm{Krause},~\bfnm{Andreas}\binits{A.}},
  \bauthor{\bsnm{Schölkopf},~\bfnm{Bernhard}\binits{B.}} \AND
  \bauthor{\bsnm{Uhler},~\bfnm{Caroline}\binits{C.}}
(\byear{2026}).
\btitle{Latent Causal Diffusions for Single-Cell Perturbation Modeling}.
\end{barticle}
\endbibitem

\bibitem[\protect\citeauthoryear{Nickl and Ray}{2020}]{Nickl2020}
\begin{barticle}[author]
\bauthor{\bsnm{Nickl},~\bfnm{Richard}\binits{R.}} \AND
  \bauthor{\bsnm{Ray},~\bfnm{Kolyan}\binits{K.}}
(\byear{2020}).
\btitle{Nonparametric statistical inference for drift vector fields of
  multi-dimensional diffusions}.
\bjournal{Ann. Statist.}
\bvolume{48}
\bpages{1383--1408}.
\bdoi{10.1214/19-AOS1851}
\end{barticle}
\endbibitem

\bibitem[\protect\citeauthoryear{Pedretscher
  et~al.}{2019}]{Pedretscher2019FokkerPlanckMaximumLikelihood}
\begin{barticle}[author]
\bauthor{\bsnm{Pedretscher},~\bfnm{B.}\binits{B.}},
  \bauthor{\bsnm{Kaltenbacher},~\bfnm{B.}\binits{B.}} \AND
  \bauthor{\bsnm{Pfeiler},~\bfnm{O.}\binits{O.}}
(\byear{2019}).
\btitle{Parameter identification and uncertainty quantification in stochastic
  state space models and its application to texture analysis}.
\bjournal{Appl. Numer. Math.}
\bvolume{146}
\bpages{38--54}.
\end{barticle}
\endbibitem

\bibitem[\protect\citeauthoryear{Pereverzyev}{2022}]{Pereverzyev2022}
\begin{bbook}[author]
\bauthor{\bsnm{Pereverzyev},~\bfnm{Sergei}\binits{S.}}
(\byear{2022}).
\btitle{An introduction to artificial intelligence based on reproducing kernel
  {H}ilbert spaces}.
\bseries{Compact Textbooks in Mathematics}.
\bpublisher{Birkh\"auser/Springer}, \baddress{Cham}.
\end{bbook}
\endbibitem

\bibitem[\protect\citeauthoryear{Peters
  et~al.}{2022}]{Peters2022CausalModelsDynamicalSystems}
\begin{bincollection}[author]
\bauthor{\bsnm{Peters},~\bfnm{Jonas}\binits{J.}},
  \bauthor{\bsnm{Bauer},~\bfnm{Stefan}\binits{S.}} \AND
  \bauthor{\bsnm{Pfister},~\bfnm{Niklas}\binits{N.}}
(\byear{2022}).
\btitle{Causal Models for Dynamical Systems}.
In \bbooktitle{Probabilistic and Causal Inference: The Works of Judea Pearl}.
\bseries{{ACM} Books}
\bvolume{36}
\bpages{671--690}.
\bpublisher{{ACM}}.
\end{bincollection}
\endbibitem

\bibitem[\protect\citeauthoryear{Peters et~al.}{2017}]{Peters2017CausalityBook}
\begin{bbook}[author]
\bauthor{\bsnm{Peters},~\bfnm{Jonas}\binits{J.}},
  \bauthor{\bsnm{Janzing},~\bfnm{Dominik}\binits{D.}} \AND
  \bauthor{\bsnm{Sch\"olkopf},~\bfnm{Bernhard}\binits{B.}}
(\byear{2017}).
\btitle{Elements of causal inference}.
\bseries{Adaptive Computation and Machine Learning}.
\bpublisher{MIT Press}, \baddress{Cambridge, MA}.
\end{bbook}
\endbibitem

\bibitem[\protect\citeauthoryear{Pratapa et~al.}{2020}]{Pratapa2020b}
\begin{barticle}[author]
\bauthor{\bsnm{Pratapa},~\bfnm{Aditya}\binits{A.}},
  \bauthor{\bsnm{Jalihal},~\bfnm{Amogh~P.}\binits{A.~P.}},
  \bauthor{\bsnm{Law},~\bfnm{Jeffrey~N.}\binits{J.~N.}},
  \bauthor{\bsnm{Bharadwaj},~\bfnm{Aditya}\binits{A.}} \AND
  \bauthor{\bsnm{Murali},~\bfnm{T.~M.}\binits{T.~M.}}
(\byear{2020}).
\btitle{Benchmarking algorithms for gene regulatory network inference from
  single-cell transcriptomic data}.
\bjournal{Nature Methods}
\bvolume{17}
\bpages{147--154}.
\end{barticle}
\endbibitem

\bibitem[\protect\citeauthoryear{Rey-Bellet and
  Spiliopoulos}{2015}]{ReyBellet2015IrreversibleLargeDeviation}
\begin{barticle}[author]
\bauthor{\bsnm{Rey-Bellet},~\bfnm{Luc}\binits{L.}} \AND
  \bauthor{\bsnm{Spiliopoulos},~\bfnm{Konstantinos}\binits{K.}}
(\byear{2015}).
\btitle{Irreversible {L}angevin samplers and variance reduction: a large
  deviations approach}.
\bjournal{Nonlinearity}
\bvolume{28}
\bpages{2081--2103}.
\end{barticle}
\endbibitem

\bibitem[\protect\citeauthoryear{Rohbeck et~al.}{2024}]{Rohbeck2024Bicycle}
\begin{binproceedings}[author]
\bauthor{\bsnm{Rohbeck},~\bfnm{Martin}\binits{M.}},
  \bauthor{\bsnm{Clarke},~\bfnm{Brian}\binits{B.}},
  \bauthor{\bsnm{Mikulik},~\bfnm{Katharina}\binits{K.}},
  \bauthor{\bsnm{Pettet},~\bfnm{Alexandra}\binits{A.}},
  \bauthor{\bsnm{Stegle},~\bfnm{Oliver}\binits{O.}} \AND
  \bauthor{\bsnm{Ueltzh\"offer},~\bfnm{Kai}\binits{K.}}
(\byear{2024}).
\btitle{Bicycle: Intervention-Based Causal Discovery with Cycles}.
In \bbooktitle{Proceedings of the Third Conference on Causal Learning and
  Reasoning}
\bvolume{236}
\bpages{209--242}.
\bpublisher{PMLR}.
\end{binproceedings}
\endbibitem

\bibitem[\protect\citeauthoryear{Rosasco et~al.}{2010}]{Rosasco2010}
\begin{barticle}[author]
\bauthor{\bsnm{Rosasco},~\bfnm{Lorenzo}\binits{L.}},
  \bauthor{\bsnm{Belkin},~\bfnm{Mikhail}\binits{M.}} \AND
  \bauthor{\bsnm{De~Vito},~\bfnm{Ernesto}\binits{E.}}
(\byear{2010}).
\btitle{On learning with integral operators}.
\bjournal{J. Mach. Learn. Res.}
\bvolume{11}
\bpages{905--934}.
\end{barticle}
\endbibitem

\bibitem[\protect\citeauthoryear{Song and Ermon}{2019}]{Song2019}
\begin{binproceedings}[author]
\bauthor{\bsnm{Song},~\bfnm{Yang}\binits{Y.}} \AND
  \bauthor{\bsnm{Ermon},~\bfnm{Stefano}\binits{S.}}
(\byear{2019}).
\btitle{Generative Modeling by Estimating Gradients of the Data Distribution}.
In \bbooktitle{Annual Conference on Neural Information Processing Systems 2019,
  NeurIPS 2019, Vancouver, BC, Canada}
\bpages{11895--11907}.
\end{binproceedings}
\endbibitem

\bibitem[\protect\citeauthoryear{Sriperumbudur
  et~al.}{2017}]{Sriperumbudur2017KernelExponentialFamily}
\begin{barticle}[author]
\bauthor{\bsnm{Sriperumbudur},~\bfnm{Bharath~K.}\binits{B.~K.}},
  \bauthor{\bsnm{Fukumizu},~\bfnm{Kenji}\binits{K.}},
  \bauthor{\bsnm{Gretton},~\bfnm{Arthur}\binits{A.}},
  \bauthor{\bsnm{Hyv{\"{a}}rinen},~\bfnm{Aapo}\binits{A.}} \AND
  \bauthor{\bsnm{Kumar},~\bfnm{Revant}\binits{R.}}
(\byear{2017}).
\btitle{Density Estimation in Infinite Dimensional Exponential Families}.
\bjournal{J. Mach. Learn. Res.}
\bvolume{18}
\bpages{57:1--57:59}.
\end{barticle}
\endbibitem

\bibitem[\protect\citeauthoryear{Steinwart and
  Christmann}{2008}]{Steinwart2008SVM}
\begin{bbook}[author]
\bauthor{\bsnm{Steinwart},~\bfnm{Ingo}\binits{I.}} \AND
  \bauthor{\bsnm{Christmann},~\bfnm{Andreas}\binits{A.}}
(\byear{2008}).
\btitle{Support vector machines}.
\bseries{Information Science and Statistics}.
\bpublisher{Springer}, \baddress{New York}.
\end{bbook}
\endbibitem

\bibitem[\protect\citeauthoryear{Sun and Kumar}{2014}]{Sun2014StatFP}
\begin{barticle}[author]
\bauthor{\bsnm{Sun},~\bfnm{Yifei}\binits{Y.}} \AND
  \bauthor{\bsnm{Kumar},~\bfnm{Mrinal}\binits{M.}}
(\byear{2014}).
\btitle{Numerical solution of high dimensional stationary {F}okker-{P}lanck
  equations via tensor decomposition and {C}hebyshev spectral differentiation}.
\bjournal{Comput. Math. Appl.}
\bvolume{67}
\bpages{1960--1977}.
\end{barticle}
\endbibitem

\bibitem[\protect\citeauthoryear{Sutherland et~al.}{2018}]{Sutherland2018}
\begin{binproceedings}[author]
\bauthor{\bsnm{Sutherland},~\bfnm{Danica~J.}\binits{D.~J.}},
  \bauthor{\bsnm{Strathmann},~\bfnm{Heiko}\binits{H.}},
  \bauthor{\bsnm{Arbel},~\bfnm{Michael}\binits{M.}} \AND
  \bauthor{\bsnm{Gretton},~\bfnm{Arthur}\binits{A.}}
(\byear{2018}).
\btitle{Efficient and principled score estimation with Nystr{\"{o}}m kernel
  exponential families}.
In \bbooktitle{{AISTATS} 2018}.
\bseries{Proceedings of Machine Learning Research}
\bpages{652--660}.
\bpublisher{{PMLR}}.
\end{binproceedings}
\endbibitem

\bibitem[\protect\citeauthoryear{Varando and
  Hansen}{2020}]{pmlr-v124-varando20a}
\begin{binproceedings}[author]
\bauthor{\bsnm{Varando},~\bfnm{Gherardo}\binits{G.}} \AND
  \bauthor{\bsnm{Hansen},~\bfnm{Niels~Richard}\binits{N.~R.}}
(\byear{2020}).
\btitle{Graphical continuous Lyapunov models}.
In \bbooktitle{Proceedings of the 36th Conference on Uncertainty in Artificial
  Intelligence (UAI)}
\bpages{989--998}.
\bpublisher{PMLR}.
\end{binproceedings}
\endbibitem

\bibitem[\protect\citeauthoryear{Varughese and Fatti}{2008}]{Varughese2008}
\begin{barticle}[author]
\bauthor{\bsnm{Varughese},~\bfnm{M.~M.}\binits{M.~M.}} \AND
  \bauthor{\bsnm{Fatti},~\bfnm{L.~P.}\binits{L.~P.}}
(\byear{2008}).
\btitle{Incorporating environmental stochasticity within a biological
  population model}.
\bjournal{Theoretical Population Biology}
\bvolume{74}
\bpages{115--129}.
\end{barticle}
\endbibitem

\bibitem[\protect\citeauthoryear{Vasilev}{2026}]{Viktor2026}
\begin{bmastersthesis}[author]
\bauthor{\bsnm{Vasilev},~\bfnm{Viktor}\binits{V.}}
(\byear{2026}).
\btitle{Nonparametric Estimation of Causal Lyapunov Models},
\btype{Master's thesis},
\bpublisher{Technical University of Munich}.
\end{bmastersthesis}
\endbibitem

\bibitem[\protect\citeauthoryear{Wibisono et~al.}{2024}]{Wibisono2024}
\begin{binproceedings}[author]
\bauthor{\bsnm{Wibisono},~\bfnm{Andre}\binits{A.}},
  \bauthor{\bsnm{Wu},~\bfnm{Yihong}\binits{Y.}} \AND
  \bauthor{\bsnm{Yang},~\bfnm{Kaylee~Yingxi}\binits{K.~Y.}}
(\byear{2024}).
\btitle{Optimal score estimation via empirical Bayes smoothing}.
In \bbooktitle{{COLT}}.
\bseries{Proceedings of Machine Learning Research}
\bvolume{247}
\bpages{4958--4991}.
\bpublisher{{PMLR}}.
\end{binproceedings}
\endbibitem

\bibitem[\protect\citeauthoryear{Zhao et~al.}{2026}]{Zhao2025}
\begin{barticle}[author]
\bauthor{\bsnm{Zhao},~\bfnm{Wenjun}\binits{W.}},
  \bauthor{\bsnm{Fertig},~\bfnm{Elana~J.}\binits{E.~J.}} \AND
  \bauthor{\bsnm{Stein-O{\textquoteright}Brien},~\bfnm{Genevieve}\binits{G.}}
(\byear{2026}).
\btitle{CycleGRN: Inferring Gene Regulatory Networks from Cyclic Flow Dynamics
  in Single-Cell RNA-seq}.
\bjournal{bioRxiv}.
\bdoi{10.1101/2025.11.12.688126}
\end{barticle}
\endbibitem

\bibitem[\protect\citeauthoryear{Zhou et~al.}{2020}]{Zhou2020ScoreEstimators}
\begin{binproceedings}[author]
\bauthor{\bsnm{Zhou},~\bfnm{Yuhao}\binits{Y.}},
  \bauthor{\bsnm{Shi},~\bfnm{Jiaxin}\binits{J.}} \AND
  \bauthor{\bsnm{Zhu},~\bfnm{Jun}\binits{J.}}
(\byear{2020}).
\btitle{Nonparametric Score Estimators}.
In \bbooktitle{{ICML} 2020}.
\bseries{Proceedings of Machine Learning Research}
\bvolume{119}
\bpages{11513--11522}.
\bpublisher{{PMLR}}.
\end{binproceedings}
\endbibitem

\end{thebibliography}
%% if required, the content of .bbl file can be included here once bbl is generated
%%\input sn-article.bbl

\newpage
\begin{appendices}

\section{Proofs for the identifiability section}

\theoremIdentifiability*
\begin{proof}
    We show the claim by induction over the topological ordering of $\mathcal{D}$, which we assume to be $1,\ldots, d$ without loss of generality as discussed. Setting $i=1$ in Lemma \ref{lemmaMargRecursion} and recalling that $\pa{1}=\emptyset$ by assumption, we obtain \begin{equation*}
        \int_{\mathbb{R}} \beta \varphi'\,\dd{p_1} = \frac{1}{2}\int_{\mathbb{R}^d} D_{11}(\mathbf{x}) \varphi''(x_1)p(\mathbf{x})\dd{\mathbf{x}}\qquad\forall\varphi\in C^\infty_0(\mathbb{R}),
    \end{equation*} where $\beta$ can be either of $b_1$ and $b_1^*$. Taking the difference of both versions of this equation, once for $\beta:=b_1$ and once for $\beta:=b_1^*$, we find $\int_\mathbb{R} \delta \varphi'\,\dd{p_1}=0$ where $\delta:= b_1 - b_1^*$. This means $\delta \cdot p_1$ has a weak derivative which is zero almost everywhere, implying that $\delta\cdot p_1=0$ almost everywhere (Lemma \ref{lemmaZeroWeakDeriv}). As $p_1>0$ by $p>0$, we conclude $\delta=0$ almost everywhere, i.e. $b_1 = b_1^*$ almost everywhere.

    Assume $\mathbf{b}_{S}=\mathbf{b}_S^*$ with $S=\{1,\ldots,i-1\}$ for some $i\geq 2$. In particular, $\mathbf{b}_{\pa{i}}=\mathbf{b}_{\pa{i}}^*$. Taking the difference of equation \eqref{eq:lemmaMargRecursion} with $\mathbf{b}$ and equation \eqref{eq:lemmaMargRecursion} with $\mathbf{b}^*$, we find $\int_{\mathbb{R}^{|\Pa{i}|}} \delta \,\partial_i \varphi\,\dd{p_{\Pa{i}}} = 0$ with $\delta:= b_i - b_i^*$. Since $\varphi$ ranges through $C^\infty_0(\mathbb{R}^{|\Pa{i}|})$, the function $\delta\cdot p_\Pa{i}$ has one zero weak derivative, which implies $\delta\cdot p_\Pa{i}=0$ almost everywhere (Lemma \ref{lemmaZeroWeakDeriv}). By $p_{\Pa{i}}>0$ we conclude $b_i = b_i^*$ almost everywhere. 
\end{proof}

\begin{restatable}[Feller's explosion test \citep{Bhattacharya2023DiffusionBook}]{lemma}{lemmaFellerExplosionTest}\label{lemmaFellerExplosionTest}
    Let $I(z):=\int_0^z\tfrac{2}{\sigma^2}b(v)\dd{v}$. Consider\begin{align*}
        \int_0^\infty e^{-I(y)}\int_0^y e^{I(v)}\dd{v}\dd{y} &&\text{and}&&\int_{-\infty}^0 e^{-I(y)}\int_y^0 e^{I(v)}\dd{v}\dd{y}.
    \end{align*} If the left integral is finite, $x_t$ solving the SDE \eqref{eq:univ_sde} explodes towards $+\infty$ with positive probability. Similarly, the right integral indicates explosion towards $-\infty$.
\end{restatable} % \citet[Thm 12.2]{Bhattacharya2023DiffusionBook}

\lemmaUnivExplodes*
\begin{proof}
    We consider the case $c>0$ and prove explosion towards $+\infty$ by showing that the left integral in Feller's explosion test (Lemma \ref{lemmaFellerExplosionTest}) is finite. When $c<0$, then the right integral is finite and explosion is towards $-\infty$.

    First, compute \begin{equation*}
        I(z) = \int_0^z \frac{2}{\sigma^2} b_c(v)\dd{v} = \int_0^z (\log p)'(v) + \frac{2c}{\sigma^2 p(v)} \dd{v} = \log p(z) - \log p(0) + \int_0^z \frac{2c}{\sigma^2 p(v)} \dd{v}.
    \end{equation*}
    We show that $\int_0^\infty \int_v^\infty e^{I(v) - I(y)}\dd{y}\dd{v} < \infty$. By Fubini's theorem, $\int_0^\infty e^{-I(y)}\int_0^y e^{I(v)}\dd{v}\dd{y}<\infty$, which is the claim. For $0\leq v\leq y$, we compute \begin{multline*}
        \int_0^\infty \int_v^\infty e^{I(v) - I(y)}\dd{y}\dd{v} = \int_0^\infty \int_v^\infty \frac{p(v)}{p(y)}e^{-\int_v^y\frac{2c}{\sigma^2p(w)}\dd{w}}\dd{y}\dd{v} = \\
        \int_0^\infty \frac{-\sigma^2 p(v)}{2}\int_v^\infty \frac{\dd{}}{\dd{y}}\left( e^{-\int_v^y \frac{2c}{\sigma^2 p(w)}\dd{w}}\right)\dd{y}\dd{v} = \\
        \int_0^\infty \frac{\sigma^2}{2} p(v) \left(1 - e^{-\int_v^\infty \frac{2c}{\sigma^2 p(w)}\dd{w}}\right)\dd{v} \leq \int_0^\infty \frac{\sigma^2}{2} p(v) \cdot 1\, \dd{v} < \infty.
    \end{multline*}
\end{proof}

\lemmaUniqueOneDim*
\begin{proof}
    Taking the difference of the Fokker-Planck equation \eqref{eq:stat_FP} for $b$ and $b^*$ respectively, we find \begin{equation*}
        \int_\mathbb{R} ((b - b^*)p)(x)\,\varphi'(x)\dd{x} = 0 \qquad\forall\varphi\in C^\infty_0(\mathbb{R}).
    \end{equation*} From Lemma \ref{lemmaZeroWeakDeriv} it follows that $(b - b^*)p = 0$ almost everywhere. By $p>0$ it follows that $b=b^*$ almost everywhere.
\end{proof}

\lemmaMargRecursion*
\begin{proof}
    Abbreviate $A:=\Pa{i}$. When $|A|=d$, the equation to prove is simply the Fokker-Planck equation \eqref{eq:stat_FP} which holds by assumption. In the following, we consider the case $|A|<d$.
    
    For $n\in\mathbb{N}$, define the approximation $\varphi_n(\mathbf{x}) := \varphi(\mathbf{x}) \cdot \gamma_n(\mathbf{x}_{-A})$ with $\gamma_n\in C^\infty_0(\mathbb{R}^{d-|A|})$ from Lemma \ref{lemmaOneApprox}. Particularly, $C:= \sup_{n\in\mathbb{N}} \|\gamma_n\|_{W^{2,\infty}} < \infty$ and $\varphi_n\in C^\infty_0(\mathbb{R}^d)$, which implies \begin{equation}\label{eq:proof_lemmaMargRecursion}
        \int_{\mathbb{R}^d}\mathbf{b}^T \nabla\varphi_n\dd{p} = \frac{1}{2}\int_{\mathbb{R}^d}\langle \mathbf{D}, \nabla^2\varphi_n\rangle_F\,\dd{p}
    \end{equation} by the Fokker-Planck relation \eqref{eq:stat_FP}. 
    
    We examine the limit as $n\to\infty$. Let $j\leq d$. 
    When $j\in A$, we have $|\partial_j\varphi_n(\mathbf{x})| = |\partial_j \tilde{\varphi}(\mathbf{x}_{A}) \cdot \gamma_n(\mathbf{x}_{-A})| \leq \|\partial_j \tilde{\varphi}\|_{L^\infty} \cdot C < \infty$ and $\partial_j \varphi_n(\mathbf{x})\to \partial_j\tilde{\varphi}(\mathbf{x}_{-A})\cdot1 = \varphi(\mathbf{x})$ as $n\to\infty$ for all $\mathbf{x}\in\mathbb{R}^d$. The dominated convergence theorem implies $\int_{\mathbb{R}^d}b_j\partial_j\varphi_n\,\dd{p} \to \int_{\mathbb{R}^d}b_j\partial_j \varphi \, \dd{p}$ as $n\to\infty$.

    When $j\notin A$, we have $|\partial_j\varphi_n(\mathbf{x})| = |\tilde{\varphi}(\mathbf{x}_{A}) \cdot \partial_j\gamma_n(\mathbf{x}_{-A})| \leq \| \tilde{\varphi}\|_{L^\infty} \cdot C < \infty$ and $\partial_j \varphi_n(\mathbf{x})\to \tilde{\varphi}(\mathbf{x}_{-A})\cdot0 = 0$ as $n\to\infty$ for all $\mathbf{x}\in\mathbb{R}^d$. The dominated convergence theorem implies $\int_{\mathbb{R}^d}b_j\partial_j\varphi_n\,\dd{p} \to 0$ as $n\to\infty$.

    With a similar case distinction for $\langle\mathbf{D}, \nabla^2\varphi_n\rangle_F$, by letting $n\to\infty$ in equation \eqref{eq:proof_lemmaMargRecursion} we obtain that \begin{equation*}
        \int_{\mathbb{R}^d} \mathbf{b}_{A}^T \nabla_{A}\varphi \,\dd{p} = \int_{\mathbb{R}^d}\frac{1}{2}\langle\mathbf{D}_{A, A}, \nabla_{A}^2 \varphi\rangle_F\,\dd{p}.
    \end{equation*} All that remains for proving the Lemma's claim is to solve for the $b_i\partial_i\varphi$-term and to recall that $\int_{\mathbb{R}^d}b_i\partial_i\varphi\,\dd{p} = \int_{\mathbb{R}^{|A|}}b_i \partial_i \varphi\,\dd{p_{A}}$ since both $b_i$ and $\partial_j \varphi$ only depend on $\mathbf{x}_{A}$.
\end{proof}

\section{Proofs and details for the theoretical analysis}
\subsection{Setup}
\lemmaZetaUnbiased*
\begin{proof}
    Let $y\in\mathbb{R}^d$, let $i\leq d$, and abbreviate $A:= \Pa{i}$. By assumption, there is $\tilde{K}\in W^{2,\infty}(\mathbb{R}^{|A|})$ such that $\tilde{K}(\mathbf{x}_A) = \mathcal{K}_{i,\mathbf{y}}(\mathbf{x})$ almost everywhere. As $p$ was assumed to be bounded, there is a sequence $(\tilde{K}_n)_{n\in\mathbb{N}}\subset C^\infty_0(\mathbb{R}^{|A|})$ with $\|\tilde{K} - \tilde{K_n}\|_{W^{2,2}(\mathbb{R}^{|A|}, p_A)}\to 0$ as $n\to\infty$ by Lemma \ref{lemmaDenseInInftySobolev}. Set $K_n(\mathbf{x}):= \tilde{K}_n(\mathbf{x}_A)$. By Lemma \ref{lemmaMargRecursion}, we have \begin{equation}\label{eq:proof_lemmaParentMarg}
        \int_{\mathbb{R}^{|A|}} b_i \, \partial_i K_n \, \dd{p_A} = - \int_{\mathbb{R}^d}(\mathbf{b}^*_{\pa{i}})^T \nabla_{\pa{i}} K_n \, \dd{p} + \frac{1}{2}\int_{\mathbb{R}^d}\langle\mathbf{D}_{A,A}, \nabla_{A}^2 K_n\rangle_F\,\dd{p}.
    \end{equation} Using Hölder's inequality, we have \begin{multline*}
        \left|\int_{\mathbb{R}^d} (\mathbf{b}_A^*)^T\nabla_A \mathcal{K}_{i,\mathbf{y}}\,\dd{p} - \int_{\mathbb{R}^d}(\mathbf{b}_A^*)^T\nabla_A K_n\,\dd{p}\right| \leq \int_{\mathbb{R}^d} \|\mathbf{b}_A^*\|_2 \|\nabla_A \mathcal{K}_{i,\mathbf{y}} - \nabla_A K_n\|_2 \,\dd{p} 
        \\\leq \|\mathbf{b}^*\|_{L^2(\mathbb{R}^d, \mathbb{R}^d, p)} \cdot \|\tilde{K} - \tilde{K}_n\|_{{W^{2,2}(\mathbb{R}^{|A|}, p_A)}} \xrightarrow{n\to\infty} 0.
    \end{multline*} Together with an analogous result for $\langle\mathbf{D},\nabla^2_A K_n\rangle_F$, this implies the Lemma's claim by letting $n\to\infty$ in equation \eqref{eq:proof_lemmaParentMarg}.
\end{proof}

\subsection{Statistical analysis}

\lemmaZetaGeneralizationBound*
\begin{proof}
    By the triangle inequality and $\|\partial_j \mathcal{K}_{i,\cdot}(\cdot)\|_{L^\infty}\leq \kappa$ for all $j\in\pa{i}$, we find that\begin{equation*}
        \|\hat{\zeta}_i - \zeta_i\|_{L^2(p)} \leq \left\|\frac{1}{n}\sum_{j=1}^nz_i(\mathbf{x}_j, \cdot) - \zeta_i\right\|_{L^2(p)} + \kappa \sum_{j\in\pa{i}} \|\hat{\mathbf{b}}_j - \mathbf{b}_j^*\|_{E^n}.
    \end{equation*} Note $\ex{z_i(\mathbf{x}_1,\cdot)} = \zeta_i(\cdot)$ using Lemma \ref{lemmaZetaUnbiased} since $\partial_i\mathcal{K}_{i,\mathbf{y}} = k_i(\cdot,\mathbf{y})$ almost everywhere for all $\mathbf{y}\in\mathbb{R}^d$, and that  \begin{equation*}
        \|\zeta_i - z_i(\mathbf{x}_1,\cdot)\|_{L^2(p)} \leq % \|z_i(\mathbf{x}_1,\cdot) - \zeta_i\|_{L^\infty(\mathbb{R}^d)} \leq
        \|\zeta_i\|_{L^\infty(\mathbb{R}^d)} + \sum_{j\in\pa{i}}\kappa B + \sum_{j,k\in \Pa{i}} \kappa B \leq a_i^2\kappa B.
    \end{equation*} Then, by Hoeffding's inequality for separable Hilbert spaces \citep{Pereverzyev2022} % e.g., Proposition 4.1
    it holds that \begin{equation*}
        \left\|\frac{1}{n}\sum_{j=1}^nz_i(\mathbf{x}_j, \cdot) - \zeta_i\right\|_{L^2(p)} \leq \frac{a_i^2\kappa B \sqrt{2 \log(\tfrac{2}{\delta})}}{\sqrt{n}}
    \end{equation*} with probability at least $1-\delta$. 
\end{proof}

\lemmaZetaTrainingBound*
\begin{proof}
    By $\hat{\zeta}_i = \frac{1}{n}\sum_{k=1}^n z_i(\mathbf{x}_k,\cdot) + \sum_{j\in\pa{i}}\frac{1}{n}\sum_{l=1}^n (\hat{b}_j(\mathbf{x}_l) - b_j^*(\mathbf{x}_l))\cdot\partial_{x_j}\mathcal{K}_{i,\cdot}(\mathbf{x}_l)$, the triangle inequality gives that \begin{equation}\label{eq:proof_lemmaZetaTrainingBound1}
        \|\hat{\bs{\zeta}}_i- \bs{\zeta}_i\|_{E^n} \leq \sqrt{\frac{1}{n}\sum_{j=1}^n \left(\frac{1}{n}\sum_{k=1}^n z_i(\mathbf{x}_k,\mathbf{x}_j) - \zeta_i(\mathbf{x}_j)\right)^2} + \kappa \sum_{j\in\pa{i}} \|\hat{\mathbf{b}}_j - \mathbf{b}_j^*\|_{E^n}.
    \end{equation}
    Define the following kernel $h_\mathbf{x}$ indexed by the parameter $\mathbf{x}\in\mathbb{R}^d$: \begin{equation*}
        h_\mathbf{x}(\mathbf{a},\tilde{\mathbf{a}}):= \left(z_i(\mathbf{a}, \mathbf{x}) - \zeta_i(\mathbf{x})\right)\cdot \left(z_i(\mathbf{\tilde{a}}, \mathbf{x}) - \zeta_i(\mathbf{x})\right).
    \end{equation*} Further define the index sets $S(j):= \{(k,l)\in\{1,\ldots,n\}^2 \mid k\neq l, k\neq j, l\neq j\}$. We multiply out the square under the root in \eqref{eq:proof_lemmaZetaTrainingBound1}. Using $\sup_{\mathbf{a}, \tilde{\mathbf{a}}\in\mathbb{R}^d} |h_\mathbf{x}(\mathbf{a}, \tilde{\mathbf{a}})| \leq (a_i^2 \kappa B)^2$ and that $\{1,\ldots,n\}^2\setminus S(j)$ contains at most $3n$ elements, we find that\begin{equation*}
        \frac{1}{n}\sum_{j=1}^n \left(\frac{1}{n}\sum_{k=1}^n z_i(\mathbf{x}_k,\mathbf{x}_j) - \zeta_i(\mathbf{x}_j)\right)^2 \leq \frac{3 a_i^4 \kappa^2 B^2}{n} + \frac{1}{n}\sum_{j=1}^n \frac{1}{n^2}\sum_{k,l\in S(j)} h_{\mathbf{x}_j}(\mathbf{x}_k, \mathbf{x}_l).
    \end{equation*} Next, we apply a $U$-statistic bound to the second term. Let $\mathbf{x}\in\mathbb{R}^d$ be a deterministic constant. The kernel $h_\mathbf{x}$ is symmetric and bounded by $h_\infty := (a_i^2 \kappa B)^2$ independently of $\mathbf{x}$. For $k\neq l$, we have $\ex{h_\mathbf{x}(\mathbf{x}_k, \mathbf{x}_l)} = 0$ and $\mathbf{a}\mapsto \ex{h_\mathbf{x}(\mathbf{a}, \mathbf{x}_l)}$ is the constant zero function, hence $h_\mathbf{x}$ is $p$-canonical. By \citet{Arcones1993}, %To check $P$-canonical, see directly above equation (1.5) in Arcones. The symmetrization $S_m$ does not change the symmetric $h$; note that they prove 2.3.d only for symmetric kernels 
    there are positive global constants $c_1, c_2$ such that for any $j\in \{1,\ldots, d\}$ and $t>0$: \begin{multline*}
        \pr{\left|\frac{1}{n-1}\sum_{k,l\in S(j)}h_\mathbf{x}(\mathbf{x}_k, \mathbf{x}_l)\right| \geq t} \leq c_1 \exp\left(- \frac{c_2 t}{h_\infty}\right) \Rightarrow\\
        \pr{\left|\frac{1}{n^2}\sum_{k,l\in S(j)}h_\mathbf{x}(\mathbf{x}_k, \mathbf{x}_l)\right| \geq t} \leq c_1 \exp\left(- \frac{c_2 t \cdot n^2}{h_\infty\cdot (n-1)}\right) \leq c_1 \exp\left(- \frac{c_2 t\cdot n}{h_\infty}\right).
    \end{multline*} By independence of $\mathbf{x}_j$ and $(\mathbf{x}_k, \mathbf{x}_l)_{k,l\in S(j)}$ the above inequality also holds for the conditional probability given $\mathbf{x}_j = \mathbf{x}$. % If A, B are indep, then P(A) = P(A|B)
    Taking the expectation with respect to $\mathbf{x}_j$, we have shown: \begin{equation*}
         \pr{\left|\frac{1}{n^2}\sum_{k,l\in S(j)}h_\mathbf{x_j}(\mathbf{x}_k, \mathbf{x}_l)\right| \geq t} \leq c_1 \exp\left(- \frac{c_2 t\cdot n}{h_\infty}\right).
    \end{equation*} Using the triangle inequality and a union bound, it follows that \begin{equation*}
        \pr{\left|\frac{1}{n}\sum_{j=1}^n \frac{1}{n^2}\sum_{k,l\in S(j)} h_{\mathbf{x}_j}(\mathbf{x}_k, \mathbf{x}_l)\right| \geq t + \frac{h_\infty \log(n)}{c_2 n}} \leq c_1 \exp\left(- \frac{c_2 t\cdot n}{h_\infty}\right).
    \end{equation*} The claim follows by taking $t:= \tfrac{h_\infty \log(\max(c_1, 2)/\delta)}{c_2 n}$ and from sub-additivity of the square root.
\end{proof}

\lemmaMercerKernelApprox*
\begin{proof}
    Apply Hoeffding's inequality on the space of Hilbert-Schmidt operators using $\|L_i\|_{\text{HS}}\leq \kappa$ from Section \ref{subsec:RKHS_theory} % Bound HS by nuclear norm
    and $\|\hat{L}_i\|_{\text{HS}}\leq \kappa$. % Discrete probabiliy measure
    Also see \citep{Pereverzyev2022}.
\end{proof}

\lemmaDeterministicTrainingBound*
\begin{proof}
    Fix $i\in\{1,\ldots,d\}$; we write $\lambda$ instead of $\lambda_i$ to reduce indices. To apply functional analysis tools, we note Lemma \ref{lemmaSampleTikhonovInverse} may equivalently be stated as \begin{equation}\label{eq:proof_DeterministicTrainingBound1}
        (\hat{L}_i + \lambda I)^{-1} f = -\frac{1}{\lambda}J_\mathbf{x}^*(J_\mathbf{x} J_\mathbf{x}^* + \lambda I)^{-1} J_\mathbf{x} f + \frac{f}{\lambda}
    \end{equation} for all $f\in \mathcal{H}_i$. Further, note our definition of $\hat{b}_i$ can be written as \begin{equation}\label{eq:proof_DeterministicTrainingBound2}
        \hat{b}_i = \frac{1}{\lambda}J_\mathbf{x}^*(J_\mathbf{x} J_\mathbf{x}^* + \lambda I)^{-1} \hat{\bs{\zeta}}_i - \frac{1}{\lambda}\hat{\zeta}_i.
    \end{equation}Our proof uses the following decomposition: 
    \begin{equation}\label{eq:proof_DeterministicTrainingBound3}
    \begin{aligned}
        b_i^* - \hat{b}_i &= b_i^* - b_{i,\lambda}\\
        &+b_{i,\lambda} - (\hat{L}_i + \lambda I)^{-1}(-\zeta_i) \\
        &+ (\hat{L}_i + \lambda I)^{-1}(-\zeta_i) - \hat{b}_i,
    \end{aligned}\end{equation} where $b_{i,\lambda}:= (L_i + \lambda I)^{-1}(-\zeta_i) = (L_i + \lambda I)^{-1}J^* b_i^*$ is the Tikhonov regularizer of $b_i^*$.

    For the bound on $\hat{\mathbf{b}}_i - \mathbf{b}_i^*$, we evaluate both sides of \eqref{eq:proof_DeterministicTrainingBound3} at the random variables $\mathbf{x}_1,\ldots,\mathbf{x}_n$ and take $\|\cdot\|_{E^n}$ on both sides. As that Evaluation of RKHS elements can be written by applying the sampling operator $J_\mathbf{x}\colon\mathcal{H}_i\to E^n$, we have found that \begin{multline}\label{eq:proof_DeterministicTrainingBound4}
        \|\mathbf{b}_i^* - \hat{\mathbf{b}}_i\|_{E^n} \leq \|\mathbf{b}_i^* - \mathbf{b}_{i,\lambda}\|_{E^n} + \|J_\mathbf{x}((L_i + \lambda I)^{-1} - (\hat{L_i} + \lambda I)^{-1})\zeta_i\|_{E^n} + \\
        \frac{1}{\lambda} \|J_\mathbf{x}J_\mathbf{x}^*(J_\mathbf{x} J_\mathbf{x}^* + \lambda I)^{-1}(\hat{\bs{\zeta}}_i - \bs{\zeta}_i)\|_{E^n} + \frac{1}{\lambda} \|\hat{\bs{\zeta}}_i - \bs{\zeta}_i\|_{E^n},
    \end{multline} where we denoted $\mathbf{b}_{i,\lambda}=(b_{i,\lambda}(\mathbf{x}_1), \ldots,b_{i,\lambda}(\mathbf{x}_n))\in\mathbb{R}^n$. Using $A^{-1} - B^{-1} = B^{-1}(B - A)A^{-1}$ for invertible operators $A,B$, we have that\begin{equation}\label{eq:proof_DeterministicTrainingBound5}
    (L_i + \lambda I)^{-1} -(\hat{L}_i + \lambda I)^{-1}= (\hat{L}_i + \lambda I)^{-1}(\hat{L}_i - L_i)(L_i + \lambda I)^{-1}. 
\end{equation}Recall that $\hat{L}_i = J_\mathbf{x}^*J_\mathbf{x}$, that $L_i=J^*J$ and finally that $\zeta_i = -J^*b_i^*$. We therefore bound: \begin{align}\label{eq:proof_DeterministicTrainingBound6}
     \|J_\mathbf{x}(J_\mathbf{x}^*J_\mathbf{x} + \lambda I)^{-1}\|_{\mathcal{L}(\mathcal{H}_i, E^n)} & \leq \frac{1}{\sqrt{\lambda}} \\\label{eq:proof_DeterministicTrainingBound7}
     \|(J^*J+\lambda I)^{-1}J^*\|_{\mathcal{L}(L^2(p), \mathcal{H}_i)} & \leq \frac{1}{\sqrt{\lambda}}.
\end{align} One way to see \eqref{eq:proof_DeterministicTrainingBound8} is to multiply the operator with its adjoint and to subsequently apply the spectral theorem for $J^*J$. Using $\sup_{\mu\geq 0} \tfrac{\mu}{(\mu + \lambda)^2}\leq \tfrac{1}{\lambda}$ and taking the square root yields \eqref{eq:proof_DeterministicTrainingBound8}, also see \citet{Vito2005}. Equation \eqref{eq:proof_DeterministicTrainingBound7} follows by replacing $J$ with $J_\mathbf{x}$. Applying an operator norm factorization bound, the second term on the right hand side of \eqref{eq:proof_DeterministicTrainingBound4} is bounded by $\varepsilon_3 \|b_i^*\|_{L^2(p)}\cdot\lambda^{-1}$.  Using  $\|J_\mathbf{x}J_\mathbf{x}^*(J_\mathbf{x} J_\mathbf{x}^* + \lambda I)^{-1}\|_{\mathcal{L}(E^n, E^n)}\leq 1$ and $\|\hat{\bs{\zeta}}_i - \bs{\zeta}_i\|_{E^n} \leq \varepsilon_1$ by assumption, the bottom line of \eqref{eq:proof_DeterministicTrainingBound4} is bounded by $2 \varepsilon_1\cdot\lambda^{-1}$. Thus, \begin{equation*}
    \|\mathbf{b}_i^* - \hat{\mathbf{b}}_i\|_{E^n} \leq \|\mathbf{b}_i^* - \mathbf{b}_{i,\lambda}\|_{E^n} + \frac{\varepsilon_3 \|b_i^*\|_{L^2(p)}}{\lambda} + \frac{2\varepsilon_1}{\lambda}
\end{equation*} as claimed. 

Next, we study $\hat{b}_i - b_i^*$ in the $L^2(p)$ norm and again use the decomposition \eqref{eq:proof_DeterministicTrainingBound3}. Making the embedding $J\colon\mathcal{H}_i\to L^2(p)$ explicit for all elements from $\mathcal{H}_i$ on the right side of \eqref{eq:proof_DeterministicTrainingBound3} and applying the triangle inequality, we find that \begin{multline}\label{eq:proof_DeterministicTrainingBound8}
    \|b_i^* - \hat{b}_i\|_{L^2(p)} \leq \|b_i^* - Jb_{i,\lambda}\|_{L^2(p)} + \|J((L_i + \lambda I)^{-1} - (\hat{L_i} + \lambda I)^{-1})\zeta_i\|_{L^2(p)} + \\
    \frac{1}{\lambda}\|J J_\mathbf{x}^*(J_\mathbf{x} J_\mathbf{x}^* + \lambda I)^{-1}(\hat{\bs{\zeta}}_i - \bs{\zeta}_i)\|_{L^2(p)} + \frac{1}{\lambda}\|\hat{\zeta}_i - J\zeta_i\|_{L^2(p)}.
\end{multline} Treating the second term on the right hand side with \eqref{eq:proof_DeterministicTrainingBound5} and \eqref{eq:proof_DeterministicTrainingBound7}, we find that\begin{equation*}
    \|J((L_i + \lambda I)^{-1} - (\hat{L_i} + \lambda I)^{-1})\zeta_i\|_{L^2(p)} \leq \|J(\hat{L}_i + \lambda I)^{-1}\|_{\mathcal{L}(\mathcal{H}_i, L^2(p))} \cdot \varepsilon_3 \cdot \frac{1}{\sqrt{\lambda}} \cdot \|b_i^*\|_{L^2(p)}.
\end{equation*}Let $O_\lambda:= (\hat{L}_i + \lambda I)^{-1}$. To resolve the mismatch between population operator $J$ and sample operator $\hat{L}_i$, we consider \begin{equation*}
    (JO_\lambda)^*JO_\lambda = O_\lambda J_\mathbf{x}^* J_\mathbf{x} O_\lambda + O_\lambda (J^*J - J_\mathbf{x}^* J_\mathbf{x})O_\lambda. 
\end{equation*}Using $\|O_\lambda J_\mathbf{x}^*\|_{\mathcal{L}(E^n,\mathcal{H}_i)} = \|J_\mathbf{x} O_\lambda\|_{\mathcal{L}(\mathcal{H}_i, E^n)}\leq \tfrac{1}{\sqrt{\lambda}}$ by \eqref{eq:proof_DeterministicTrainingBound6} and $O_\lambda^*=O_\lambda$, and using $\|O_\lambda\|_{\mathcal{L}(\mathcal{H}_i, \mathcal{H}_i)} \leq \tfrac{1}{\lambda}$, we have shown \begin{equation*}
    \|J(\hat{L}_i + \lambda I)^{-1}\|_{\mathcal{L}(\mathcal{H}_i, L^2(p))} = \|JO_\lambda\|_{\mathcal{L}(\mathcal{H}_i, L^2(p))} \leq \sqrt{ \frac{1}{\lambda} + \frac{\varepsilon_3}{\lambda^2}} \leq \frac{1}{\sqrt{\lambda}} + \frac{\sqrt{\varepsilon_3}}{\lambda}.
\end{equation*}% where we used $\|A\| = \sqrt{\|A^*A\|}$ for operators $A$. 
We treat the third term in \eqref{eq:proof_DeterministicTrainingBound8} similarly. With $U_\lambda := (J_\mathbf{x}J_\mathbf{x}^* + \lambda I)^{-1}$, we find that \begin{equation*}
(JJ_\mathbf{x}^*U_\lambda)^*JJ_\mathbf{x}^*U_\lambda = U_\lambda J_\mathbf{x} J_\mathbf{x}^* J_\mathbf{x} J_\mathbf{x}^*U_\lambda + U_\lambda J_\mathbf{x} (J^*J - J_\mathbf{x}^* J_\mathbf{x})J_\mathbf{x}^*U_\lambda.
\end{equation*} Using $\|J_\mathbf{x} J_\mathbf{x}^* U_\lambda\|_{\mathcal{L}(\mathcal{H}_i, E^n)} \leq 1$ and $\|J_\mathbf{x}^*U_\lambda\|_{\mathcal{L}(E^n, \mathcal{H}_i)} \leq \tfrac{1}{\sqrt{\lambda}}$, we have shown that \begin{equation*}
    \|J J_\mathbf{x}^*(J_\mathbf{x}^* J_\mathbf{x} + \lambda I)^{-1}(\hat{\bs{\zeta}}_i - \bs{\zeta}_i)\|_{L^2(p)} \leq \varepsilon_1\cdot\sqrt{1 + \frac{\varepsilon_3}{\lambda}}.
\end{equation*} Collecting all terms, we have shown that \begin{equation*}
    \|b_i^* - \hat{b}_i\|_{L^2(p)} \leq \|b_i^* - Jb_{i,\lambda}\|_{L^2(p)} + \left(\frac{\varepsilon_3}{\lambda} + \left(\frac{\varepsilon_3}{\lambda}\right)^{3/2}\right)\|b_i^*\|_{L^2(p)} + \frac{\varepsilon_1}{\lambda} + \frac{\varepsilon_1 \sqrt{\varepsilon_3}}{\lambda^{3/2}} + \frac{\varepsilon_2}{\lambda}
\end{equation*}as claimed.
\end{proof} 

\theoremDriftConcentration*
\begin{proof}
    Any $\mathcal{O}(\cdot)$-notation we use is with respect to $n\to\infty$. We define the shorthand $\PaTrEr:= \sum_{j\in\pa{i}} \|\hat{\mathbf{b}}_j - \mathbf{b}_j^*\|_{E^n}$ for the cumulative parent training error. Further, note that for any any $0\leq \alpha \leq \gamma\leq 1$ and $\beta \geq 2$, it holds that \begin{equation}\label{eq:proof_StatsThm1}
    \sup_{\delta\in(0,1)}\frac{\log^\alpha(\beta/\delta) }{\log^\gamma(2/\delta)} \leq \frac{\log^\alpha(\beta/2)}{\log^\gamma(2)} + \log(2)^{\alpha-\gamma}.
    \end{equation} % Use $\log^\alpha(\beta/\delta) = (\log(\beta/2) + \log(2/\delta))^\alpha \leq \log^\alpha(\beta/2) + \log^\alpha(2/\delta)$

    Let $\delta\in(0,1)$ and insert $\delta/3$ into Lemmas \ref{lemmaMercerKernelApprox}, \ref{lemmaZetaTrainingBound}, and \ref{lemmaZetaGeneralizationBound}. By equation \eqref{eq:proof_StatsThm1}, we may unify $\log^{1/2} \tfrac{6}{\delta} = \log^{1/2}(\tfrac{2}{\delta}) \cdot \mathcal{O}(1)$ and $a_i^2\kappa B\sqrt{3} = \log^{1/2}(\tfrac{2}{\delta})\cdot\mathcal{O}(1)$, to name two examples. Finally, applying a union bound over $\tfrac{\delta}{3}$, we obtain an event of probability at least $1-\delta$ on which simultaneously 
    \begin{align}
        \|\hat{\bs{\zeta}}_i- \bs{\zeta}_i\|_{E^n} \leq& \frac{\mathcal{O}(\sqrt{\log n})}{\sqrt{n}}\cdot\log^{1/2}\tfrac{2}{\delta} + \kappa\PaTrEr ,\\ 
        \|\hat{\zeta}_i - \zeta_i\|_{L^2(p)} \leq & \frac{\mathcal{O}(1)}{\sqrt{n}}\cdot\log^{1/2} \tfrac{2}{\delta} + \kappa\PaTrEr,\\
        \|\hat{L}_i - L_i\|_{\mathcal{L}(\mathcal{H}_i, \mathcal{H}_i)} \leq & \frac{\mathcal{O}(1)}{\sqrt{n}}\cdot\log^{1/2}\tfrac{2}{\delta}.
    \end{align}
    Inserting these bounds into Lemma \ref{lemmaDeterministicTrainingBound}, we have with probability at least $1-\delta$ that \begin{multline*}
        \|\hat{\mathbf{b}}_i - \mathbf{b}_i^*\|_{E^n} \leq \|\mathbf{b}_{i,\lambda_i} - \mathbf{b}_i^*\|_{E^n} + \frac{\mathcal{O}(1)}{\lambda_i\sqrt{n}}\cdot\log ^{1/2 }\tfrac{2}{\delta} + \frac{\mathcal{O}(\sqrt{\log n})}{\lambda_i\sqrt{n}}\cdot\log ^{1/2}\tfrac{2}{\delta} +\frac{2\kappa}{\lambda_i} \PaTrEr.
    \end{multline*} Combining $\mathcal{O}(1) + \mathcal{O}(\sqrt{\log n}) =\mathcal{O}(\log n)$ and using $\log^{1/2}\tfrac{2}{\delta} = \mathcal{O}(1)\cdot \log \tfrac{2}{\delta}$ by equation \eqref{eq:proof_StatsThm1}, the claim \eqref{eq:thm_driftConc1} follows. Continuing, Lemma \ref{lemmaDeterministicTrainingBound} shows that \begin{multline*}
        \|\hat{b}_i - b_i^*\|_{L^2(p)}  \leq \|b_{i,\lambda_i} - b_i^*\|_{L^2(p)} + \frac{\mathcal{O}(1)}{\lambda_i\sqrt{n}}\cdot\log^{1/2}(\tfrac{2}{\delta}) + \frac{\mathcal{O}(1)}{(\lambda_i\sqrt{n})^{3/2}}\cdot\log^{3/4}(\tfrac{2}{\delta}) \\ + \frac{\mathcal{O}(\sqrt{\log n})}{\lambda_i\sqrt{n}} \log^{1/2} \tfrac{2}{\delta} + \frac{2\kappa}{\lambda_i}\PaTrEr + \frac{1}{\lambda_i^{3/2}} \left(\frac{\mathcal{O}(\sqrt{\log n})}{\sqrt{n}}\log^{1/2}\tfrac{2}{\delta} + \kappa\PaTrEr\right)\frac{\mathcal{O}(1)}{n^{1/4}}\log^{1/4}\tfrac{2}{\delta} \\
        = \|b_{i,\lambda_i} - b_i^*\|_{L^2(p)} + \frac{\mathcal{O}(\sqrt{\log n})}{\lambda_i\sqrt{n}}\log^{1/2}(\tfrac{2}{\delta}) + \frac{\mathcal{O}(\sqrt{\log n})}{(\lambda_i\sqrt{n})^{3/2}}\log^{3/4}(\tfrac{2}{\delta}) + \frac{\kappa}{\lambda_i
        }\left(2 + \frac{\mathcal{O}(1)\cdot \log^{1/4}\tfrac{2}{\delta}}{(\lambda_i \sqrt{n})^{1/2}}\right)\PaTrEr.
    \end{multline*} By assumption, $(\lambda_i\sqrt{n})^{-1/2} = \mathcal{O}(1)$. Applying equation \eqref{eq:proof_StatsThm1} to the log terms, in particular $2 + \log^{1/4} \tfrac{2}{\delta} = \mathcal{O}(1)\cdot \log \tfrac{2}{\delta}$, yields the claim \eqref{eq:thm_driftConc2}.
\end{proof}

\corollayExpectGenError*
\begin{proof}
    Let $\delta\in(0,1)$. By inserting $\delta/d$ in Theorem \ref{theoremDriftConcentration}, absorbing $d$ into the $\mathcal{O}(\cdot)$ terms via \eqref{eq:proof_StatsThm1}, and applying a union bound, we have that 
    \begin{align}\label{eq:proof_corollaryExpectGenErr0}
        \begin{aligned}
        \|\hat{b}_i - b_i^*\|_{L^2(p)}   \leq  \|b_{i,\lambda_i} \hspace{-0.1cm}- b_i^*\|_{L^2(p)} + \frac{\mathcal{O}(\log n)}{\lambda_i \sqrt{n}} \log( \tfrac{2}{\delta}) 
          +\frac{\mathcal{O}(1)}{\lambda_i} \log \tfrac{2}{\delta}\hspace{-0.2cm} \sum_{j\in\pa{i}} \|\hat{\mathbf{b}}_j - \mathbf{b}_j^*\|_{E^n},
    \end{aligned}\\\label{eq:proof_corollaryExpectGenErr1}
         \|\hat{\mathbf{b}}_i - \mathbf{b}_i^*\|_{E^n} \leq \|\mathbf{b}_{i,\lambda_i} - \mathbf{b}_i^*\|_{E^n} + \frac{\mathcal{O}(\log n)}{\lambda_i \sqrt{n}}\cdot\log(\tfrac{2}{\delta}) + \frac{2\kappa}{\lambda_i} \sum_{j\in\pa{i}} \|\hat{\mathbf{b}}_j - \mathbf{b}_j^*\|_{E^n}
    \end{align} for \emph{all} $i\in\{1,\ldots,d\}$ with probability at least $1-\delta$. We now recursively plug equation \eqref{eq:proof_corollaryExpectGenErr1} into equation \eqref{eq:proof_corollaryExpectGenErr0} and transform to something that's amenable to taking expectation. Let $(\lambda_i)_{i=1,\ldots,d}$ satisfy \eqref{eq:corollaryExpectGenError_claim1}. Via induction over $d_i\geq1$, we show that\begin{equation*}
      \frac{\log \tfrac{2}{\delta}}{\lambda_i} \sum_{j\in\pa{i}} \|\hat{\mathbf{b}}_j - \mathbf{b}_j^*\|_{E^n} \leq \frac{\mathcal{O}(\log n)}{n^{1/r_i}} \log^2\tfrac{2}{\delta} + \mathcal{O}(1)\sum_{m\in\anc{i}} (T_m^{1/3^{d_{mi}}}(\lambda_m) \vee T_m(\lambda_m))(Q_m + \log^2\tfrac{2}{\delta})
    \end{equation*} for all $i\leq d$ on this event with probability at least $1-\delta$, where $Q_m:=\tfrac{\|\mathbf{b}_{m,\lambda_m} - \mathbf{b}_m^*\|_{E^n}^2}{\|b_{m,\lambda_m} - b_m^*\|^2_{L^2(p)}}$. First consider $d_i=1$. Then, $\anc{i}=\pa{i}$, otherwise we had $d_i\geq 2$. In particular, $\pa{j}=\emptyset$ for all $j\in\anc{i}$. Further, $d_{mi}=1$ for all $m\in\anc{i}$. By \eqref{eq:proof_corollaryExpectGenErr1} we have that \begin{multline*}
        \frac{\log \tfrac{2}{\delta}}{\lambda_i} \sum_{j\in\pa{i}} \|\hat{\mathbf{b}}_j - \mathbf{b}_j^*\|_{E^n} \leq \sum_{j\in\pa{i}} \frac{\mathcal{O}(\log n)}{\lambda_i\lambda_j\sqrt{n}}\log^2(\tfrac{2}{\delta}) + \sum_{j\in\pa{i}}\frac{\|b_{j,\lambda_i} - b_j^*\|_{L^2(p)}}{\lambda_i} \log(\tfrac{2}{\delta}) \sqrt{Q_i}. 
    \end{multline*} For the second term, use $\lambda_i^{-1}\leq \|b_{j,\lambda_j} - b_j^*\|_{L^2(p)}^{-2/3}$ for all $j\in\pa{i}$ by construction, and apply $\log(\tfrac{2}{\delta})\sqrt{Q_i} \leq \log^2(\tfrac{2}{\delta}) + Q_i$. Regarding the first term, consider that \begin{equation*}
        \tfrac{1}{\lambda_i\lambda_j\sqrt{n}} \leq n^{2/r_i}\cdot n^{2/r_j}\cdot n^{-1/2} = n^{2/18}\cdot n^{2/6}\cdot n^{-1/2} = n^{-1/18} = n^{-1/r_i},
    \end{equation*} and absorb $|\pa{i}|\cdot\mathcal{O}(\log n)=\mathcal{O}(\log n)$ to see the induction claim. Let $i$ be a node such that the induction claim has been shown for $d_i-1$. Note that $d_j\leq d_i - 1$ for all $j\in\anc{i}$, otherwise there would exist a path of length $d_i+1$ with $i$ as endpoint via node $j$. By \eqref{eq:proof_corollaryExpectGenErr1}, \begin{multline}\label{eq:proof_corollaryExpectGenErr2}
        \frac{\log \tfrac{2}{\delta}}{\lambda_i} \sum_{j\in\pa{i}} \|\hat{\mathbf{b}}_j - \mathbf{b}_j^*\|_{E^n} \leq \sum_{j\in\pa{i}} \frac{\mathcal{O}(\log n)}{\lambda_i\lambda_j\sqrt{n}}\log^2(\tfrac{2}{\delta}) + \sum_{j\in\pa{i}}\frac{\|b_{j,\lambda_j} - b_j^*\|_{L^2(p)}}{\lambda_i} \log(\tfrac{2}{\delta}) \sqrt{Q_i} \\+ \frac{1}{\lambda_i}\sum_{j\in\pa{i}} \frac{2\kappa \log \tfrac{2}{\delta}}{\lambda_j} \sum_{m\in\pa{j}} \|\hat{\mathbf{b}}_m - \mathbf{b}_m^*\|_{E^n}.
    \end{multline}By $\tfrac{1}{r_j} = \tfrac{1}{6\cdot 3^{d_j}}\geq \tfrac{1}{6\cdot 3^{d_i -1}} = \tfrac{3}{r_i}$, we have \begin{equation}\label{eq:proof_corollaryExpectGenErr3}
        \tfrac{1}{\lambda_i n^{1/r_j}}\leq n^{2/r_i}\cdot n^{-3/r_i} = n^{-1/r_i}.
    \end{equation}Since $(\lambda_j\sqrt{n})^{-1} \leq n^{2/r_j}\cdot n^{-1/2}\leq n^{-1/r_j}$, % this bound is not tight; the critical case is when the above display is applied to the induction case
    the first term on the right hand side of \eqref{eq:proof_corollaryExpectGenErr2} is bounded by $\tfrac{\mathcal{O}(\log n)}{n^{1/r_i}}\log^2\tfrac{2}{\delta}$ using \eqref{eq:proof_corollaryExpectGenErr3}. Treat the second term like in the induction start, using $\lambda_i^{-1}\leq T_j(\lambda_j)^{-2/3}$ for $j\in\pa{i}$. Recalling the induction hypothesis, we find that the third term on the right hand side of \eqref{eq:proof_corollaryExpectGenErr2} is bounded by \begin{multline*}
        \sum_{j\in\pa{i}} \frac{\mathcal{O}(\log n)}{\lambda_i n^{1/r_j}} \log^2\tfrac{2}{\delta} + \mathcal{O}(1)\sum_{j\in\pa{i}}\sum_{m\in\anc{j}} \frac{(Q_m + \log^2\tfrac{2}{\delta})}{\lambda_i}(T_m^{1/3^{d_{mj}}}(\lambda_m) \vee T_m(\lambda_m))\\
        \overset{\eqref{eq:proof_corollaryExpectGenErr3}}{\leq} \frac{\mathcal{O}(\log n)}{n^{1/r_i}}\log^2\tfrac{2}{\delta} + \mathcal{O}(1) \sum_{j\in\pa{i}}\sum_{m\in\anc{j}}\frac{T_m^{1/3^{d_{mj}}}(\lambda_m) \vee T_m(\lambda_m)}{T_m^{2/3^{d_{mi}}}(\lambda_m)}(Q_m + \log^2\tfrac{2}{\delta}). 
    \end{multline*} If $T_m(\lambda_m)\leq 1$, use $d_{mi} \geq d_{mj}+1$ for $j\in\pa{i}$, and therefore $3^{-d_{mj}} - 2\cdot 3^{-d_{mi}} \geq \tfrac{1}{3}3^{-d_{mj}} \geq 3^{- d_{mi}}$. Otherwise, bound $T_m^{-1/2^{d_{mi}}}(\lambda_m)\leq 1$. Absorbing $|\pa{j}|$ into $\mathcal{O}(1)$, we have shown the induction claim. Now, consider \eqref{eq:proof_corollaryExpectGenErr0}, and for $d_i\geq 1$ plug in the induction claim to obtain: \begin{multline*}
        \|\hat{b}_i - b_i^*\|_{L^2(p)}  \leq \|b_{i,\lambda_i} - b_i^*\|_{L^2(p)} + \frac{\mathcal{O}(\log n)}{\lambda_i \sqrt{n}} \log\left( \tfrac{2}{\delta}\right)  +1_{d_i\geq 1}\cdot\frac{\mathcal{O}(\log n)}{n^{1/r_i}} \log^2(\tfrac{2}{\delta}) \\
        + \mathcal{O}(1)\sum_{m\in\anc{i}} \|b_{m,\lambda_m} - b_m^*\|_{L^2(p)}^{1/3^{d_{mi}}}\vee  \|b_{m,\lambda_m} - b_m^*\|_{L^2(p)}(Q_m + \log^2(\tfrac{2}{\delta})).
    \end{multline*} for all $i\leq d$ with probability at least $1-\delta$, where $1_{d_i\geq 1}$ is $1$ if $d_{i}\geq 1$ and zero otherwise. Now, insert $(\lambda_i\sqrt{n})^{-1} \leq n^{-1/r_i}$. Then, integrate the bound as demonstrated in Lemma \ref{lemmaExpectHighProb} and use $\ex{Q_m}=1$. Using $3^{d_{ii}}=3^0=1$ yields the theorem claim.

    If $\mathcal{H}_i$ is dense in $L^2(\mathbb{R}^{|\Pa{i}|}, p_{\Pa{i}})$ for all $i\leq d$, it follows that $\inf_{b\in\mathcal{H}_i} \|b - b_i^*\|_{L^2(p)}^2 =0$. This implies $\lim_{\lambda\to0} T_i(\lambda)=0$ \citep[Sct. 5.4]{Steinwart2008SVM}. Setting $\lambda_i$ to the right hand side of \eqref{eq:corollaryExpectGenError_claim1} yields $\lim_{n\to\infty}\lambda_i(n)=0$ and  $\lim_{n\to\infty} \ex{\|\hat{b}_i - b_i^*\|_{L^2(p)}} = 0$ by \eqref{eq:corollaryExpectGenError_claim2}.

    Let $T_j(\lambda) = \|b_{j,\lambda} - b_j^*\|_{L^2(p)} = \mathcal{O}(\sqrt{\lambda}\cdot \log^{\alpha_j}\tfrac{1}{\lambda})$ as $\lambda\to0$ for all $j\leq d$ and let $0<\varepsilon<\tfrac{1}{2d}$. Set $\lambda_j(n):= n^{-2(1-\varepsilon_j)/r_j}$ for all $j\leq d$. Let $i\leq d$. For any $j\in\anc{i}$ we have by $\lambda_j\to 0$ that \begin{equation*}
        \|b_{j,\lambda_j} - b_j^*\|_{L^2(p)}^{1/3^{d_{ji}}} = \mathcal{O}(n^{-(1-\varepsilon_j)/(r_j 3^{d_{ji}})} \cdot\log^{\alpha_j/3^{d_{ji}}}(n^{2(1-\varepsilon_j)/r_j})) = \mathcal{O}(n^{-(1-\varepsilon_i)/r_i}),
    \end{equation*} since $r_j 3^{d_{ji}} \leq r_i$ and $n^{(\varepsilon_j - \varepsilon_i)/r_i}$ decays faster to zero than the log-term as $n\to\infty$. In particular, $(\lambda_i)_{i\leq d}$ satisfies \eqref{eq:corollaryExpectGenError_claim1} for $n$ large enough. The leading term in equation \eqref{eq:corollaryExpectGenError_claim2} is $T_i(\lambda_i)= \mathcal{O}(n^{-(1-\varepsilon_i)/r_i} \cdot \log^{\alpha_i} n^{2(1-\varepsilon_i)/r_i})=\mathcal{O}(n^{-(1-2\varepsilon_i)/r_i})$, which proves the claim. 
\end{proof}

\subsection{Theory discussion details}\label{subsec:app_theory_discuss}
First, we demonstrate that the stationary distribution and parent set alone do not identify $b_3^*$ in a three variable example. Consider the following matrices \begin{align*}
    \Sigma &:= \begin{pmatrix}
        0.5& 0.125& 0.1563\\ 0.125& 0.5625& 0.21094\\ 0.1563& 0.21094&
   0.683594
    \end{pmatrix} & B_1 &:= \begin{pmatrix}
        -1 & 0 & 0 \\ 0.5 & -1 & 0 \\ 0.5 & 0.5 & -1
    \end{pmatrix} & B_2 := \begin{pmatrix}
        -1.1111& 0.44444& 0\\ 0& -0.88889& 0\\ 0.3125& 0.63889& -1
    \end{pmatrix}.
\end{align*} Then, both drift functions $B_1\mathbf{x}$ and $B_2\mathbf{x}$ induce a Gaussian with mean zero and covariance $\Sigma$ as stationary distribution when the constant identity matrix is used as diffusivity, as seen by the continuous Lyapunov equation. Note that the third drift component, i.e., the third row, differs even though node $3$ has the same parent set in both drifts functions. In other words, the stationary distribution and the parent set alone do \emph{not} identify the third drift component here. It also depends on the parent drifts, which differ in both cases since $1$ is a parent of $2$ in $B_1$ and $2$ a parent of $1$ in $B_2$.  

For the second part of the discussion, consider the setup from example \ref{runExample}. We learn the drift $b_2^*(x_1, x_2)$ using a  shift-invariant product kernel $k_2((x_1, x_2), (y_1, y_2)) = k(x_1 - y_1)\cdot k(x_2 - y_2)$ for some univariate function $k$. We assume $k(-x)=k(x)$ and $k'(-x)=-k'(x)$, and that $k$ is integrable. We use $\mathcal{K}_{2,\mathbf{y}}(\mathbf{x}):= \int_{-\infty}^{x_2}k_2((x_1,  z),\mathbf{y})\dd{z}$. Under sufficient tail decay of $k$, \begin{multline*}
    \partial_1 \mathcal{K}_{2,\mathbf{y}}(\mathbf{x}) = \int_{-\infty}^{x_2}\partial_{x_1}k_2((x_1,  z),\mathbf{y})\dd{z} = k'(x_1 - y_1) \int_{-\infty}^{x_2}k(z - y_2)\dd{z} = \\-k'(y_1 - x_1)\int_{-\infty}^{x_2-y_2}k(z)\dd{z} = -\partial_{y_1}k(x_1 - y_1)\int_{y_2}^{\infty}k(x_2 - z)\dd{z}.
\end{multline*} Let $\varepsilon(x_1):= b_1^*(x_1) - b(x_1)$ for some proposal drift $b$. Then,  
\begin{multline*}
    \zeta_2(\mathbf{y}) - \ex{\partial_1 \mathcal{K}_{2,\mathbf{y}} \cdot b - \tfrac{1}{2}\Delta \mathcal{K}_{2,\mathbf{y}}} = \int_{\mathbb{R}^2} \partial_1\mathcal{K}_{2,\mathbf{y}}(\mathbf{x}) \varepsilon(x_1)\dd{p(\mathbf{x})} = \\
    \int_{y_2}^{\infty} -\partial_{y_1} \int_{\mathbb{R}^2} k_2(\mathbf{x},(y_1, z)) \varepsilon(x_1)\dd{p}(\mathbf{x})\dd{z} =  -\int_{y_2}^\infty \partial_{y_1} \, (J_2^* \varepsilon)(y_1,z)\,\dd{z}.
\end{multline*} In other words, using the incorrect root drift $b$ for learning $\zeta_2$ leads to a bias, which in the infinite sample limit can be expressed in terms of $J_2^* \varepsilon$. The bias in $\zeta_2$ induces a bias on the second drift component via  $L^2(p)\ni b_{2,\lambda_2} = J_2 (J_2^*J_2 + \lambda_2 I)^{-1}\zeta_2$. Therefore, the final bias induced by $b\neq b_1^*$ heuristically is: \begin{equation*}
    J_2 (J_2^*J_2 + \lambda_2 I)^{-1} \left(-\int_{y_2}^\infty \partial_{y_1} \, (J_2^* \varepsilon)(y_1,z)\,\dd{z}\right).
\end{equation*} Since $\|J_2 (J_2^*J_2 + \lambda I)^{-1}J_2^*\|_{\mathcal{L}(L^2(p))}\leq 1$, the bias could be up to order $\|\varepsilon\|_{L^2(p)} = \|b_1^* - b\|_{L^2(p)}$. This would improve the current bound $\lambda^{-1} \cdot\|b_1^* - b\|_{L^2(p)}$, which is due to $\|(J_2^*J_2 + \lambda I)^{-1}\|_{\mathcal{L}(\mathcal{H}_2,\mathcal{H}_2)} \leq \lambda^{-1}$. However, it is difficult to bound the effect of $\int_{y_2}^\infty \partial_{y_1} \, \cdot\,\dd{z}$ on $J_2^* \varepsilon$, in particular since $\int_{y_2}^\infty \cdot \dd{z}$ does not map into $\mathcal{H}$ due to boundary value mismatch, making the above bias heuristic actually ill-defined. Future research could derive alternative estimators of $\zeta_2$ beyond sample averages of $\mathcal{K}_{2,\mathbf{y}}(\mathbf{x}_j)$, which may be better behaved analytically.

\subsection{Cross-validation}\label{subsec:appendix_cv}
The following describes how the cross-validation algorithm approximates $\langle\hat{b}_{i,\lambda;\mathcal{I}},\hat{\zeta}_{i;\mathcal{J}}\rangle_{\mathcal{H}_i}$. Let $\mathbf{K}_i=(k_i(\mathbf{x}_j, \mathbf{x}_l))_{j,l\in\mathcal{J}}$ be the kernel matrix on the test data and let $\gamma$ be an eigenvalue of $\tfrac{1}{|\mathcal{J}|}\mathbf{K}$ with eigenvector $\mathbf{u}$. Then, $\gamma$ is an eigenvalue of $\hat{L}_{i,\mathcal{J}}$ with eigenfunction $v(\mathbf{y}):= \tfrac{1}{\sqrt{|\mathcal{J}|\cdot\gamma}}\sum_{j\in\mathcal{J}} u_j  k_i(\mathbf{x}_j, \mathbf{y})$ \citep{Rosasco2010}. Let $\gamma_\text{min}>0$ and $\gamma_1,\ldots,\gamma_s$ be all eigenvalues of $\tfrac{1}{|\mathcal{J}|}\mathbf{K}$ with magnitude greater than $\gamma_\text{min}$. The projection onto the corresponding eigenfunctions $(v_j)_{j=1,\ldots,s}$ can be written as \begin{multline*}
    \argmin_{\bs{\alpha}\in\mathbb{R}^s} \|\sum_{j=1}^s \alpha_jv_j - \hat{\zeta}_{i;\mathcal{J}}\|_\mathcal{H}^2 = \argmin_{\bs{\alpha}\in\mathbb{R}^s} \sum_{jklm} \frac{\alpha_j \alpha_k}{|\mathcal{J}| \sqrt{\gamma_j\gamma_k}} u_{j,l}u_{k,m}k_i(x_l, x_m)\\
    -2 \sum_{j=1}^s \frac{\alpha_j}{\sqrt{|\mathcal{J}| \gamma_j}} \sum_{k\in\mathcal{J}} u_{j,k} \hat{\zeta}_{i;\mathcal{J}}(\mathbf{x}_k) = \argmin_{\bs{\alpha}\in\mathbb{R}^s}\frac{1}{|\mathcal{J}|} (\mathbf{U}\bs{\Gamma}\bs{\alpha})^T\mathbf{K} \mathbf{U} \bs{\Gamma}\bs{\alpha} - \frac{2}{\sqrt{|\mathcal{J}|}} (\mathbf{U}\bs{\Gamma}\bs{\alpha})^T \mathbf{h},
\end{multline*} where $\mathbf{U}\in\mathbb{R}^{|\mathcal{J}|\times s}$ has the $u$-vectors as columns, where we defined $\bs{\Gamma}:= \diag(\sqrt{\bs{\gamma}}^{-1})\in\mathbb{R}^{s\times s}$ and $h_k := \hat{\zeta}_{i;\mathcal{J}}(x_k)$, and where we ignored the additive constant $\|\hat{\zeta}_{i;\mathcal{J}}\|_\mathcal{H}^2$ in the optimization. Using $\tfrac{1}{|\mathcal{J}|}\bs{\Gamma}\mathbf{U}^T\mathbf{K}\mathbf{U}\bs{\Gamma} =\mathbf{I}$, the minimizer is given by 
\begin{equation*}
    \bs{\alpha}^* = \frac{1}{\sqrt{|\mathcal{J}|}} \bs{\Gamma}\mathbf{U}^T\mathbf{h}.
\end{equation*}
For an RKHS function $f$, we therefore approximate $\langle f, \hat{\zeta}_{i;\mathcal{J}}\rangle_{\mathcal{H}_i}$ by $\tfrac{1}{\sqrt{|\mathcal{J}|}}(\mathbf{U}\bs{\Gamma}\bs{\alpha}^*)^T \mathbf{f}$, where $f_j := f(\mathbf{x}_j)$ for $j\in\mathcal{J}$. All in all, our crossvalidation method optimizes \begin{equation*}
    \frac{1}{|\mathcal{J}|} \|\mathbf{b}_i\|_2^2 + \frac{2}{\sqrt{|\mathcal{J}|}} (\mathbf{U}\bs{\Gamma}\bs{\alpha}^*)^T \mathbf{b}_i,
\end{equation*} where $\mathbf{b}_i=(\hat{b}_{i,\lambda;\mathcal{I}}(\mathbf{x}_j))_{j\in\mathcal{J}}$.

\section{Simulation and extension details}
\subsection{KDS for DAG-structured drifts}\label{app:kds}
The kernel deviation from stationarity (KDS) by \citet{Lorch2024StationaryDiffusions} is an alternative approach to fit nonparametric drift functions to samples from the stationary distribution of an SDE. We extend the accompanying Python package to DAG-structured drifts by simply adding input masks in the existing Neural network class MLPSDE. Note that the current KDS package version assumes $b_i(\mathbf{x}) = -x_i + f_{\bs{\Theta}}(\mathbf{x}_{-i})$ for a two-layer neural network $f$ with trainable weights $\bs{\Theta}$. We tried activating some low-level functions in the package which generalize to arbitrary self-regulation, but obtained NaNs in the KDS training, even when using the MLPSDE class. Therefore, we train on data from a Gaussian distribution with drift $b_1(x_1) = - x_1$, $b_2(x_1, x_2) = x_1 - x_2$, and $b_3(x_1, x_2, x_3) = -\tfrac{1}{2}x_1 + \tfrac{1}{2}x_2 - x_3$ matching the package assumptions. 

We train on $n=1000$ samples using the default package hyperparameters and record the generalization error for varying numbers of gradient steps. After $40000$ steps, the average generalization error noticeably increased and we did not observe double descent even when training $10$ times longer. The optimal recorded generalization error averaged $0.07\pm 0.006$ on $N=50$ experiments. This confirms that KDS approximates the true drift when oracle-stopped. We are not aware of package functionality for adaptive early stopping; reducing the learning rate significantly could lead to convergence at the expense of longer training.

For comparison, our Tikhonov regularized estimator achieves an average generalization error of $0.2\pm 0.01$ on this problem while adaptively choosing the regularization strength. Additionally it learns the self-regulation $-x_i$; in particular it must learn $b_1$ and also fit a three-variable function $b_3$ instead of a two-variable function like KDS.

\subsection{Irreversible diffusion models}\label{subsec:diffusion_sim_details}
We consider a $5$-dimensional Gaussian density $p$ with mean zero and score function $\mathbf{s}(\mathbf{x})$ rounded to three digits shown below. Also consider the alternative drift function $\mathbf{b}(\mathbf{x})$ defined by \begin{align*}
    \mathbf{b}(\mathbf{x}) & = \begin{pmatrix}
        -1 &  0 &  0 &  0 &  0 \\
        0.5& -1 &  0 &  0 &  0 \\
       -0.1&  0.3& -1 &  0 &  0 \\
        0.4& -0.6&  1 & -1 &  0 \\
        0.7& -0.1& -0.6&  0.1& -1
    \end{pmatrix}\mathbf{x};&\mathbf{s}(\mathbf{x}) & = \begin{pmatrix}
        -1.185&  0.275& -0.032  &  0.147&  0.282 \\
        0.275& -1.016&  0.241& -0.201&  0.015\\
       -0.032  &  0.241& -1.210&  0.348& -0.243\\
        0.147& -0.201&  0.348& -0.756& -0.007\\
        0.282 &  0.015& -0.243& -0.007& -0.832
    \end{pmatrix} \mathbf{x}.
\end{align*} Both matrices are linked via the Lyapunov equation such that for all $\omega\in\mathbb{R}$ the SDE \begin{equation*}
    \dd{\mathbf{x}}(t) = (\mathbf{s} - \omega\cdot(\mathbf{s} - \mathbf{b}))(\mathbf{x}(t))\dd{t} + \sqrt{2}\dd{\mathbf{w}(t)}
\end{equation*} has the density $p$ as its stationary density. One also could obtain $\mathbf{b}$ via Theorem \ref{theoremDriftExistenceCompleteDAG} from $p$ directly, however the Lyapunov method is simpler to compute and must yield the same result by our identifiability theory. 

For $\omega\in \{0, 1, 5, 10, -20\}$ we simulate $i=1,\ldots,100$ trajectories $\mathbf{x}_{i}^{(\omega)}(t)$ from $t=0$ to $T=2$ using the Euler-Maruyama method with $5000$ steps. The initialization points are drawn from a Gaussian distribution with mean vector $(5,5,5,5,5)$ and identity covariance and the initialization is shared across different $\omega$. At each time $t$ we compute the energy distance between $\mathbf{x}_{1}^{(\omega)}(t), \ldots, \mathbf{x}_{100}^{(\omega)}(t)$ and an independent sample of size $1000$ from $p$. We then plot this distance depending on $t$ and $\omega$. 

\theoremDriftExistenceCompleteDAG*
\begin{proof}
    For reference, note the following identities for $j\geq 1$: \begin{align*}
        f_{j+1}(\mathbf{x}) &= \frac{p_{1:(j+1)}(\mathbf{x}_{1:(j+1)})}{p_{1:{j}}(\mathbf{x}_{1:j})} & F_{j+1}(\mathbf{x}) &= \int_{-\infty}^{x_{j+1}} f_{j+1}(\mathbf{x}_{1:j}, z)\,\dd{z}.
    \end{align*}
    We show via induction over $j$ that $b_j$ only depends on $\mathbf{x}_{1:j}$ (thus proving drift structuring according to the complete DAG) and that \begin{equation*}
        \frac{\sigma^2}{2} \Delta_{1:j}\, p_{1:j} = \nabla_{1:j}\cdot(\mathbf{b}_{1:j}\, p_{1:j}),
    \end{equation*} which proves the theorem claim when taking $j=d$.

    The case $j=1$ holds by definition. We next show that if the induction claim holds for $j$, it also holds for $j+1$. From the above identities, we see that $f_{j+1}$ and $F_{j+1}$ only depend on $\mathbf{x}_{1:(j+1)}$. Since $\mathbf{b}_{1:j}$ only depends on $\mathbf{x}_{1:j}$ by induction hypothesis, it follows that $b_{j+1}$ only depends on $\mathbf{x}_{1:(j+1)}$, as desired. To see that the PDE holds, multiply the constructed $b_{j+1}$ with $p_{1:(j+1)}$, obtaining \begin{multline*}
        p_{1:(j+1)}b_{j+1} = \frac{\sigma^2}{2}p_{1:j}\,\partial_{j+1}f_{j+1} + \frac{\sigma^2}{2}p_{1:j}\Delta_{1:j}F_{j+1} + \\
        \frac{\sigma^2}{2}\cdot 2\cdot(\nabla_{1:j}\,p_{1:j})^T\cdot\nabla_{1:j}F_{j+1} -(p_{1:j}\mathbf{b}_{1:j})^T\cdot\nabla_{1:j}F_{j+1}.
    \end{multline*}The first two terms can be simplified to $\tfrac{\sigma^2}{2}p_{1:j}\Delta_{1:(j+1)}F_{j+1}$ by definition of $F_{j+1}$. We now add zero to the right hand side, more specifically \begin{equation*}
        F_{j+1}\cdot\left(\tfrac{\sigma^2}{2}\Delta_{1:j}\, p_{1:j} - \nabla_{1:j}\cdot(p_{1:j}\mathbf{b}_{1:j})\right),
    \end{equation*}which is zero by the induction hypothesis. All terms involving $\tfrac{\sigma^2}{2}$ read \begin{equation*}
        \frac{\sigma^2}{2}p_{1:j}\Delta_{1:(j+1)}F_{j+1} + \frac{\sigma^2}{2}\cdot 2\cdot(\nabla_{1:j}\,p_{1:j})^T\cdot\nabla_{1:j}F_{j+1} + \frac{\sigma^2}{2}F_{j+1}\cdot\Delta_{1:j}\, p_{1:j}.
    \end{equation*}Since $\partial_{j+1}p_{1:j}=0$, we have $\Delta_{1:(j+1)}p_{1:j}=\Delta_{1:j}p_{1:j}$ and also \begin{equation*}
        (\nabla_{1:j}\,p_{1:j})^T\cdot\nabla_{1:j}F_{j+1} = (\nabla_{1:(j+1)}p_{1:j})^T\cdot\nabla_{1:(j+1)}F_{j+1}.
    \end{equation*} Recalling the Laplacian product rule $\Delta (fg) = f\Delta g + 2 \nabla f\cdot\nabla g + g\Delta f$, we can simplify all terms involving $\tfrac{\sigma^2}{2}$ to \begin{equation*}
        \frac{\sigma^2}{2}\Delta_{1:(j+1)}(p_{1:j}F_{j+1}). 
    \end{equation*} The remaining terms not involving $\tfrac{\sigma^2}{2}$ read \begin{equation*}
        -(p_{1:j}\mathbf{b}_{1:j})^T\cdot\nabla_{1:j}F_{j+1} - F_{j+1}\cdot\nabla_{1:j}\cdot(p_{1:j}\mathbf{b}_{1:j}) = -\nabla_{1:j}\cdot(\mathbf{b}_{1:j}\,p_{1:j}\,F_{j+1})
    \end{equation*} by the divergence product rule. Collecting the results, we have shown \begin{equation*}
        p_{1:(j+1)}b_{j+1} = \frac{\sigma^2}{2}\Delta_{1:(j+1)}(p_{1:j}F_{j+1}) -\nabla_{1:j}\cdot(\mathbf{b}_{1:j}\,p_{1:j}\,F_{j+1}).
    \end{equation*} Taking $\partial_{j+1}$ on both sides, changing the derivative order by Schwarz's theorem, % note F_{j+1} is three times continuously differentiable since it's an integral
    and using $\partial_{j+1}F_{j+1}=f_{j+1}$ as well as that neither $p_{1:j}$ nor $\mathbf{b}_{1:j}$ depend on $x_{j+1}$, we obtain \begin{equation*}
        \partial_{j+1}(p_{1:(j+1)}b_{j+1}) = \frac{\sigma^2}{2}\Delta_{1:(j+1)}(p_{1:j}\,f_{j+1}) -\nabla_{1:j}\cdot(\mathbf{b}_{1:j}\,p_{1:j}\,f_{j+1}).
    \end{equation*}Simplifying $p_{1:j}\,f_{j+1} = p_{1:(j+1)}$ and rearranging yields the induction claim.  
\end{proof}

\section{Technical Lemmas}
\begin{restatable}[Standard Mollifier]{definition}{defStandardMollifier}\label{defStandardMollifier}
    Define $\eta\colon\mathbb{R}^d\to\mathbb{R}$ by \begin{equation*}
        \eta(\mathbf{x}):= \begin{cases}
            C\exp\left(\frac{1}{\|\mathbf{x}\|_2^2 - 1}\right) & \|\mathbf{x}\|_2 \leq 1 \\
            0 & \|\mathbf{x}\|_2 > 1
        \end{cases}
    \end{equation*} with $C$ such that $\int_{\mathbb{R}^d}\eta(\mathbf{x})\dd{\mathbf{x}} = 1$. Finally, let $\eta_\varepsilon(\mathbf{x}) := \tfrac{1}{\varepsilon^d}\eta(\frac{\mathbf{x}}{\varepsilon})$ for $\varepsilon>0$ and note $\int_{\mathbb{R}^d}\eta_\varepsilon(\mathbf{x})\dd{\mathbf{x}} = 1$.
\end{restatable}

\begin{restatable}{lemma}{lemmaZeroWeakDeriv}\label{lemmaZeroWeakDeriv}
    Let $u\in L^1(\mathbb{R}^k)$, $k\geq 1$, with $\int_{\mathbb{R}^k}u\, \partial_l\varphi\,\dd{\mathbf{x}} = 0$ for all $\varphi\in C^\infty_0(\mathbb{R}^k)$ and one index $l\in\{1,\ldots,k\}$. Then $u=0$ almost everywhere on $\mathbb{R}^k$.
\end{restatable}
\begin{proof}
    The assumptions imply that $u$ has a weak derivative in $x_l$ direction, given by zero almost everywhere. We write $\partial_l u = 0$.

    First, we study the case $k=1$. Then, $u\in W^{1,1}(\mathbb{R})$ with $u'=0$. In this case, $u$ is a.e. equal to an absolutely continuous function with (ordinary) derivative $u'=0$ on every finite interval $(a,b)$; see problem 5.10.5 in \citep{EvansPDEbook}. It follows that $u$ is zero globally by $u\in L^1(\mathbb{R})$.
    
    Next, consider $k\geq 2$. By reordering input arguments, we can assume $l=k$ without loss of generality. For $\varepsilon>0$, set $u_\varepsilon := u *\eta_\varepsilon$, the convolution with the standard mollifier $\eta_\varepsilon$ from Definition \ref{defStandardMollifier}. Then, by the proof of theorem 1 in section 5.3.1 of \citep{EvansPDEbook},\begin{itemize}
        \item $u_\varepsilon\in C^\infty(\mathbb{R}^k)$ with
        \item $\partial_k u_\varepsilon = \eta_\varepsilon * \partial_k u = \eta_\varepsilon * 0 = 0$, and
        \item $u_\varepsilon \to u$ in $L^1(V)$ as $\varepsilon\to 0$ for every $V\subset\subset \mathbb{R}^k$.
    \end{itemize}

    The first two bullet points imply that $u_\varepsilon(\mathbf{x}) = \tilde{u}_\varepsilon(\mathbf{x}_{-k})$ for functions $\tilde{u}_\varepsilon\in C^\infty(\mathbb{R}^{k-1})$. Let $B\subset\mathbb{R}^{k-1}$ be an Euclidean ball of arbitrary (finite) radius around the origin. For any $\varepsilon>0$ and $n\in\mathbb{N}$, by the triangle inequality \begin{equation*}
        \|u_\varepsilon\|_{L^1(B\times[-n,n])} \leq \|u_\varepsilon - u\|_{L^1(B\times[-n,n])} + \|u\|_{L^1(\mathbb{R}^k)}.
    \end{equation*} Construct a sequence $\varepsilon(n)$, such that $\varepsilon(n)\leq\tfrac{1}{n}$ and $\|u_\varepsilon - u\|_{L^1(B\times[-n,n])}\leq 1$ for all $\varepsilon \leq \varepsilon(n)$. Noting that $\|u_\varepsilon\|_{L^1(B\times[-n,n])} = 2n\|\tilde{u}_\varepsilon\|_{L^1(B)}$, this implies \begin{equation*}
        2n\|\tilde{u}_{\varepsilon(n)}\|_{L^1(B)} \leq 1 + \|u\|_{L^1(\mathbb{R}^k)}\iff \|\tilde{u}_{\varepsilon(n)}\|_{L^1(B)} \leq \frac{1 + \|u\|_{L^1(\mathbb{R}^k)}}{2n}.
    \end{equation*} By assumption, $\|u\|_{L^1(\mathbb{R}^k)} < \infty$, so the right hand side converges to zero and consequently $\tilde{u}_{\varepsilon(n)}\to 0$ in $L^1(B)$. For arbitrary $a\in\mathbb{R}$, it holds that \begin{equation*}
        \|u_{\varepsilon(n)}\|_{L^1(B\times[a,a+1])} = \|\tilde{u}_{\varepsilon(n)}\|_{L^1(B)}\to 0,
    \end{equation*} implying $u_{\varepsilon(n)}\to 0$ in $L^1(B\times[a,a+1])$. Recalling that $\varepsilon(n)\to 0$ by construction, the third bullet point implies $u=0$ almost everywhere on $B\times[a,a+1]$. Covering $\mathbb{R}^k$ by varying $B$ and $a$ yields $u=0$ almost everywhere on $\mathbb{R}^k$.
\end{proof}

\begin{restatable}{lemma}{lemmaDenseInInftySobolev}\label{lemmaDenseInInftySobolev}
    When the Lebesgue density $p$ is bounded, then for any $\varphi\in W^{2,\infty}(\mathbb{R}^k)$ there exists $(\varphi_n)_{n\in\mathbb{N}}\subset C^\infty_0(\mathbb{R}^k)$ with $\|\varphi_n - \varphi\|_{W^{2,2}(\mathbb{R}^k, p)}\to 0$ as $n\to\infty$.
\end{restatable}
\begin{proof}
    Let $n\in\mathbb{N}$, set $B_r:= \{\mathbf{x}\in\mathbb{R}^k \mid \|\mathbf{x}\|_2 < r\}$ for $r\geq 0$ and let $\varphi\in W^{2,\infty}(\mathbb{R}^k)$. It holds that $\varphi\in W^{2,\infty}(B_{n+1})\subset W^{2,2}(B_{n+1})$. %Evans 5.2.3 Theorem 1 iii, and the subset is due to $L^\infty(B_{n+1})\subset L^1(B_{n+1})$
    By theorem 1 in Section 5.3.1 of \citep{EvansPDEbook}, there is $\varepsilon(n)\leq 0.5$ such that $\|(\varphi * \eta_{\varepsilon(n)}) - \varphi\|_{W^{2,2}(B_n)} \leq \tfrac{1}{n}$, where $*$ denotes convolution and $\eta_\varepsilon$ is the standard mollifier from Definition \ref{defStandardMollifier}. % Convergence in W^{k,p}_loc(U) in Evans is defined as convergence in W^{k,p}(V) for all V compactly embedded in U
    Further, for any multi-index $\alpha$ with $|\alpha|\leq 2$ we have $D^\alpha (\varphi * \eta_{\varepsilon(n)}) = \eta_{\varepsilon(n)} * (D^\alpha \varphi)$ on $B_{n+0.5}$. % See the proof of the referenced theorem in Evans.

    Define $\varphi_n := \gamma_n \cdot (\varphi * \eta_{\varepsilon(n)})$ with $\gamma_n$ from Lemma \ref{lemmaOneApprox}. Recall $\gamma_n\in C^\infty_0(\mathbb{R}^k)$ with support in $B_{n+0.25}$. Since $\varphi * \eta_{\varepsilon(n)} \in C^\infty(B_{n+0.5})$, it follows $\varphi_n\in C^\infty_0(\mathbb{R}^k)$.  Further, for any multi-index $\alpha$ with $|\alpha|\leq 2$, by Leibnitz' formula %see theorem 1 in 5.2.3 (iv) of Evans
    \begin{equation*}
        D^\alpha \varphi_n = \sum_{\beta\leq \alpha} \begin{pmatrix}
            \alpha \\ \beta
        \end{pmatrix} D^\beta \gamma_n \cdot (\eta_{\varepsilon(n)} * D^{\alpha - \beta}\varphi).
    \end{equation*} First, notice $|(\eta_{\varepsilon(n)} * D^{\alpha - \beta}\varphi)(\mathbf{x})| \leq \|D^{\alpha - \beta}\varphi\|_\infty \cdot \int_{\mathbb{R}^d} \eta_{\varepsilon(n)}(\mathbf{y})\dd{\mathbf{y}} = \|D^{\alpha - \beta}\varphi\|_\infty$ for all $\mathbf{x}\in\mathbb{R}^k$. Second, $\|D^\beta\gamma_n\|_{L^\infty(\mathbb{R}^k)}\leq \sup_n \|\gamma_n\|_{W^{2,\infty}(\mathbb{R}^k)} < \infty$ by Lemma \ref{lemmaOneApprox}.
    Combining these facts yields $ \sup_n \|\varphi_n\|_{W^{2,\infty}(\mathbb{R}^d)} < \infty $. Thus,\begin{multline*}
        \|\varphi_n - \varphi\|^2_{W^{2,2}(\mathbb{R}^d, p)} = \|\varphi_n - \varphi\|^2_{W^{2,2}(B_n,\, p)} + \|\varphi_n - \varphi\|^2_{W^{2,2}(\mathbb{R}^d\setminus B_n,\, p)} \leq\\
        \|\varphi_n - \varphi\|^2_{W^{2,2}(B_n)} \|p\|_{L^\infty(\mathbb{R}^k)} + 2\left(\|\varphi\|^2_{W^{2,\infty}(\mathbb{R}^d)} + \sup_n \|\varphi_n\|^2_{W^{2,\infty}(\mathbb{R}^d)}\right)\int_{\mathbb{R}^d\setminus B_n} p(\mathbf{x})\dd{\mathbf{x}},
    \end{multline*} which converges to $0$ when $n\to\infty$ as desired.

\end{proof}

\begin{restatable}{lemma}{lemmaOneApprox}\label{lemmaOneApprox}
    There is a sequence $(\gamma_n)_{n\in\mathbb{N}}\subset C^\infty_0(\mathbb{R}^k)$ satisfying $\gamma_n(\mathbf{x})= 1$ for $\|\mathbf{x}\|_2\leq n$, $\gamma_n(\mathbf{x})=0$ for $\|\mathbf{x}\|_2\geq n+\tfrac{1}{4}$, and
        $\sup_{n\in\mathbb{N}} \|\gamma_n\|_{W^{2,\infty}(\mathbb{R}^k)} < \infty$.
\end{restatable}
\begin{proof}
    Choose $\gamma\in C^\infty(\mathbb{R})$ with  $\gamma = 1$ for $x\leq 0$ and $\gamma(x) = 0$ for $x\geq \tfrac{1}{4}$. Define $\gamma_n(\mathbf{x}):=\gamma(\|\mathbf{x}\|_2 - n)$. By construction, $\gamma_n(\mathbf{x})= 1$ for $\|\mathbf{x}\|_2\leq n$ and $\gamma_n(\mathbf{x})=0$ for $\|\mathbf{x}\|_2\geq n+\tfrac{1}{4}$.

    Since $\|\cdot\|\in C^\infty(\mathbb{R}^d\setminus\{0\})$, it holds that $\gamma_n\in C^\infty_0(\mathbb{R}^k)$. Further, $\|\gamma_n\|_{L^\infty(\mathbb{R}^k)} \leq \|\gamma\|_{L^\infty([0, 1/4])} < \infty$. For any $i\leq k$, we have \begin{equation*}
        \|\partial_i\gamma_n(\mathbf{x})\|_{L^\infty(\mathbb{R}^k)} = \sup_{\mathbf{x}\in\mathbb{R}^k} |\gamma'(\|\mathbf{x}\|_2 - n) \cdot (\partial_i \|\cdot\|_2)(\mathbf{x})| \leq \|\gamma'\|_{L^\infty([0,1/4])} \cdot \big\|\partial_i \|\cdot\|_2\big\|_{L^\infty(\|\mathbf{x}\| \geq 1)}.
    \end{equation*} For any $i,j\leq k$ we have \begin{multline*}
        \|\partial_{ij}\gamma_n(\mathbf{x})\|_{L^\infty(\mathbb{R}^k)} = \sup_{\mathbf{x}\in\mathbb{R}^k} |\gamma''(\|\mathbf{x}\|_2 - n)\cdot (\partial_i \|\cdot\|_2)(\mathbf{x}) \cdot (\partial_j \|\cdot\|_2)(\mathbf{x}) +  \gamma'(\|\mathbf{x}\|_2 - n) \cdot (\partial_{ij} \|\cdot\|_2)(\mathbf{x})| \\
        \leq \|\gamma''\|_{L^\infty([0,1/4])} \cdot \big\|\partial_i \|\cdot\|_2\big\|_{L^\infty(\|\mathbf{x}\| \geq 1)}\cdot \big\|\partial_j \|\cdot\|_2\big\|_{L^\infty(\|\mathbf{x}\| \geq 1)} +  \|\gamma'\|_{L^\infty([0,1/4])}\cdot \big\|\partial_{ij} \|\cdot\|_2\big\|_{L^\infty(\|\mathbf{x}\| \geq 1)}.
    \end{multline*} From $\|\partial_{\mathbf{x}^\beta} \|\cdot\|_2\big\|_{L^\infty(\|\mathbf{x}\| \geq 1)}<\infty$ for any multi-index $\beta$ with $1\leq |\beta|\leq 2$ the claim follows. %The norm derivatives are x_i/||x|| (bounded by 1), and delta_{i=j}/||x|| - (x_i x_j)/||x||^3, which decays in x.
\end{proof}

\begin{restatable}{lemma}{lemmaExpectHighProb}\label{lemmaExpectHighProb}
    Let $A,B$ be nonnegative random variables and $\alpha,\beta>0$ be constants such that $A\leq B + \beta \log^\alpha \tfrac{2}{\delta}$ with probability at least $1-\delta$ for all $\delta\in (0,1)$. Then, $\ex{A} \leq 2 \ex{B} +  \beta \cdot c(\alpha)$ for some constant $c$ only depending on $\alpha$.
\end{restatable}
\begin{proof}
    Let $t>0$ be arbitrary. We have that \begin{equation*}
        \pr{A> t + \beta\log^\alpha 2} = \pr{B + (A-B) -\beta\log^\alpha 2> t} \leq \pr{B> \tfrac{t}{2}} + \pr{A > B+\tfrac{t}{2} + \beta\log^\alpha 2}.
    \end{equation*}Further, letting $\delta:=2\exp(- (\tfrac{t}{2\beta} + \log^\alpha 2)^{1/\alpha}) \in (0,1)$, it holds that \begin{equation*}
        \pr{A > B+\tfrac{t}{2} + \beta\log^\alpha 2} \leq 2\exp(- (\tfrac{t}{2\beta} + \log^\alpha 2)^{1/\alpha}).
    \end{equation*}We have that \begin{multline*}
        \ex{A} = \int_0^\infty \pr{A > t}\dd{t} \leq 1\cdot\beta\log^\alpha 2 + \int_{0}^\infty \pr{A > t + \beta\log^\alpha 2}\dd{t} \\
        \leq \beta\log^\alpha 2 + \int_0^\infty \pr{B > \tfrac{t}{2}}\dd{t} + \int_0^\infty 2\exp(- (\tfrac{t}{2\beta} + \log^\alpha 2)^{1/\alpha})\dd{t}.
    \end{multline*} Transforming $\tfrac{t}{2}\mapsto u$ and $\tfrac{t}{2\beta}\mapsto u$ in the integrals proves the theorem claim.
\end{proof}

\end{appendices}

\end{document}